%% file: main.tex
\documentclass[twoside]{article}

\usepackage[accepted]{aistats2026}

\input{macros}

\usepackage{wrapfig}
\usepackage{multirow}
\usepackage{pdflscape} 
\usepackage{tabularray}
\usepackage{titlesec}
\usepackage{titletoc}
\usepackage[round]{natbib}

\usepackage[table]{xcolor}
\definecolor{lightgray}{gray}{0.9}

\begin{document}

\twocolumn[

\aistatstitle{Minimizing Human Intervention in Online Classification}

\aistatsauthor{%
  William R\'eveillard\textsuperscript{1} \And Vasileios Saketos\textsuperscript{1} \AND
  Alexandre Proutiere\textsuperscript{1} \And Richard Combes\textsuperscript{2}
}

\aistatsaddress{%
  \textsuperscript{1}KTH Royal Institute of Technology  \\ \texttt{\{wilrev,saketos,alepro\}@kth.se}\And
  \textsuperscript{2}Univ. Paris-Saclay, CNRS, CentraleSup\'elec \\  \texttt{richard.combes@centralesupelec.fr}
}]

\begin{abstract}
Training or fine-tuning large language model (LLM)–based systems often requires costly human feedback, yet there is limited understanding of how to minimize such intervention while maintaining strong error guarantees. We study this problem for LLM-based classification systems in an active learning framework: an agent sequentially labels $d$-dimensional query embeddings drawn i.i.d. from an unknown distribution by either calling a costly expert or guessing with no feedback, with the goal of minimizing regret relative to an oracle with free expert access. When the horizon $T$ is at least exponential in the embedding dimension $d$, the geometry of the class regions can be learned. In this regime, we propose the Conservative Hull-based Classifier (\vhc), which maintains convex hulls of expert-labeled queries and calls the expert when a query lands outside all known hulls. \vhc\ attains $\mathcal{O}(\log^d T)$ regret in $T$ and is minimax optimal for $d=1$. Otherwise, the geometry cannot be reliably learned in general. We show that for queries drawn from a subgaussian mixture and $T \le e^d$, a Center-based Classifier (\extc) achieves regret proportional to $N\log{N}$ where $N$ is the number of labels. To bridge these regimes, we introduce the Generalized Hull-based Classifier (\vht), a practical extension of \vhc\ that enables more aggressive guessing via a tunable parameter. Our approach is validated on real-world question-answering datasets using state-of-the-art text embedding models.
\end{abstract}

\input{a.main_sections/1.introduction_new}
\input{a.main_sections/2.related_new}
\input{a.main_sections/3.models}

\input{a.main_sections/4.1.CHCalgo}

\input{a.main_sections/4.2.ETCalgo}

\input{a.main_sections/4.3.GHCalgo}
\input{a.main_sections/5.experiments}

\input{a.main_sections/6.conclusion}
\bibliographystyle{apalike}
\bibliography{ref}

\clearpage
\appendix
\input{AISTATS_CHECKLIST}
\onecolumn
\startcontents[app]

\printcontents[app]{l}{1}{\section*{Table of Contents}}{}

\input{b.appendix/1.notation}
\input{b.appendix/2.proofs_CHCupperbound}
\input{b.appendix/3.proofs_lowerbound}
\input{b.appendix/4.proofs_ETC}

\input{b.appendix/5.proofs_GHC}
\input{b.appendix/6.lemmas}
\input{b.appendix/7.additional_related_work}
\input{b.appendix/8.additional_discussions}

\input{b.appendix/9.synthetic_experiments}
\input{b.appendix/10.Datasets}

\input{b.appendix/11.Our_LLMs}
\input{b.appendix/12.real_experiments}

\end{document}

%% file: macros.tex
\usepackage{graphicx} 
\usepackage[utf8]{inputenc} 
\usepackage[T1]{fontenc}    
\usepackage{hyperref}       
\usepackage{url}            
\usepackage{booktabs}       
\usepackage{amsfonts}       
\usepackage{nicefrac}       
\usepackage{microtype}      
\usepackage{xcolor}         

\usepackage[utf8]{inputenc} 
\usepackage[T1]{fontenc} 
\usepackage{hyperref} 
\usepackage{amsmath}
\usepackage{amsthm}
\usepackage{amssymb}
\usepackage{bbm}
\usepackage{paralist}
\usepackage{enumitem}
\usepackage{float}
\usepackage{algorithm}
\usepackage{algorithmic}
\floatstyle{ruled}
\restylefloat{algorithm}
\usepackage{xcolor}
\usepackage{mathtools}
\usepackage{dsfont}
\usepackage{subcaption}
\usepackage{yfonts}
\usepackage{mathrsfs}
\usepackage{tabularx}
\usepackage{dsfont}
\usepackage{comment}
\usepackage{bold-extra}
\usepackage{bm}
\usepackage{tabularray}
\usepackage{adjustbox}
\usepackage[table]{xcolor}
\usepackage{subfiles} 

\theoremstyle{plain}
\newtheorem{theorem}{Theorem}[section]
\newtheorem{proposition}[theorem]{Proposition}
\newtheorem{lemma}[theorem]{Lemma}
\newtheorem{corollary}[theorem]{Corollary}
\theoremstyle{definition}

\newtheorem{assumption}[theorem]{Assumption}
\theoremstyle{remark}
\newtheorem{remark}[theorem]{Remark}

\newcommand{\Sph}{\mathcal{S}}

\newcommand{\bigO}{\mathcal{O}}

\newcommand{\NN}{\mathbb{N}}
\newcommand{\RR}{\mathbb{R}}
\newcommand{\PP}{\mathbb{P}}

\newcommand{\EE}{\mathbb{E}}

\newcommand{\Hcone}{\mathcal{H}}

\newcommand{\convs}{\mathrm{conv}_s}
\newcommand{\conv}{\mathrm{conv}}
\newcommand{\Bin}{\operatorname{Bin}}
\newcommand{\vht}{\textnormal{\texttt{GHC}}}
\newcommand{\vhc}{\textnormal{\texttt{CHC}}}

\newcommand{\extc}{\textnormal{\texttt{CC}}}

\newcommand{\indic}{\mathbbm{1}}

%% file: a.main_sections/1.introduction_new.tex
Question answering and customer support systems often start with no prior knowledge and must learn to respond to user queries over time. 
Initially, human experts are needed to answer incoming queries, and each answered query is stored to aid future responses. However, expert interventions are costly, and the system should minimize unnecessary queries to the expert. This raises a fundamental question: {\it How can we design a system that progressively builds competence in answering queries while minimizing reliance on expert interventions?} We address this question in a setting where each query is represented by an embedding, typically produced by a language model. We formalize this problem as an active learning task. At each round $t=1,2,\dots$, an agent observes a query embedding $q_t \in \RR^d$ drawn i.i.d. from an unknown distribution. There are $N$ latent answer labels, each corresponding to some unknown region in the embedding space. The agent must assign a label to the query: she may either predict a label immediately or ask an expert to obtain the true label (incurring a cost). Crucially, if the agent guesses on its own, no label is revealed. 
The goal is to devise algorithms that minimize the cumulative cost after $T$ queries, where the cost comes from wrongly guessing the label or calling the expert. Our contributions are as follows:

{\bf a. Algorithms} We introduce three algorithms: most importantly, the Conservative Hull-based Classifier (\vhc) and its generalization, \vht. Both exploit the geometry of class regions in the embedding space to guide decisions. \vhc\ is conservative and never makes incorrect predictions; \vht\ introduces a tunable threshold that allows for more frequent guessing. This threshold controls the trade-off between accuracy and expert use, and is particularly beneficial in high-dimensional spaces such as those defined by LLM embeddings. We additionally present a simpler Center-Based Classifier (\extc) that relies on estimates of the class regions centers to make guesses.

{\bf b. Theoretical Guarantees} We provide bounds on the regret of \vhc\ and \extc. Using tools from convex geometry, we show that \vhc\ achieves a regret scaling as $\bigO(\log^d{T})$ after $T$ queries, where $d$ is the embedding dimension. These bounds are distribution-free and robust. Moreover, we show that in dimension $d=1$, \vhc\ is minimax optimal.
To complement our distribution-free bounds, we additionally show that for well-separated subgaussian mixtures with equal weights, \extc attains regret proportional to $N\log{N}$ when $T \le e^d.$

{\bf c. Empirical Validation} We evaluate \vhc, \extc\ and \vht\ on both synthetic data and real-world question-answering tasks, using embeddings from various text embedding models. We demonstrate that leveraging higher-dimensional embeddings significantly reduces the overall regret. Additionally, we show that \vht\ outperforms \extc\ and baselines from the literature in all settings.

%% file: a.main_sections/2.related_new.tex
\section{RELATED WORK}\label{sec:relatedwork}


There is a growing literature on learning problems where an agent must  decide when to defer to an expert or abstain from providing an answer. Our work is related to two frameworks where this objective is central: learning to defer and active learning. We further review adjacent literature on clustering.  Additional related work on online classification, embedding-based retrieval and convex geometry is presented in Appendix~\ref{sec:additional_related_work}.

{\bf Learning to Defer} studies how to combine automated predictions with human oversight \citep{madras2018,mozannar2020}. It is framed as an offline problem where the learner has access to a training dataset of inputs, ground-truth labels, and expert predictions. The primary objective is to minimize a test-time loss that accounts for both prediction error and the cost of human intervention. Our online setting differs fundamentally: rather than leveraging a pre-existing dataset, the agent faces a cold-start problem and must decide at each round whether to call the expert, building competence over time. Consequently, we evaluate performance via cumulative regret rather than test-time risk.

{\bf Active learning.} Online learning has a rich history in theoretical computer science and statistical machine learning \citep{littlestone1988,bendavid2009,rakhlin2015}. Of particular relevance to our work is the subfield of active learning, where the learner decides when to query for a label to minimize annotation costs \citep{dasgupta2011,hanneke2014}. Much of the literature addresses the \emph{pool-based} setting, where the entire dataset is available from the start \citep{lewis1994,Houlsby2011,Kane2017,gentile2022}, and the goal is to minimize the number of queries required to learn a classifier with low error rate, ignoring prediction errors during training. In the \emph{stream-based} setting, data points are observed sequentially, but  the objective typically remains minimizing the error rate at test-time instead of the cumulative number of prediction errors \citep{Freund1997,sabato2018}. A foundational approach to this problem is the CAL algorithm \citep{Cohn1994}, which queries for the label as soon as there is some disagreement among classifiers consistent with previously queried labels, i.e., it queries only when the label cannot be inferred. Although known to be optimal in terms of the number of label requests among zero-error classifiers in the realizable setting \citep{hanneke2014}, it is generally computationally inefficient. Our \vhc\ algorithm (Section \ref{sec:d-dim}) can be understood as a computationally efficient approximation of CAL, when class membership is dictated by a convex partition of the input space. Our setting is most similar to online selective sampling \citep{hanneke2021}, where the object of study is the same as ours, i.e., the trade-off between the number of prediction errors and the number of label queries. Several theoretical works focus on linear binary classification with inputs either adversarially selected with a hard margin constraint \citep{cesa2004,Lu2016}, or stochastic under the Tsybakov noise condition \citep{Wang2016}, whereas we study the multiclass setting and derive bounds over a wide class of distributions by exploiting the geometric structure of the class regions. 

{\bf Online Selective Classification}. Finally, our setting is closely related to selective classification (also known as classification with a reject option), where the learner may opt to not make a prediction. Typically, data is available in batches \citep{Chow1970,bartlett2008,elyaniv2010}, and the goal is to balance the prediction error probability assuming a prediction is made with the probability of making a prediction, also known as the \emph{coverage}. This setting was shown to bear similarity with stream-based active learning in \citet{elyaniv2012}. Recent work has explored online variants where labels are revealed only when the learner \emph{predicts} \citep{cortes2018,bechavod2019} or \emph{abstains} \citep{gangrade2021}. The latter work is the most relevant to ours. They derive bounds on the cumulative number of incorrect guesses and label queries, but assuming a finite set of candidate classifiers, whereas ours is infinite (all classifiers partitioning the embedding space as convex regions). 

{\bf Clustering} refers to the task of grouping data points based on similarity, and is hence related to our problem. It is widely used to summarize and organize large datasets \citep{pandove2018}, typically under the assumption that all data is available upfront in a batch setting. Clustering in high dimensions is notoriously difficult, both statistically and computationally \citep{Steinbach2004}. To overcome this, most batch clustering theoretical works assume strong generative or separability conditions. A rich line of work aims at characterizing the optimal clustering error rate and devising algorithms reaching these limits. For isotropic Gaussian Mixture Models (GMMs) \citep{pearson1894,Renshaw1987}, it is shown in \citet{lu2016statistical} that the average number of misclustered points of Lloyd's algorithm \citep{lloyd1982} achieves the minimax error rate. More recent analyses extend to slightly more general settings, for example anisotropic GMMs \citep{chen2024} and mixture with sub-exponential tails \citep{dreveton2024}, for which optimal rates are obtained with a variant of Lloyd's algorithm.
All such results rely on strong distributional assumptions. In contrast, our framework is distribution-free, and we notably derive regret bounds that hold even in the uniform setting, where no natural clusters exist. In fact, the class regions in our setting need only be defined by convex polytopes of the embedding space, a structure we use to specify the expert’s labeling policy. Furthermore, our framework is online and incorporates a decision trade-off between prediction and expert consultation not found in these works.

Many modern applications—such as news feeds, social media, and interactive search—generate data as continuous streams. These scenarios have motivated the development of online clustering methods, which must operate in real time as data arrives. 
\citet{choromanska2011} propose an online clustering algorithm that maintains an ensemble of batch $k$-means algorithms as ``experts'' and re-weights them by approximating their current clustering cost. \citet{cohen2021} consider online $k$-means clustering: at each step the algorithm maintains $k$ centers and incurs loss equal to the squared distance from the new point to the nearest center, aiming to minimize regret against the best fixed clustering in hindsight. 
These methods focus on minimizing a distance-based cost rather than the classification/regret perspective we consider. Moreover, they do not capture the fundamental trade-off in our model between querying an expert and making an unsupported prediction.

%% file: a.main_sections/3.models.tex
\section{MODEL AND OBJECTIVES}\label{sec:models}

\subsection{Online Classification Problem}\label{sec:problem}

We consider an online optimization problem in which users submit queries to an agent who can consult an expert for assistance. The interactions between the users, the agent, and the expert are as follows. In each round $t\ge 1$:

{\bf a. Query} A user generates a query characterized by its embedding or representation $q_t\in {\cal E}\subset \mathbb{R}^d$ in an i.i.d. manner according to an unknown distribution $\mu$. The agent must address the query. We assume that there are $N$ possible labels $i \in [N]:=\{1,\ldots,N\}$. Throughout the paper, the embedding space is either the hypercube ${\cal E}={\cal I}^d:=[0,1]^d$ or the unit sphere ${\cal E}=\mathcal{S}^{d-1} := \{ x \in \mathbb{R}^d : \Vert x \Vert_2 = 1 \}$ (this is usually the case with text embedding models).

{\bf b. Agent's Guess with or without Expert Guidance} To label the query $q_t$, the agent has access to a dataset ${\cal D}_t$ consisting of query-label pairs that the expert labeled up to round $t$. She can then either (i) ask the expert, or (ii) make a guess. In the former case, the expert provides the correct label denoted by $i_t\in [N]$, and the pair $(q_t,i_t)$ is appended to ${\cal D}_t$, i.e., ${\cal D}_{t+1}\leftarrow {\cal D}_t\cup \{(q_t,i_t)\}$. In the latter case, nothing is added to ${\cal D}_t$ (the correct label is not observed). 

{\bf c. Partially Observed Reward} The reward received at the end of round $t$ is as follows. Let $\hat\imath_t$ denote the label provided by the agent. If the agent asked the expert, then $\hat\imath_t = i_t$ (she has the correct label), but the agent experiences a negative reward equal to $\alpha  \le 0$ corresponding to the cost of calling the expert. 
If she decides to guess without the expert guidance, then she does not observe the reward, and the latter is $\beta \ge 0$ if the label is correct ($\hat\imath_t= i_t$), and $\gamma \le \alpha $ otherwise.


A learning algorithm is defined as a sequence of decisions made by the agent over time. Formally, such an algorithm $\pi$ specifies the decision in round $t$ as a function of the current query $q_t$ and the past observations ${\cal D}_t$. Specifically, it is $\sigma(q_t, {\cal D}_t)$-measurable (where $\sigma(X)$ denotes the sigma-algebra generated by the random variable $X$).
We denote the random reward obtained by the learning algorithm $\pi$ in round $t$ as $r_{\pi}(t)$.

\subsection{Voronoi Regret}\label{sec:voronoi}

\begin{figure}
\begin{center}
\includegraphics[width=0.6\linewidth]{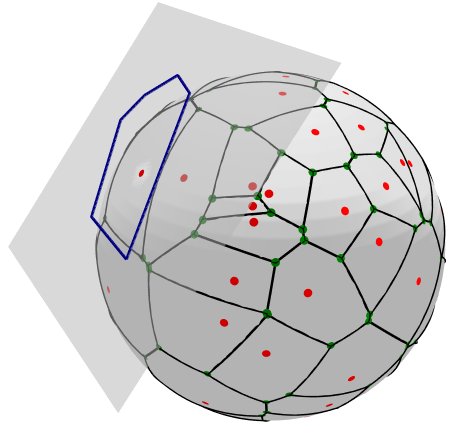}
\end{center}
\caption{Voronoi tessellation of $\mathcal{S}^2$. In blue, gnomonic projection of a cell onto the tangent plane at its seed (used in Section \ref{sec:regretVHC}).}
\label{fig:gnomonic_projection}
{\vspace{-0.8cm}}
\end{figure}
The performance of the learning algorithm used by the agent is measured by the expected cumulative reward over the first $T$ rounds. Alternatively, this performance can be evaluated in terms of regret, defined as the difference between the expected reward of the agent and that of an Oracle algorithm with knowledge of the expert labeling policy. Throughout the paper, we assume that the expert labeling policy is dictated by a partition ${\cal C}_1,\ldots, {\cal C}_N$ of ${\cal E}$, unknown to the agent. The cell ${\cal C}_i$ corresponds to label $i$: whenever $q_t\in {\cal C}_i$, the expert provides label $i$. 

A natural example of partition is constructed as follows. Assume that each label $i$ also has a representation $s_i\in {\cal E}$ that is referred to as its {\it seed}. $s_i$ may be interpreted as the typical query corresponding to the correct label $i$ ($i$ is the correct label to queries whose embeddings would be close to $s_i$). $\mathcal{E}$ can then be partitioned as the Voronoi tessellation generated by the seeds $s_1,\ldots,s_N$. This means that for any $t$, $i_t\in \arg\min_{i \in [N]} \| q_t - s_i \|_2$ and for any $i$, $\mathcal{C}_i = \{q \in \mathcal{E} : \| q-s_i\|_2 \le \| q-s_j\|_2, \;\forall j \in [N] \setminus \{i\}\}.$ This example of partition is justified as follows: When the query distribution $\mu$ is such that $\mathbb{P}(i_t=i)=1/N$ and has a conditional density $p(q_t |i_t=i)=f(\Vert s_i-q_t \Vert_2) $ where f is decreasing,
the expert's label corresponds to the Maximum Likelihood Estimator (MLE) for the correct label, assuming the $s_i$ are known. Furthermore, when $\mathcal{E} = \mathcal{S}^{d-1}$ is the unit sphere, the expert assigns to a query $q$ the index of the seed maximizing the cosine similarity $q^\top s_i$, as $q,s_i$ have unit norm (cosine similarity is a standard measure of semantic similarity in embedding-based Natural Language Processing (NLP) models \citep{reimers-gurevych-2019-sentence}). In this case, the boundaries of the Voronoi cells are arcs of great circles, and the resulting cells form \emph{spherical} convex polytopes, as illustrated in Fig.~\ref{fig:gnomonic_projection}.

Since the Oracle algorithm knows the correct labels, it earns a reward $\beta$ per round. We hence define the regret of an algorithm $\pi$ up to round $T$ as $$R_\pi(T) = \beta T - \sum_{t=1}^T \mathbb{E}[r_{\pi}(t)].$$ Our objective is to devise an algorithm that minimizes this regret. 

The i.i.d. query assumption introduced in Section \ref{sec:problem} is essential to make this objective meaningful: in an adversarial setting, any algorithm must incur linear regret. Indeed, in dimension one, an (even oblivious) adversary can binary-search the label boundary so that the agent either pays the expert nearly every round or suffers a constant error probability per round. Refer to Remark \ref{rem:adversarial} for a more detailed argument.


\subsection{Notation and Assumptions}\label{sec:assumptions}

For a finite set of points $M= \{ m_1,\dots,m_n\} \subset \mathbb{R}^d$, we denote their convex hull as $\conv(M):=\{\sum_{j=1}^{n}\alpha_jm_j: \alpha \ge 0, \sum_{j} \alpha_j =1\}$, and their \emph{spherical} convex hull as $\conv_s(M):=\Sph^{d-1} \cap \{\sum_{j=1}^n \alpha_j m_j : \alpha \ge 0\}$. A subset of $\mathbb{R}^d$ is said to be a convex polytope (resp. \emph{proper} spherical convex polytope) if it can be written as $\conv(M)$ (resp. $\convs(M)$) for some finite set $M \subset \RR^d$ (resp. $M \subset \Sph^{d-1}$).
We will use the unifying notation $\operatorname{hull}_{\mathcal{E}}(M)$ to denote $\conv(M)$ if $\mathcal{E}=\mathcal{I}^d$ (hypercube) and $\conv_s(M)$ if $\mathcal{E}=\Sph^{d-1}$ (unit sphere). We write $f(T)=\bigO(g(T))$ if there exists a constant $K$ that may depend on every model parameter besides $T$ ( including the dimension) such that $\vert f(T) \vert \le K\vert g(T) \vert$ for $T$ sufficiently large. Refer to Appendix \ref{sec:notation} for a full notation table.

Our theoretical results are derived using some of the following assumptions.

\begin{assumption}\label{as:embedding_space}
    The embedding space $\mathcal{E}$ is either the hypercube $\mathcal{I}^d$ or the unit sphere ${\cal S}^{d-1}.$ 
\end{assumption}
\begin{assumption}\label{as:cells}
{\it (i)} When $\mathcal{E}={\cal I}^d$, the cells $\{{\cal C}_i\}_{i\in [N]}$ are convex polytopes 
and form a partition of ${\cal E}$. When $\mathcal{E}={\cal S}^{d-1}$, the cells $\{{\cal C}_i\}_{i\in [N]}$ form a partition of $\mathcal{E}$ and are proper spherical convex polytopes, and each cell $\mathcal{C}_i$ is contained in some open half-sphere $\Sph_{e_i}^{+}=\{v \in \Sph^{d-1}, v^{\top}e_i > 0\}$ where the pole $e_i$ may vary with $i$. \\
{\it (ii)} In some contexts, we will further assume that the partition of ${\cal E}$ is a Voronoi tessellation with seeds $s_1,\dots,s_N \in \mathcal{E}.$
\end{assumption}

When ${\cal E} = \Sph^{d-1}$, the assumption that each cell lies within an open half-sphere is instrumental for applying the {\it gnomonic projection} (see Section 3.1 in \citet{Besau2014-rw}) in our regret analysis. Let $V$ denote the \emph{natural volume measure} on $\mathcal{E}$, namely the Lebesgue measure $\lambda$ on $\mathbb{R}^d$ if $\mathcal{E}=\mathcal{I}^d$, or the spherical Lebesgue measure $\omega$ on $\Sph^{d-1}$ if $\mathcal{E}=\Sph^{d-1}$. 

\begin{assumption}\label{as:density}
Depending on the context, we assume either: \\
{\it (i)} \emph{(Bounded density)} $\mu$ is absolutely continuous with respect to $V$ with bounded density
$f_{\mu}=\frac{d\mu}{dV}$ : there are constants $0 < c,C < \infty$ such that $c\le f_{\mu}(x)\le C$ for $V$\nobreakdash-a.e. $x\in\mathcal{E}$. \\
{\it (ii)} \emph{(Subgaussian mixture)}  $\mu$ is a mixture of $N$ $\sigma$-subgaussian distributions on $\mathbb{R}^d$ with component means $s_i$ and mixture weights $p_i$. Furthermore, the mixture centers are sufficiently separated: $\delta^2_{\min} \ge 80\sigma^2d$ where $\delta_{\min}:=\min_{i \ne j} \Vert s_i - s_j \Vert_2$.

\end{assumption}

The plausibility of the center separation assumption depends on the geometry of the embedding space, a point we detail in Appendix \ref{app:separation_assumption}.



%% file: a.main_sections/4.1.CHCalgo.tex
\section{ALGORITHMS AND REGRET GUARANTEES}\label{sec:algorithm}

In this section, we present our algorithms together with their regret analyses. The first algorithm, \vhc\ (Conservative Hull-based Classifier), is designed for large horizons $T$. It infers the geometry of the cells and makes cautious predictions, calling the expert whenever the correct label cannot be deduced from past observations. For smaller horizons, when $T \le e^d$, we introduce \extc\ (Center-based Classifier), a simple algorithm tailored to subgaussian mixtures: it queries the expert until the component centers are accurately estimated, and then predicts using the label of the nearest estimated center. Finally, we propose \vht\ (Generalized Hull-based Classifier), a flexible extension of \vhc\ with a tunable threshold parameter that allows the learner to guess more frequently.

\subsection{Conservative Hull-Based Classifier (\vhc)}\label{sec:d-dim}

In round $t$, \vhc\ maintains, for each possible label $i$, an estimated set $\hat{\mathcal{C}}_{i,t}$ of queries whose correct label is surely $i$ based on past observations. This set is constructed based on the previous queries in $\mathcal{Q}_{i,t}:=\{q \in \mathcal{E}:(q,i) \in \mathcal{D}_t\}$ for which the agent asked the expert and the expert label was $i$. The definition of $\hat{\mathcal{C}}_{i,t}$ depends on the choice of ${\cal E}$.

(i) If $\mathcal{E}=\mathcal{I}^d$ (hypercube), $\hat{\mathcal{C}}_{i,t}=\conv(\mathcal{Q}_{i,t})$ 
is the convex hull of the queries in $\mathcal{C}_i$ that triggered an expert call. \\
(ii) If $\mathcal{E}=\Sph^{d-1}$ (unit sphere), to ensure that  $\hat{\mathcal{C}}_{i,t}$ is included in ${\cal E}$, we need to consider the spherical convex hull of $\mathcal{Q}_{i,t}$, $\hat{\mathcal{C}}_{i,t}=\conv_s(\mathcal{Q}_{i,t})$.\\
The \vhc~algorithm, whose pseudo-code is given 
in Algorithm \ref{alg:vhc}, returns label $i$ if $q_t\in \hat{\cal C}_{i,t}$. If $q_t\notin \cup_i \hat{\cal C}_{i,t}$, then the expert is called and provides the label. By construction, $\hat{\mathcal{C}}_{i,t}\subseteq \mathcal{C}_i$, so \vhc\ always returns the correct label. In Fig. \ref{fig:mixg}, we provide an example in dimension two of the convex hulls $\hat{\mathcal{C}}_{i,t}$ for $t=200$ when $\mu$ is a mixture of truncated Gaussian distributions.

\begin{algorithm}[H]
\caption{Conservative Hull-based Classifier (\vhc)}
\label{alg:vhc}
\begin{algorithmic}[1]
\STATE Initialize $\mathcal{Q}_{i,1}\leftarrow\emptyset$  for $i\in [N]$
\FOR{$t=1,\dots,T$}
\IF{$\exists\,i \in [N]: q_t\in \operatorname{hull}_{\mathcal{E}}(\mathcal{Q}_{i,t})$} 
\STATE $\hat{\imath}_t \leftarrow i$
\ELSE
\STATE call expert, and set $\hat{\imath}_t \leftarrow i_t$
\STATE $\mathcal{Q}_{{i_t,t+1}}\leftarrow \mathcal{Q}_{{i_t,t}}\cup\{q_t\}$
\ENDIF
\ENDFOR
\end{algorithmic}
\end{algorithm}

\begin{figure}[H]
\centering
\includegraphics[width=0.6\linewidth]{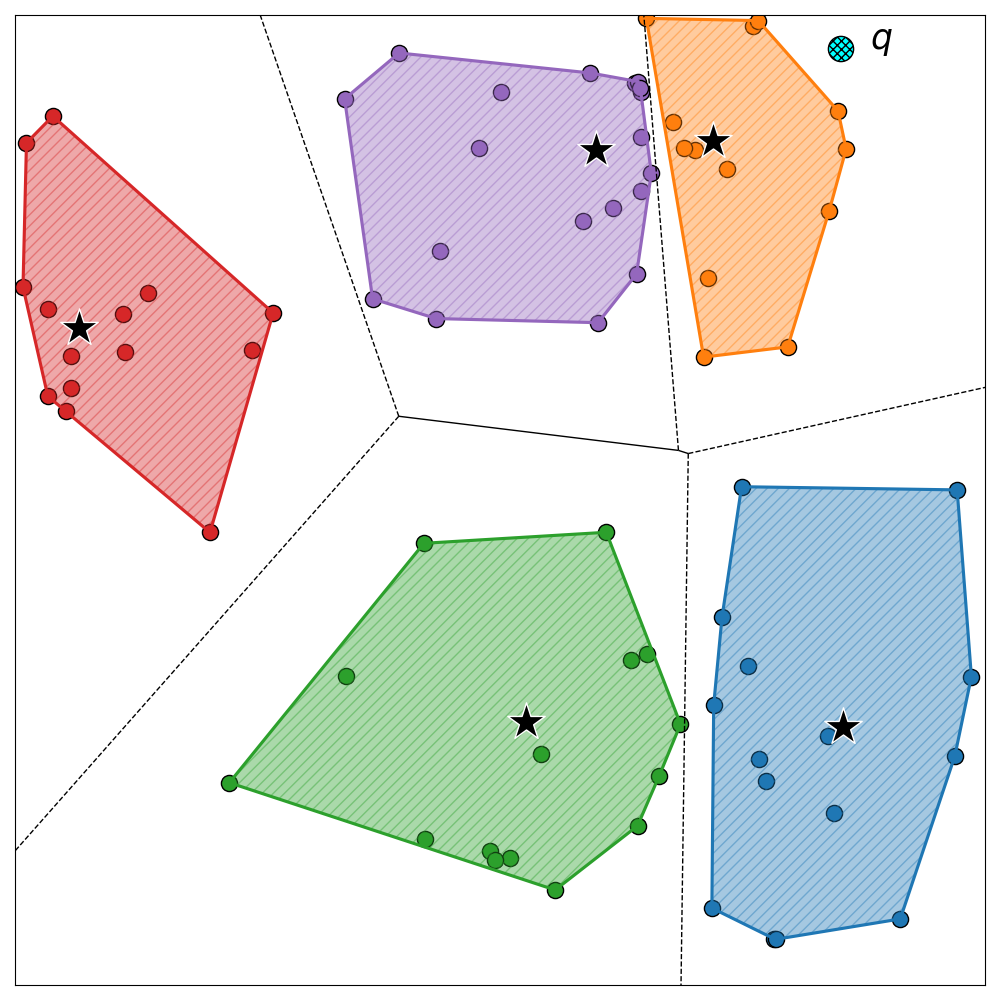}
\caption{Hulls $\hat{\cal C}_{i,t}$ of \vhc\ at $t=200$. $\mu$ is a mixture of truncated Gaussian distributions with equal weights and covariance matrix $0.01I$. Stars are the seeds, circles are the queries that required an expert call.}
\label{fig:mixg}
\end{figure}

\paragraph{Computational Complexity} Importantly, to implement \vhc, we do not need to compute convex hulls (this would be computationally prohibitive—algorithms such as QuickHull have a worst-case complexity that is exponential in $d/2$, see Theorem 3.2 in \citet{barber1996}). Indeed, we only need to be able to check if a query $q \in \mathcal{E}$ belongs to the convex hull of some finite set of queries ${\cal Q}_i=\{ p_1,\dots,p_n\}$ .
If $\mathcal{E}=\mathcal{I}^d$, this amounts to checking if there exists $\alpha \ge 0 \in \mathbb{R}^n$ such that $\sum_{i=1}^{n}\alpha_i =1$ and $q=\sum_{i=1}^{n}\alpha _i p_i$. If $\mathcal{E}=\Sph^{d-1}$, the same condition must be checked but without requiring $\sum_{i=1}^{n}\alpha_i =1$. Checking these conditions can be formulated as linear programs with $n$ variables and $n+d+1$ or $n+d$ constraints respectively. These programs can be solved in a time polynomial in the dimension and the number of queries by interior-point methods, see e.g. \citet{khachiyan1979polynomial,karmarkar1984new}. 


\paragraph{Regret Analysis of \vhc\ in Arbitrary Dimension}\label{sec:regretVHC}

To state the regret guarantees for \vhc, we need to define the number $F(P)$ of {\it flags} of a polytope $P$. A flag of $P$ is a sequence $(F_j)_{j=0}^{d-1}$ of faces\footnote{Faces of a polytope are specific planar surfaces on its boundary, e.g., for a cube, the $0$-dim faces are the vertices, the $1$-dim faces are the $12$ edges, and the $2$-dim faces are the $6$ squares forming the boundary.} of $P$ such that $\dim(F_j)=j$ and $F_0 \subset \dots \subset F_{d-1}$. For example, for $d=2$ and $P=[0,1]^2$, a flag is a sequence formed by a vertex of the square and an edge incident to that vertex, and $F(P)=4 \times 2 =8$. Refer to Section 2.1 in \citet{besau2018} for details. The following regret guarantees for \vhc\ hold for any convex polytope partition $\{\mathcal{C}_i\}_{i \in [N]}$ of the embedding space by the expert.

\begin{theorem}\label{thm:d-dimensional-regret}
Under Assumptions \ref{as:cells}(i) and \ref{as:density}(i): \\
(a) if $\mathcal{E}=\mathcal{I}^d$ and $d \ge 2$, then 
the regret of \vhc\ satisfies 
\begin{align*}R_{\vhc}(T) &\le (\beta-\alpha)\frac{C}{c}\frac{\sum_{i=1}^{N}F(\mathcal{C}_i)}{(d+1)^{d-1}d!}\log^{d}(T) \\
&+\bigO(\log^{d-1}(T)\log\log(T)).\end{align*}
(b) if $\mathcal{E}=\Sph^{d-1}$, $d \ge 3$ and each cell $\mathcal{C}_i$ is contained in an open halfsphere $\Sph_{e_i}^{+}$, 
then the regret of \vhc\ satisfies \begin{align*}R_{\vhc}(T) &\le (\beta-\alpha)\frac{KC}{c}\frac{\sum_{i=1}^{N}F(\mathcal{C}_i)}{d^{d-2}(d-1)!}\log^{d-1}(T) \\
&+\bigO(\log^{d-2}(T)\log\log(T)),\end{align*} 
where $K=\displaystyle \max_{i \in [N]}\left(\frac{\max_{y \in \mathcal{C}_i} y^{\top}e_i}{\min_{y \in \mathcal{C}_i}y^{\top}e_i}\right)^{d}$.  
\end{theorem}


The proof of Theorem \ref{thm:d-dimensional-regret}, provided in Appendix \ref{subsec:proof_regret_bound}, begins by noting that since \vhc\ always returns the correct label, its regret is given by $R_{\vhc}(T) = (\beta - \alpha)\mathbb{E}[C_T]$  where $C_T$ denotes the total number of expert calls made up to round $T$. As the probability of calling the expert for a given query $q_t$ is precisely the probability that $q_t$ falls outside the current convex hull estimates at round $t$, we obtain the following expression for the regret:
\begin{align*}R_{\vhc}(T) &=(\beta-\alpha)\sum_{t=1}^{T} \sum_{i=1}^{N}\mathbb{E}[\mu(\mathcal{C}_i \setminus \hat{\mathcal{C}}_{i,t})].\end{align*}
To upper bound $\mathbb{E}[\mu(\mathcal{C}_i \setminus \hat{\mathcal{C}}_{i,t})]$, we build upon a result from \citet{barany1993}, which characterizes the expected volume of the convex hull of $n$ points sampled uniformly at random from a convex polytope in $\mathbb{R}^d$. We extend this result in two directions: (i) to distributions with bounded density via a rejection sampling argument; (ii) to the spherical setting where points lie on $\Sph^{d-1}$, by mapping spherical polytopes to Euclidean polytopes in $\mathbb{R}^{d-1}$ using the {gnomonic projection}. The dependence on $C/c$ (and $CK/c$ in the spherical case) arises from the rejection sampling step and is likely an artifact of the proof technique—we conjecture that this dependence could be removed with a tighter analysis.

We demonstrate in Appendix \ref{app:dist_independent_lower_bound} that our analysis is sharp by deriving a distribution-free asymptotic lower bound on the regret of \vhc\ that matches Theorem \ref{thm:d-dimensional-regret} when $\mu$ is uniform. This suggests that \vhc\ does not exploit favorable scenarios where $\mu$ is highly concentrated around the seed queries of the cells.




{\bf Voronoi Regret.} We finally derive explicit regret upper bounds when the cells form a Voronoi tessellation, as described in Section \ref{sec:voronoi}. The following corollary is obtained from Theorem \ref{thm:d-dimensional-regret} by controlling the total number of flags, $\sum_{i=1}^{N} F(\mathcal{C}_i)$. The proof, given in Appendix \ref{subsec:flag_bounds}, leverages an upper bound on the number of faces of a convex polytope from \citet{McMullen1970}, along with the observation that the number of facets of each cell $\mathcal{C}_i$ is at most $N - 1 + 2d$ when $\mathcal{E} = {\cal I}^d$ (the hypercube), and at most $N - 1$ when $\mathcal{E} = {\cal S}^{d-1}$ (the unit sphere). 

\begin{corollary}\label{cor:explicit_regret_bound}
Under Assumptions \ref{as:cells}(ii) and \ref{as:density}(i):\\
(a) if $\mathcal{E}=\mathcal{I}^d$ and $d \ge 2$, then the regret of \vhc\ satisfies
\begin{align*}
R_{\vhc}(T) &\le \frac{8(\beta-\alpha)CN}{3c(d+1)^{d-1}}\left(\frac{2e(N+2d)}{d-1}\right)^{d/2}\log^{d}(T) \\
&+\bigO(\log^{d-1}(T)\log\log(T)).
\end{align*}
(b) if $\mathcal{E}=\Sph^{d-1}$, $d \ge 3$ and each cell $\mathcal{C}_i$ is contained in an open half-sphere $\Sph_{e_i}^{+}$,  then the regret of \vhc\ satisfies 
\begin{align*}R_{\vhc}(T) &\le \frac{4(\beta-\alpha)KCN}{cd^{d-2}}\left(\frac{2eN}{d-2}\right)^{(d-1)/2}\log^{d-1}(T) \\
&+\bigO(\log^{d-2}(T)\log\log(T)).\end{align*}
\end{corollary}

It may seem counterintuitive that the constant preceding the leading term (e.g., $\log^{d} T$ when ${\cal E} = {\cal I}^d$) decreases with increasing $d$. This behavior suggests that the error term conceals a more intricate dependency on the dimension. Precisely characterizing this dependency is difficult and touches upon unresolved questions in convex geometry (see Appendix \ref{sec:wet_part} for further discussion). Nonetheless, our upper bounds on the regret of \vhc\ are most informative in regimes where $T$ exceeds a threshold that grows exponentially with $d$, or equivalently, when the dimension is not excessively large. This is consistent with results from \citet{chakraborti2021}, which show that the expected volume of a convex hull formed by random points remains negligible until the number of samples is exponential in $d$. In our setting, this implies that a significant reduction in expert queries only begins once this sample threshold is crossed. 


\paragraph{Minimax Optimality of \vhc\ in Dimension One}\label{sec:regretVHC1}

The following theorem, whose proof is given in Appendix \ref{subsec:proof_ub_dim_1}, provides a sharper regret bound in the case $d=1$, under a weaker condition of $\mu$. An \emph{exact} but unwieldy formula for $R_{\vhc}(T)$ is also given by equation \eqref{eq:exact_regret} in Appendix \ref{subsec:proof_ub_dim_1}. 
Whether this density-free bound can be generalized to dimension $d \ge 2$ is unclear, and it is discussed in Appendix \ref{app:dim_1_vs_dim_d}.

\begin{theorem}\label{thm:1-dimensional-regret-improved}
Assume that\footnote{The theorem holds if ${\cal E}$ is an arbitrary interval of $\mathbb{R}$.} ${\cal E}=[0,1]$ and that $\mu$ has no atoms. Then for all $T \ge 1,$ the regret of \vhc\ satisfies $$R_{\vhc}(T) \le 2(\beta-\alpha)N\log\left(T+1\right).$$
\end{theorem}

Finally, we present, in the following theorem proved in Appendix \ref{app:lower_bound}, minimax regret lower bounds (i.e., satisfied by {\it any} algorithm). These bounds match the regret guarantees for \vhc\ up to a universal constant, hence 
\vhc\ is minimax optimal in dimension one. 

\begin{theorem}[Lower bound on the minimax regret]\label{proposition:Lower bound on the minimax regret}
Assume that ${\cal E}=[0,1]$ and that $\mu$ is uniform on ${\cal E}$. Denote by $\theta = (s_1,...,s_N) \in [0,1]^{N}$ a set of $N$ query seeds in $[0,1]$. Then for all $T\ge 1,$ the minimax regret satisfies 
\begin{align*}
	\inf_{\pi} \max_{\theta \in [0,1]^N} R_{\pi}(T,\theta) 
	&\ge (\beta-\alpha)\frac{N-1}{64 \sqrt{2}} \log\left(\frac{T+1}{2}\right) \\
    &= \Omega( (\beta-\alpha)(N-1) \log{T}).
\end{align*}
\end{theorem}


%% file: a.main_sections/4.2.ETCalgo.tex
\subsection{Center-based Classifier (\extc)}\label{sec:etc_algo} To complement our regret bounds in the large $T$ regime, which hold under very mild distributional assumptions (Assumption \ref{as:density}(i)), we present regret bounds in the $T \le e^d$ regime for a simple Center-based Classifier (\extc), under Assumption \ref{as:density}(ii) on $\mu$. In this section, the expert labeling policy is assumed to be given by the Voronoi tessellation with seeds $s_i$, as described in Section \ref{sec:voronoi}. \extc\ proceeds in two phases:

\underline{Phase 1 (Explore)} In the first phase, the expert is called at each round $t$, and \extc\ stores estimates $\hat{s}_i(t):=\frac{1}{\vert \mathcal{Q}_{i,t}\vert}\sum_{q \in \mathcal{Q}_{i,t}}q$ once each label has been observed at least once. These yield an estimate for the minimum center gap $\delta_{\min}$ as $\hat{\delta}_{\min}(t)=\min_{i \ne j} \Vert \hat{s}_i(t)-\hat{s}_j(t)\Vert_2.$
The first phase ends when each label has been observed sufficiently, i.e., at $T_1:=\min\{t \le T: \forall i \in [N] \, \vert \mathcal{Q}_{i,t}\vert \ge \frac{108\sigma^2}{\hat{\delta}^2_{\min}(t)}\left(d+2\log{T}\right)\}$ with the convention $T_1=T$ if the stopping condition is not reached before round $T$.

\underline{Phase 2 (Commit)} In the second phase ($t \in [T_1+1,T]$), the center estimates are no longer updated. \extc\ guesses at each round the label of the closest estimated center: $\hat{\imath}_t=\arg \min_i \Vert q_t-\hat{s}_i(T_1) \Vert_2.$ 

\begin{theorem}\label{thm:ETc_regret_main} Let $p_{\min}:=\min_{i \in [N]} p_i.$ Under Assumptions \ref{as:cells}(ii) and \ref{as:density}(ii), if $T \le e^d,$ the regret of \extc\ satisfies 
\begin{align*}R_{\extc}(T) &\le \frac{41}{5}\frac{(\beta-\alpha)(\log{N}+1)}{p_{\min}}+2(\beta-\gamma)(N+1).\end{align*} 
\end{theorem}

For mixtures with equal weights, the regret upper bound of \extc\ is proportional to $N\log{N}.$ In Appendix \ref{sec:etc_algo_proof}, we present a more general upper bound of \extc's regret which hold for all regimes of $T$ along with its proof. 
It contains two main steps: (i) the definition of $T_1$ ensures that the center estimation error at the end of the first phase is bounded as $\Vert s_i - \hat{s}_i(T_1) \Vert_2 \le \delta_{\min}/4 $ with high probability; (ii) this entails that the probability of a wrong guess in the second phase is exponentially small in $d$. 

When $\mu$ is a subgaussian mixture, Theorem \ref{thm:ETc_regret_main} demonstrates that \extc\ is effective in the $T \le e^d$ regime, complementing the distribution-free guarantees of \vhc\ which are most meaningful for large $T$. These regimes can be bridged: a hybrid approach—running \extc\ in the early phases and switching to \vhc\ once $T$
exceeds $e^d$—would inherit the regret guarantees of Corollary \ref{cor:explicit_regret_bound} and Theorem \ref{thm:ETc_regret_main} across both regimes. However, a more practical approach is to design a single algorithm that performs well in all settings. To this end, we introduce in the next section a tunable version of \vhc.


%% file: a.main_sections/4.3.GHCalgo.tex
\subsection{Generalized Hull-based Classifier (\vht)}\label{sec:threshold}
When the embedding space is high-dimensional—as is often the case when using representations from LLMs—the \vhc\ algorithm introduced in Section \ref{sec:d-dim} may struggle to perform effectively. In particular, the number of expert-labeled queries must be at least linear in the dimension $d$ for the convex hulls to have non-zero volume, and exponential in $d$ to cover a substantial portion of the space. Consequently, unless the query budget is large, \vhc\ tends to be overly conservative, frequently deferring to the expert. Moreover, 
\vhc\ fails to fully capitalize on favorable cases where the distribution $\mu$ is sharply concentrated around the seeds of each label. 
The \extc\ algorithm on the other hand might not perform well when the horizon $T$ is not much larger than $\frac{\log{N}}{p_{\min}}$ (before the seeds are well estimated). 
To address these limitations, we introduce \vht($\tau$), an extension of \vhc\ specifically designed for high-dimensional settings. This variant incorporates a tunable threshold $\tau$ that enables the algorithm to take calculated risks, particularly when the density of $\mu$ is skewed toward known regions.

In \vht($\tau$), there is an initial phase where \vhc\ is applied until all $N$ labels have been observed. Then, all the hulls contain at least one query, and the agent guesses $\hat{\imath}_t=i$ when $d(q_t,\hat{\mathcal{C}}_{i,t}) \le \tau d(q_t,\hat{\mathcal{C}}_{j,t})$ for all $j \ne i$ for some threshold $\tau \in [0,1]$, where $d(q,S):=\inf_{q' \in S}\Vert q-q'\Vert_2$  denotes the Euclidean distance from $q$ to a set $S$. Otherwise, the expert provides the label $i_t$ and $\hat{\mathcal{C}}_{i_t,t}$ is updated. Note that the hulls $\hat{\mathcal{C}}_{i,t}$ depend on $\tau$ and will typically be formed with fewer samples as $\tau$ increases. If $\tau=0,$ \vht(0) corresponds to \vhc; if $\tau=1$, the expert is never called almost surely once all hulls are non-empty. Intuitively, if $\mu$ is concentrated around the seeds $s_i$, it is more likely that the distances $d(q_t,\hat{\mathcal{C}}_{i,t})$ are well separated. This separation allows for more aggressive guessing (by increasing the threshold $\tau$) with little additional risk, ultimately leading to lower regret. The pseudo-code of \vht\ is provided in Algorithm \ref{alg:vht}. 

\begin{algorithm}[H]
\caption{Generalized Hull-based Classifier (\vht($\tau$))}
\label{alg:vht}
\begin{algorithmic}[1]
\STATE Initialize $\mathcal{Q}_{i,1}\leftarrow\emptyset$  for $i \in [N]$
\FOR{$t=1,\dots,T$}
  \WHILE{$\exists\,i \in [N]: \mathcal{Q}_{i,t}=\emptyset$} 
        \STATE Apply Algorithm \ref{alg:vhc}
  \ENDWHILE
  \IF{$ \exists i \in [N]:$ $d(q_t,\operatorname{hull}_{\mathcal{E}}(\mathcal{Q}_{i,t}))\le \tau\min_{j\neq i} d(q_t,\operatorname{hull}_{\mathcal{E}}(\mathcal{Q}_{j,t}))$} 
    \STATE $\hat{\imath}_t \leftarrow i$ 
  \ELSE \STATE Call expert, and set $\hat{\imath}_t\leftarrow i_t$ \STATE $\mathcal{Q}_{i_t,t+1}\leftarrow \mathcal{Q}_{i_t,t}\cup\{q_t\}$ \ENDIF
\ENDFOR
\end{algorithmic}
\end{algorithm}

Obtaining regret bounds for \vht\ is far more challenging than for \vhc, the main difficulty being that the convex hulls $\hat{\mathcal{C}}_{i,t}$ are no longer formed by i.i.d. queries. This is discussed in more detail in Appendix \ref{sec:analysis_threshold}. 

In Fig. \ref{fig:guessing_regions}, we present, for $\mathcal{E}=[0,1]^2$, the {decision regions} of \vht($\tau$) in round $t=250$ for a few values of $\tau$. Here, $\mu$ is a mixture of Gaussian distributions with small variance. If $q_t$ lands in the ``correct guess'' region (resp. ``wrong guess'' region), the agent guesses and labels $q_t$ correctly (resp. incorrectly). If $q_t$ lands in the ``expert call'' region, the expert provides its correct label. This figure illustrates the benefit of choosing a larger $\tau$ for distributions that are strongly concentrated around representative queries, which is typically the case in practice. Additional examples are provided in Fig. \ref{fig:guessing_regions_full} in Appendix \ref{sec:analysis_threshold}.
\begin{figure}[ht]
    \centering      \includegraphics[width=\linewidth]{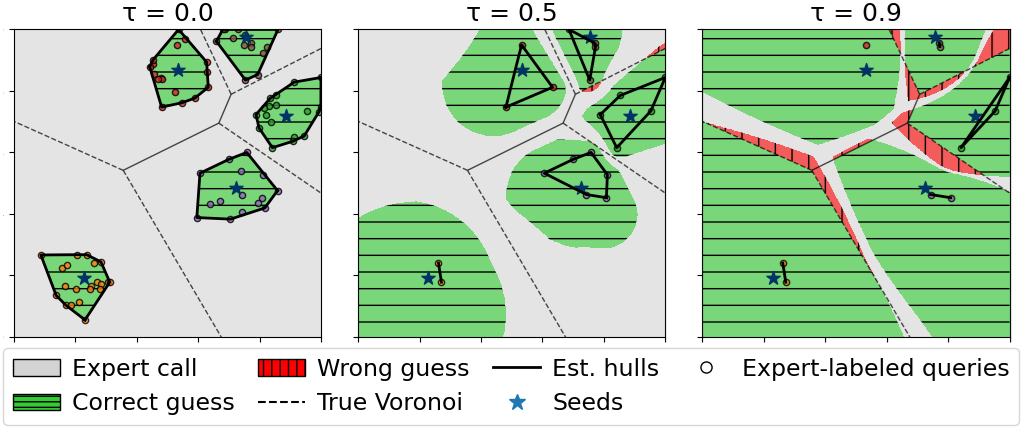}
    \caption{Decision regions of \vht($\tau$) for a mixture of truncated Gaussian distributions, covariance matrix $0.0025I$.}
    \label{fig:guessing_regions}
{\vspace{-0.4cm}}
\end{figure}

\paragraph{Computational Complexity}
To run \vht($\tau$) for $\tau > 0$, we need to compute the distance of $q$ to a (spherical) convex hull of queries $p_1,\dots,p_n$. If $\mathcal{E}=\mathcal{I}^d$, this can be performed in polynomial time without computing the convex hull, simply by computing the value of  the quadratic program $\min_{\alpha \in \mathbb{R}^n} \; \Vert q-\sum_{i=1}^{n} \alpha_ip_i \Vert_2^2$ subject to $\alpha_i \ge 0$, $\sum_{i=1}^{n} \alpha_i=1$. If $\mathcal{E}=\Sph^{d-1}$, the distance to the spherical convex hull is instead given as the value of the same optimization problem but where the last constraint is replaced by $\sum_{i=1}^{n}\alpha_ip_i \in \Sph^{d-1}$. We show in Appendix \ref{sec:proofs_threshold} that this amounts to solving the non-negative least squares problem $\min_{\alpha \ge 0} \Vert q - \sum_{i=1}^{n} \alpha_i p_i \Vert_2^2$. If $\alpha^\star$ is its solution and $\sum_i\alpha_i^\star p_i \ne 0$, we establish that $d(q,\convs\{p_1,\dots,p_n\})=\sqrt{2-2\Vert \sum_{i=1}^{n}\alpha_i^\star p_i \Vert_2}$.

%% file: a.main_sections/5.experiments.tex
\section{NUMERICAL EXPERIMENTS}\label{sec:experiments}
In this section, we evaluate \vht\ on a real-world dataset using various embedding models, and compare \vht\ with \extc\ and baselines from the literature. The reward values are instantiated as $\alpha=-1$, $\beta=+1$, $\gamma=-10$. Additional experiments on both synthetic and real-world datasets are deferred to Appendix \ref{app:experiments}.
The code used for our experiments is available at \href{https://github.com/wilrev/MinimizingHumanIntervention}{\texttt{github.com/wilrev/MinimizingHumanIntervention}}.

To evaluate \vht\ in a realistic setting, we introduce the Quora Question Group (QQG) dataset, which is constructed from the Quora Question Pairs dataset~\citep{DBLP:journals/corr/WangHF17}. Each question in the dataset is assigned one of $N=1103$ labels, where each label corresponds to a group of questions that can be addressed by a single answer.
We compare the performance of three different retrievers/embedding models: Nomic\footnote{huggingface.co/nomic-ai/nomic-embed-text-v1} \citep{nussbaum2024nomic}, E5\footnote{huggingface.co/intfloat/e5-large}\citep{wang2022text}  and Mistral\_E5 \footnote{huggingface.co/intfloat/e5-mistral-7b-instruct} 
\citep{wang-etal-2024-improving-text} (see Appendix \ref{OurLLMS}). We first encode each question with the chosen retrievers, producing embeddings of dimension 784 for Nomic, 1,024 for E5, and 4,096 for Mistral\_E5. Given the relatively small size of our datasets (each label corresponds to a few questions only), we initialize the algorithm by picking an example question for each label, thereby obviating the first phase of \vht\ (where \vhc\ is applied until every label is represented). 

Fig. \ref{fig:img1_qqg} shows the average cumulative regret 
for the three models when we use their respective best performing threshold. Results are averaged over three independent runs. As expected, Mistral\_E5, the most recent and largest model in our experiments with 7 billion parameters, outperforms both E5 (330 million parameters) and Nomic (110 million parameters). Fig. \ref{fig:img2_qqg} shows each model’s performance across all threshold values $\tau$. 
We observe that the best performing thresholds are quite high, which is expected since our embeddings are high-dimensional.


In Appendix \ref{sec:additional_real_world_experiments}, we extend our evaluation to the ComQA \citep{abujabal2019comqa} and CQADupStack \citep{hoogeven2015} datasets. ComQA contains open-domain questions derived from Wikipedia, whereas CQADupStack comprises expert-domain questions from technical forums, including topics such as physics and mathematics. We provide additional details about the datasets used in our experiments in Appendix \ref{Datasets}.

\begin{figure}[htbp]
  \centering
  \begin{subfigure}[b]{0.45\textwidth}
    \centering
    \includegraphics[width=\linewidth]{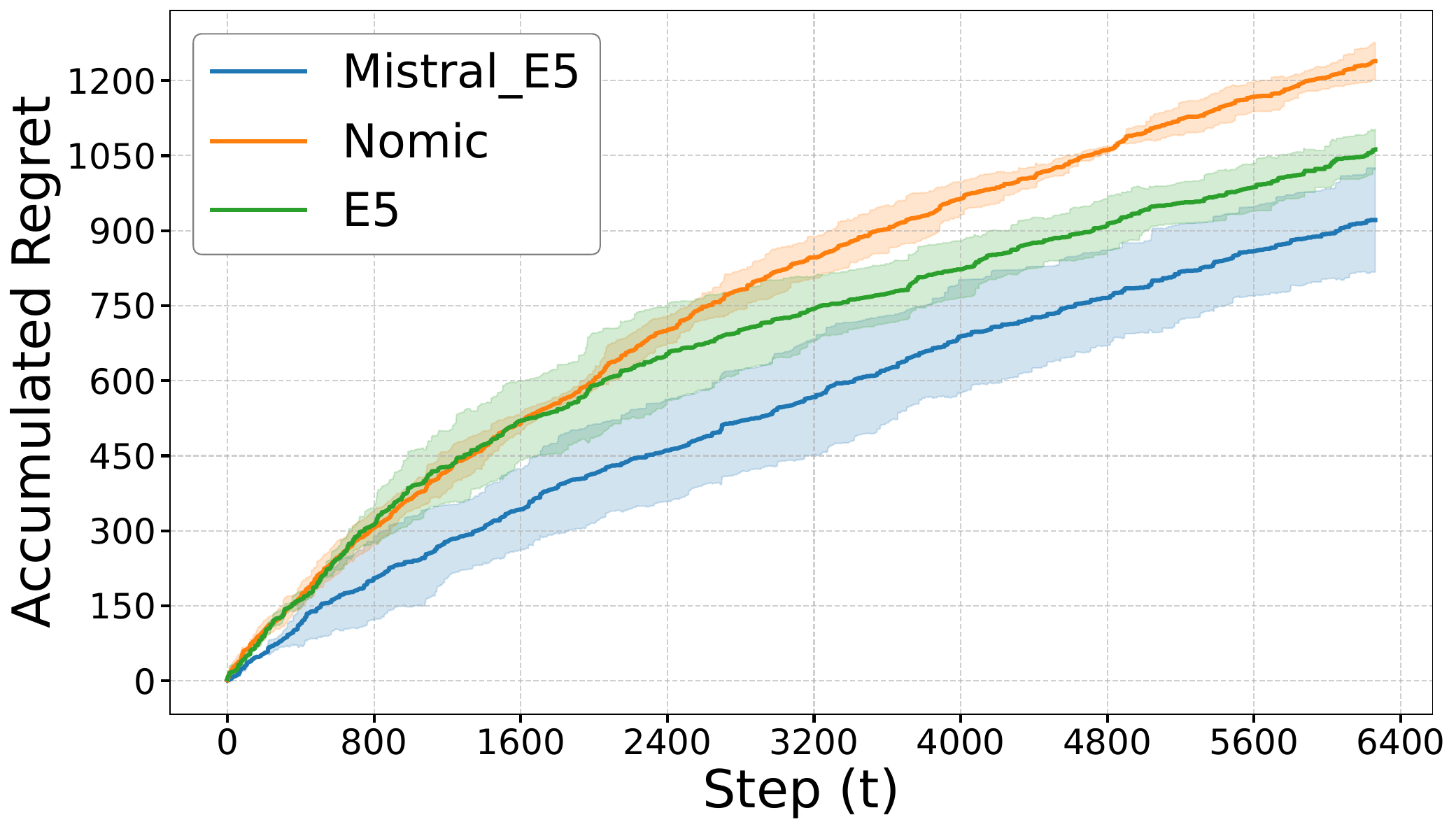}
    \caption{Regret vs. step ($\tau=0.9$)}
    \label{fig:img1_qqg}
  \end{subfigure}\hfill
  \begin{subfigure}[b]{0.45\textwidth}
    \centering
    \includegraphics[width=\linewidth]{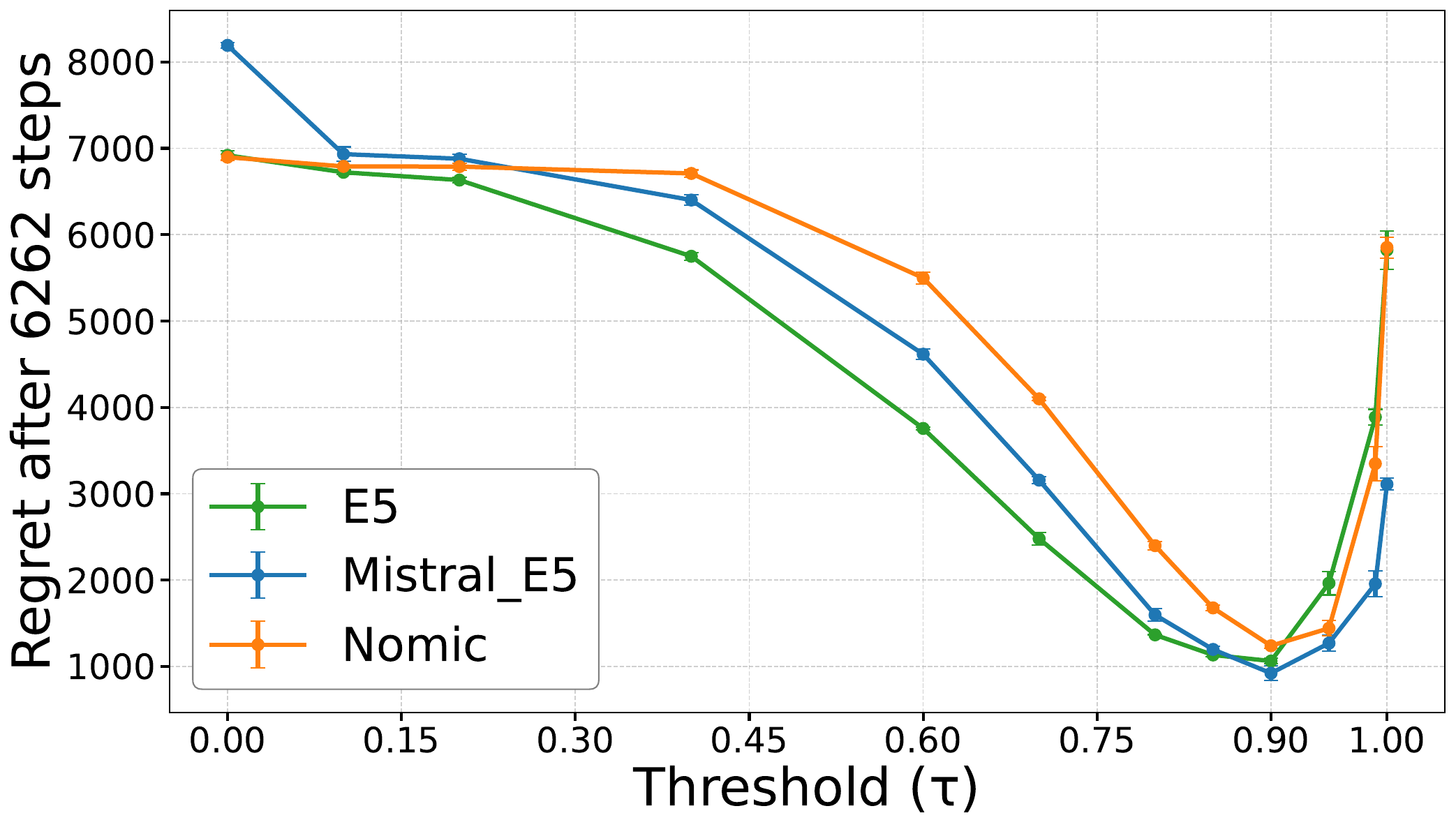}
    \caption{Regret vs. threshold ($T=6262$)}
    \label{fig:img2_qqg}
  \end{subfigure}\hfill
  \caption{Comparison of \vht\ using different text embedding models on Quora Question Groups dataset.}
  \label{fig:Experiements with real data}
\end{figure}

We then compare the performance of \vht, \extc\ (Section~\ref{sec:etc_algo}) and two baselines on the QQG dataset, using embeddings generated by Mistral\_E5. 
The first baseline (\texttt{SKM}) is based on sequential $k$-means \citep{MacQueen1967}. After a supervised initialization phase, this algorithm updates centroids at each round and always guesses based on the closest centroid. 
The second baseline (\texttt{AMP}) is an active variant of the multiclass perceptron \citep{crammer2003,cesa2004}. It defers to the expert when its confidence margin falls below a tunable threshold $\tau \in [0,1].$ Both baselines are described in more detail in Appendix~\ref{app:synthetic_ghc_skm}. In Fig.~\ref{fig:Algorithm_comparison_with_real_data}, we report the average cumulative regret of each algorithm over three independent runs, selecting the best performing threshold for \vht\ and \texttt{AMP} among a range of thresholds (see Appendix \ref{sec:comparison_of_different_algorithm} for the specific ranges tested). \extc\ exhibits approximately linear regret, which is expected since $\log{N}/p_{\min}$ is not negligible with respect to the horizon $T$ in this experiment (see first paragraph of Section \ref{sec:threshold}). The unsupervised baseline \texttt{SKM} also exhibits similar regret scaling, which indicates that both algorithms fail to effectively learn the data structure over time. In contrast, both \vht\ and \texttt{AMP} exhibit sublinear regret. \vht\ consistently achieves the lowest cumulative regret, and its performance gap with \texttt{AMP} widens over time, highlighting the superior long-term learning capability of \vht. Additional comparisons other datasets and with different embedding models are deferred to Appendix \ref{sec:comparison_of_different_algorithm}.

\begin{figure}[htbp]
  \centering
  \begin{subfigure}[b]{0.45\textwidth}
    \centering
    \includegraphics[width=\linewidth]{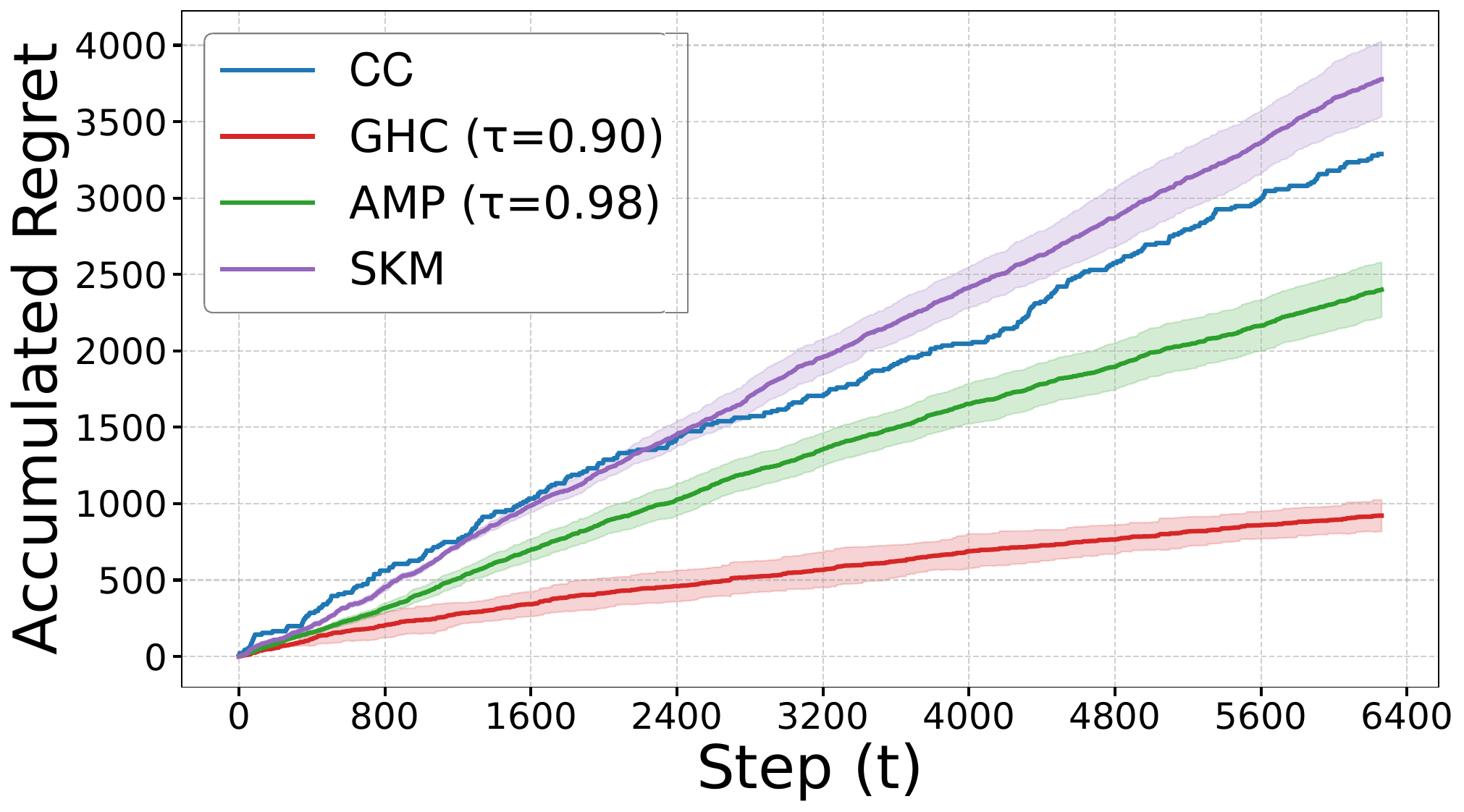}
    \caption{Regret vs. step}
    \label{fig:img_mistral}
  \end{subfigure}\hfill
  \caption{Comparison of different Algorithms on Quora Question Groups dataset with Mistral\_E5.}
  \label{fig:Algorithm_comparison_with_real_data}
\end{figure}


%% file: a.main_sections/6.conclusion.tex
\section{CONCLUSION}

We proposed algorithms tailored to different budget regimes to address the problem of online classification with expert guidance. In particular, we developed the \extc\ and \vhc\ algorithms, which enjoy strong regret guarantees in low- and high-horizon settings, respectively, and further proposed the more versatile \vht\ algorithm, which demonstrates strong empirical performance across all regimes. A natural next step is to provide a theoretical regret analysis of this latter algorithm.

More broadly, the online learning problem studied here represents an instance of a broader class of challenges aimed at minimizing costly human intervention during model training or fine-tuning. We anticipate that many more such problems will emerge, especially in the context of fine-tuning or adapting large foundation models using human feedback.

\section{ACKNOWLEDGMENTS}

This work was supported by computational resources from the National Academic Infrastructure for Supercomputing in Sweden (NAISS), funded in part by the Swedish Research Council (grant no. 2025/5-542).

%% file: AISTATS_CHECKLIST.tex
\section*{Checklist} 

\begin{enumerate}

  \item For all models and algorithms presented, check if you include:
  \begin{enumerate}
    \item A clear description of the mathematical setting, assumptions, algorithm, and/or model. [Yes] Our model, assumptions and algorithms are described in the main text.
    \item An analysis of the properties and complexity (time, space, sample size) of any algorithm. [Yes] The algorithm's computational complexities are discussed in Sections \ref{sec:algorithm} and \ref{sec:threshold}. A regret analysis of \vhc\ is also provided.
    \item (Optional) Anonymized source code, with specification of all dependencies, including external libraries. [Yes] The code used for the experiments is provided in the supplemental material, and will be made publicly accessible upon publication.
  \end{enumerate}

  \item For any theoretical claim, check if you include:
  \begin{enumerate}
    \item Statements of the full set of assumptions of all theoretical results. [Yes] Assumptions are collected in Section \ref{sec:assumptions}  and every theoretical result is stated with the relevant ones.
    \item Complete proofs of all theoretical results. [Yes]  Detailed proofs are provided in the appendices.
    \item Clear explanations of any assumptions. [Yes]  Assumptions are motivated when they are introduced in Section \ref{sec:models} (e.g. i.i.d. assumption, convex polytope assumption).
  \end{enumerate}

  \item For all figures and tables that present empirical results, check if you include:
  \begin{enumerate}
    \item The code, data, and instructions needed to reproduce the main experimental results (either in the supplemental material or as a URL). [Yes] All the code used to produce figures and tables is provided in the supplemental material, and will be made publicly accessible upon publication.
    \item All the training details (e.g., data splits, hyperparameters, how they were chosen). [Yes] The threshold parameter of \vht\ used in the experiments is specified, and all additional details about our setup are provided In Appendix \ref{app:experiments}.
    \item A clear definition of the specific measure or statistics and error bars (e.g., with respect to the random seed after running experiments multiple times). [Yes] The numerical results of Section \ref{sec:experiments} and Appendix \ref{app:experiments} are reported with the corresponding standard deviation when applicable.
    \item A description of the computing infrastructure used. (e.g., type of GPUs, internal cluster, or cloud provider). [Yes] Hardware specifications and compute resources are provided in Appendix \ref{sec:compute_resources}.
  \end{enumerate}

  \item If you are using existing assets (e.g., code, data, models) or curating/releasing new assets, check if you include:
  \begin{enumerate}
    \item Citations of the creator If your work uses existing assets. [Yes] The prior works that introduced the datasets or embedding models we use for our numerical experiments are all cited in the paper and direct links are provided.
    \item The license information of the assets, if applicable. [Not Applicable]
    \item New assets either in the supplemental material or as a URL, if applicable. [Yes] The QQG dataset is provided in the supplemental material.
    \item Information about consent from data providers/curators. [Not Applicable]
    \item Discussion of sensible content if applicable, e.g., personally identifiable information or offensive content. [Not Applicable]
  \end{enumerate}

  \item If you used crowdsourcing or conducted research with human subjects, check if you include:
  \begin{enumerate}
    \item The full text of instructions given to participants and screenshots. [Not Applicable] We conduct no such experiment.
    \item Descriptions of potential participant risks, with links to Institutional Review Board (IRB) approvals if applicable. [Not Applicable]
    \item The estimated hourly wage paid to participants and the total amount spent on participant compensation. [Not Applicable]
  \end{enumerate}

\end{enumerate}

\clearpage

%% file: b.appendix/1.notation.tex
\section{NOTATION}
\label{sec:notation}

\begingroup
\setlength{\fboxsep}{3pt}%
\setlength{\fboxrule}{0.4pt}%
\fbox{%
  \begin{minipage}{\dimexpr\textwidth-2\fboxsep-2\fboxrule\relax}
    \renewcommand{\arraystretch}{1.2}
    \begin{tabularx}{\linewidth}{@{} l|X @{}}
      \multicolumn{2}{@{}l}{\textbf{Problem Setting}} \\
      \hline
      $N$ & Number of labels/classes. \\
      $\mathcal{E}$ & Embedding space: hypercube $\mathcal{I}^d$ or unit sphere $S^{d-1}$. \\
      $d$ & Dimension of the embedding space $\mathcal{E}$. \\
      $\alpha, \beta, \gamma$ & reward for expert call, correct guess and incorrect guess respectively. \\
      $s_i \in \mathcal{E}$ & Unknown seed (representative) of queries with label $i$. \\
      $\mathcal{C}_i$ & Voronoi cell for queries with label $i$, defined by seeds $s_j$. \\
      $T$ & Time horizon (budget). \\
      $q_t \in \mathcal{E}$ & Query received at round $t$. \\
      $\mu$ & Distribution from which $q_t$ is sampled. \\
      $f_{\mu}$ & Density of $\mu$ w.r.t. Lebesgue or spherical measure. \\
      $C,c$ & Upper and lower bounds on $f_{\mu}.$ \\
      $\delta_{\min}$ & Minimum gap $\min_{i \ne j} \Vert s_i-s_j\Vert_2$ between centers. \\
      $i_t \in [N]$ & Correct label of $q_t$. \\
      $\hat{\imath}_t \in [N]$ & Agent's guess for label of $q_t$. \\
      $\mathcal{D}_t$ & Expert-labeled pairs available at step $t$. \\
      $\mathcal{F}_t$ & Filtration of history up to round $t$. \\
      $r_{\pi}(t)$ & Reward at round $t$. \\
      $R_{\pi}(T)$ & Cumulative regret for algorithm $\pi$: $\beta T - \sum_{t=1}^T \mathbb{E}[r_{\pi}(t)]$. \\
      $C_T$ & Total number of expert calls. \\
      $M_T$ & Total number of incorrect guesses (mistakes). \\
      \hline
      \multicolumn{2}{@{}l}{\textbf{Algorithm \& Analysis}} \\
      \hline
      $\conv$ & Euclidean convex hull. \\
      $\convs$ & Spherical convex hull. \\
      $\operatorname{hull}_{\mathcal{E}}$ & $\conv$ if $\mathcal{E}=\mathcal{I}^d$, $\convs$ if $\mathcal{E}=\Sph^{d-1}$. \\
      $d(q, S) = \inf_{s \in S}\|q - s\|_2$ & Euclidean distance from $q$ to $S$. \\
      $\mathcal{Q}_{i,t}$ & Set of queries labeled as $i$ by the expert up to round $t$. \\
      $\hat{\mathcal{C}}_{i,t}$ & Convex hull of $\mathcal{Q}_{i,t}$, estimate of $\mathcal{C}_i$ at round $t$. \\
      $\tau \in [0, 1]$ & Threshold parameter. \\
      \vhc  &  Conservative Hull-based Classifier. \\
      \vht($\tau$)  &  Generalized Hull-based Classifier with threshold $\tau$. \\
      \extc & Center-based Classifier. \\
      $f(T)=\bigO(g(T))$ & There exists $K$ independent of $T$ (potentially depending on other parameters) such that $\vert f(T)| \le K\vert g(T) \vert$ for $T \ge T_0 \in \NN$ \\
      $n_i(t)=\sum_{s=1}^{t}\mathbbm{1}\{i_s=i\}$ & Number of queries with true label $i$ up to $t$. \\
      $\mu_{\mathcal{C}_i}$ & Conditional distribution $\mu(\cdot \cap \mathcal{C}_i)/\mu(\mathcal{C}_i)$. \\
      $w^\nu$ & $\nu$-measure of the wet part. \\
      $F(P)$ & Number of flags of polytope $P$, i.e., increasing sequences $(F_j)_{j=0}^{d-1}$ of $j$-dimensional faces of $P$. \\
      $\lambda$ & Lebesgue measure on $\mathbb{R}^d$. \\
      $\omega$ & Spherical Lebesgue measure on $\Sph^{d-1}$. \\
      $\log$ & Natural logarithm (base $e$). \\
      $\Bin(n,p)$ & Binomial distribution with parameters $n$ and $p$.
    \end{tabularx}
  \end{minipage}%
}%
\endgroup

\newpage

%% file: b.appendix/2.proofs_CHCupperbound.tex
\section{PROOFS OF SECTION \ref{sec:algorithm}}\label{sec:proofs_conservative}

In this section, we provide complete proofs of all the statements in Section \ref{sec:algorithm}.

\subsection{Proofs of Section \ref{sec:regretVHC}}\label{subsec:proofs_bounded_density}

\begin{remark} Throughout this appendix, when the notations $\mathcal{O}(\cdot)$ or $o(\cdot)$ are used, the limit is w.r.t. the number of sampled points $n$ or the number of queries $t$, $T$ depending on the context. All other problem parameters—such as the cells and their dimension—are fixed, so the hidden constants may depend on them.
\end{remark}

\subsubsection{Proof of Theorem \ref{thm:d-dimensional-regret}}\label{subsec:proof_regret_bound}

\paragraph{(a) Case $\mathcal{E}=\mathcal{I}^d$.} We rely on the following Theorem from \citet{barany1993}.

\begin{proposition}[Theorem 2 from \citet{barany1993}]\label{prop:polytope_volume}
Let $P$ a convex polytope in $\mathbb{R}^d$ with $d \ge 2$ and $n \ge 1$ points $p_1,\dots,p_n$ sampled independently and uniformly at random in $P$, with convex hull $P_n$. Then $$\mathbb{E}[\lambda(P \setminus P_n)] = \frac{\lambda(P)F(P)}{(d+1)^{d-1}(d-1)!}\frac{\log^{d-1}{n}}{n}+\bigO\left(\frac{\log^{d-2}(n)\log{\log{n}}}{n}\right)$$ where $F(P)$ is the number of flags of $P$, i.e., the number of sequences $F_0 \subset F_1 \subset \dots \subset F_{d-1}$ of $i$-dimensional faces of $P$.
\end{proposition}
Using a rejection sampling argument, it can be generalized to bounded densities as follows.
\begin{corollary}\label{cor:thinned_polytope_volume}
    Let $P$ a convex polytope in $\mathbb{R}^d$ with $d \ge 2$ and $n \ge 1$ points $p_1,\dots,p_n$ sampled  independently from a distribution $\nu$ supported on $P$ with a density $f_{P}=\mathrm{d}\nu/\mathrm{d}\lambda$ that satisfies $c \le f_{P}(x) \le C$ for $\lambda$-a.e. $x \in P$ for $c>0$, $C<\infty$. Denote by $P_n$ the convex hull of $p_1,\dots,p_n$. Then $$\mathbb{E}\left[\nu(P\setminus P_n)\right] \le \frac{C}{c}\frac{F(P)}{(d+1)^{d-1}(d-1)!}\frac{\log^{d-1}{n}}{n}+\bigO\left(\frac{\log^{d-2}(n)\log{\log{n}}}{n}\right) $$
\end{corollary}

We suspect that the dependency in $C/c$ of Corollary \ref{cor:thinned_polytope_volume} is an artifact of our proof technique and could be improved. This corollary is proved further below.

\begin{proof}[Proof of Theorem \ref{thm:d-dimensional-regret}(a)]
Let us denote by $(\mathcal{F}_s)_{s \le t}$ the filtration generated by the query and expert feedback history history up to round $t$. The conditional probability of calling the expert at time $t$ is given as \begin{align}\label{eq:expert-proba}\mathbb{P}(q_t \in \bigcup_{i=1}^{N}(\mathcal{C}_i \setminus \hat{\mathcal{C}}_{i,t}) | \mathcal{F}_{t-1})=\sum_{i=1}^{N}\mu(\mathcal{C}_i \setminus \hat{\mathcal{C}}_{i,t}).\end{align} Since \vhc\ always guesses correctly, the total regret suffered at time $T$ is then $$R_{\vhc}(T) = (\beta-\alpha)\mathbb{E}[C_T]= (\beta-\alpha)\sum_{t=1}^{T}\mathbb{E}[\mathbb{P}(q_t \in \bigcup_{i=1}^{N}(\mathcal{C}_i \setminus \hat{\mathcal{C}}_{i,t}) | \mathcal{F}_{t-1})]= (\beta-\alpha)\sum_{t=1}^{T}\sum_{i=1}^{N}\mathbb{E}[\mu(\mathcal{C}_i \setminus \hat{\mathcal{C}}_{i,t})].$$ 

By assumption \ref{as:density}(i), the density of $q_t$ is bounded above by $C$ and bounded below by $c$. The conditional density of $q_t$ given $q_t \in \mathcal{C}_i$ is given by $f_{\mathcal{C}_i}(q):=f_{\mu}(q)/\mu(\mathcal{C}_i) \in [c/\mu(\mathcal{C}_i),C/\mu(\mathcal{C}_i)]$ so it is also bounded above and below, and the ratio of the upper and lower bounds on $f_{\mathcal{C}_i}$ is also equal to $C/c.$ Let $n_i(t)$ be the number of points in $\mathcal{C}_i$ at time $t$. Since $\mathbb{P}(q_t \in \mathcal{C}_i) = \mu(\mathcal{C}_i)$, $n_i(t)$ is binomial with parameters $(t,\mu(\mathcal{C}_i))$ and we have $ \mathbb{E}[n_i(t)] = t\mu(\mathcal{C}_i).$ Conditionally to $n_i(t)=n,$ $\hat{\mathcal{C}}_{i,t}$ is distributed as the convex hull $P_n$ of $n$ points sampled at random within $\mathcal{C}_i$ with density $f_{\mathcal{C}_i}.$ Denoting by $\mu_{\mathcal{C}_i}$ the conditional distribution of $q_t$ given that $q_t \in \mathcal{C}_i,$ $\mathbb{E}[\mu(\mathcal{C}_i \setminus \hat{\mathcal{C}}_{i,t}) | n_i(t)=n]=\mathbb{E}[\mu(\mathcal{C}_i \setminus P_n) | n_i(t)=n]=\mu(\mathcal{C}_i)\mathbb{E}[\mu_{\mathcal{C}_i}(\mathcal{C}_i \setminus P_n)]$   by independence between $P_n$ and $n_i(t).$
At time $t$, applying Corollary \ref{cor:thinned_polytope_volume} conditionally to $n_i(t)=n>0$ we obtain $$\mathbb{E}[\mu(\mathcal{C}_i \setminus \hat{\mathcal{C}}_{i,t}) | n_i(t)=n] = \frac{C}{c}\frac{\mu(\mathcal{C}_i)F(\mathcal{C}_i)}{(d+1)^{d-1}(d-1)!}\frac{\log^{d-1}n}{n}+\bigO\left(\frac{\log^{d-2}(n)\log{\log{n}}}{n}\right).$$
If $n=0$, we simply have $\mathbb{E}[\mu(\mathcal{C}_i \setminus \hat{\mathcal{C}}_{i,t}) | n_i(t)=n]=\mu(\mathcal{C}_i).$
Thus, for $\varphi_1$ defined by $\varphi_1(x)=\log^{d-1}(x)/x$ for $x>0$ and $\varphi_1(0)=0$ and $\varphi_2$ defined by $\varphi_2(x)=\log^{d-2}(x)(\log\log x)/x$ for $x>1$, $\varphi_2(0)=\mu(\mathcal{C}_i),\varphi_2(1)=0,$ we have $$\mathbb{E}[\mu(\mathcal{C}_i \setminus \hat{\mathcal{C}}_{i,t})]
 =\frac{C}{c}\frac{\mu(\mathcal{C}_i)F(\mathcal{C}_i)}{(d+1)^{d-1}(d-1)!}\mathbb{E}[\varphi_1(n_i(t))]+\bigO\left(\mathbb{E}[\varphi_2(n_i(t))]\right).$$ 
 
Applying Lemma \ref{lem:binomial-concentration} to $n_i(t)$ gives  $$\mathbb{E}[\varphi_1(n_i(t))]=\frac{\log^{d-1}(t\mu(\mathcal{C}_i))}{t\mu(\mathcal{C}_i)}+\bigO(1/t)$$ and $$\mathbb{E}[\varphi_2(n_i(t))]=\frac{\log^{d-2}(t\mu(\mathcal{C}_i))\log\log(t\mu(\mathcal{C}_i))}{t\mu(\mathcal{C}_i)}+\bigO(1/t).$$Using the fact that $t\mu(\mathcal{C}_i) \le t$, we obtain $$\mathbb{E}[\mu(\mathcal{C}_i \setminus \hat{\mathcal{C}}_{i,t})] \le \frac{C}{c}\frac{F(\mathcal{C}_i)}{(d+1)^{d-1}(d-1)!}\frac{\log^{d-1}(t)}{t}+\bigO\left(\frac{\log^{d-2}{t}\log\log{t}}{t}\right).$$ This then entails a bound on the regret as $$R_{\vhc}(T) \le (\beta-\alpha)\frac{C}{c}\sum_{i=1}^{N}\frac{F(\mathcal{C}_i)}{(d+1)^{d-1}(d-1)!}\sum_{t=1}^{T}\frac{\log^{d-1}(t)}{t} +\bigO\left(\sum_{t=1}^{T}\frac{\log^{d-2}{t}\log\log{t}}{t}\right).$$
To finish the proof, we simplify the sums over $t$. For $x\ge t_0=\lceil e^{d-1}\rceil$, $x \mapsto \varphi_1(x)$ is decreasing, so that \begin{align*}\sum_{t=1}^{T}\frac{\log^{d-1}(t)}{t}&=\sum_{t=1}^{t_0-1}\frac{\log^{d-1}{t}}{t}+\sum_{t=t_0}^{T}\frac{\log^{d-1}{t}}{t} \\
&\le \int_{1}^{T}\frac{\log^{d-1}{x}}{x}\mathrm{d}x+\bigO(1) \\
&=\frac{\log^{d}{T}}{d}+\bigO(1).\end{align*} Similarly, $x \mapsto \varphi_2(x)$ is decreasing for $x$ large enough, so that \begin{align*}\sum_{t=1}^{T}\frac{\log^{d-2}{t}\log\log{t}}{t}&=\int_{e^2}^{T}\frac{\log^{d-2}{x}\log\log{x}}{x}\mathrm{d}x+\bigO(1) \\
&=\frac{\log^{d-1}(T)((d-1)\log\log(T)-1)}{(d-1)^2}+\bigO(1).\end{align*} Finally, 
\begin{align*} R_{\vhc}(T) &\le (\beta-\alpha)\frac{C}{c}\sum_{i=1}^{N}\frac{F(\mathcal{C}_i)}{(d+1)^{d-1}(d-1)!}\frac{\log^{d}{T}}{d}+\bigO(\log^{d-1}(T)\log\log(T))\\
&=(\beta-\alpha)\frac{C}{c}\frac{\sum_{i=1}^{N}F(\mathcal{C}_i)}{(d+1)^{d-1}d!}\log^{d}{T}+\bigO(\log^{d-1}(T)\log\log(T)).\end{align*}

\end{proof}

\begin{proof}[Proof of Corollary \ref{cor:thinned_polytope_volume}]
In the proof, we omit the subscript $P$ and simply denote the density $f_P$ by $f$.
For a uniform sample $q_1,\dots,q_n$ from $P$ with density $u(x)=1/\lambda(P)$ for $x\in P$, we denote its convex hull by $Q_n=\operatorname{conv}(q_1,\dots,q_n)$. 
We want to upper bound $E_{f,n}:=\mathbb{E}[\nu(P \setminus P_n)]$ by relating it to $E_{u,n}:=\mathbb{E}[\lambda(P\setminus Q_n)]$, which can be bounded by Proposition \ref{prop:polytope_volume}. We proceed as follows:

\begin{enumerate}[label=(\roman*)]
\item Start from the sample $p_1,\dots,p_n$ generated from $\nu$. We denote by $f$ their density.
\item For each $i \in [n]$, sample $U_i$ uniform on $[0,1]$ independently, and keep $p_i$ if and only if $U_i \le c/f(p_i).$  Note that $0 < c/C \le c/f(p_i) \le 1.$
\end{enumerate}
Let $J$ be the set of indices of the points $p_i$ that are kept.
\begin{itemize}
    \item By construction, $\mathbb{P}(i \in J | p_i=x)=\mathbb{P}(U_i \le c/f(p_i) | p_i=x)=\mathbb{P}(U_i \le c/f(x))=c/f(x)$ so for any measurable set $A$, $\mathbb{P}(p_i \in A \cap i \in J)=\int_{A}f(x)\frac{c}{f(x)}\mathrm{d}x=c\lambda(A)$ and $\mathbb{P}(i \in J)=c\lambda(P)$. This implies that the conditional density of the kept points is $u$. Therefore, conditionally on $K:=\vert J \vert=k$, the set of kept points $\{p_i |i \in J\}$ is distributed as $\{q_1,\dots,q_k\}$ with $q_i$ independent and  uniform on $P$. 
    \item Since $\mathbb{P}(i \in J)=c\lambda(P)$, the random number of kept points $K$ follows a $\Bin(n,c\lambda(P))$ distribution. 
\item Consider $P_n^J=\operatorname{conv}\{p_i |i \in J\}$. Clearly $P_n^J \subseteq P_n$ so $\lambda(P \setminus P_n) \le \lambda(P \setminus P_n^J).$
\end{itemize}
This means that $\nu(P \setminus P_n)=\int_{P \setminus P_n}f(x)\mathrm{d}x \le C\lambda(P \setminus P_n) \le C\lambda(P \setminus P_n^J)$, so that
\begin{align*}E_{f,n} &\le C\mathbb{E}[\lambda(P \setminus P_n^J)] \\
&=C\sum_{k=0}^{n}\mathbb{P}(K=k)\mathbb{E}[\lambda(P \setminus P_n^J) |K=k].\end{align*}

Conditionally on $K=k,$ since $\{p_i | i \in J\}$ is distributed as $\{q_1,\dots,q_k\},$ $P_n^J$ is distributed as $Q_k$. Thus $$E_{f,n} \le  C\sum_{k=0}^{n}\mathbb{P}(K=k)E_{u,k}.$$ 

Then, by Proposition \ref{prop:polytope_volume}, $E_{u,k} =\frac{\lambda(P)F(P)}{(d+1)^{d-1}(d-1)!}\frac{\log^{d-1}(k)}{k}+\bigO\left(\frac{\log^{d-2}k\log{\log{k}}}{k}\right)$ if $k>1$, and $\lambda(P)$ if $k=0.$ Applying Lemma \ref{lem:binomial-concentration}  (as in the proof of Theorem \ref{thm:d-dimensional-regret}) gives $\mathbb{E}[E_{u,K}] \le \frac{\lambda(P)F(P)}{(d+1)^{d-1}(d-1)!}\frac{\log^{d-1}(nc\lambda(P))}{nc\lambda(P)}+\bigO\left(\frac{\log^{d-2}n\log{\log{n}}}{n}\right)$ which ensures

$$ E_{f,n} \le \frac{C}{c}\frac{F(P)}{(d+1)^{d-1}(d-1)!}\frac{\log^{d-1}(n)}{n}+\bigO\left(\frac{\log^{d-2}n\log{\log{n}}}{n}\right). $$ 

\end{proof}

\paragraph{(b) Case $\mathcal{E}=\Sph^{d-1}.$} To derive the regret bound in the spherical setting, we first obtain a variant of Corollary \ref{cor:thinned_polytope_volume} by
projecting spherical convex polytopes to Euclidean polytopes via the \emph{gnomonic projection} (see e.g. \citet{Besau2014-rw}), similarly to \citet{besau2018} in the proof of their Theorem 1.3. This result holds for distributions with bounded densities w.r.t. the spherical Lebesgue measure $\omega.$
\begin{corollary}\label{cor:spherical_polytope_volume}
    Let $P$ a spherically convex polytope in $\Sph^{d-1}$ contained in an open halfsphere $\Sph_e^{+}=\{v \in \Sph^{d-1}, v^{\top}e > 0\}$ with $d \ge 3$ and $n \ge 1$ points $p_1,\dots,p_n$ sampled  independently from a distribution $\nu$ supported on $P$ with a density $f_P=\mathrm{d}\nu/\mathrm{d}\omega$ that satisfies $c \le f_P(x) \le C$ for $\omega$-a.e. $x \in P$ for $c>0$, $C<\infty$. Denote by $P_n$ the spherical convex hull of $p_1,\dots,p_n$. Then $$\mathbb{E}\left[\nu(P\setminus P_n)\right] \le \frac{C}{c}\left(\frac{\max_{y \in P} y^{\top}e}{\min_{y \in P}y^{\top}e}\right)^{d}\frac{F(P)}{d^{d-2}(d-2)!}\frac{\log^{d-2}{n}}{n}+\bigO\left(\frac{\log^{d-3}(n)\log{\log{n}}}{n}\right) $$
\end{corollary}

The derivation of the regret bound is then exactly the same as in the case $\mathcal{E}=\mathcal{I}^d$ and is consequently omitted.

\begin{proof}[Proof of Corollary \ref{cor:spherical_polytope_volume}]
Let \begin{align*}
g_e : \Sph_e^{+} &\to \mathbb{R}^{d-1} \\
x &\mapsto \frac{x}{x^\top e} - e
\end{align*}
 denote the gnomonic projection from the open halfsphere $\Sph_e^{+}$ onto the hyperplane tangent to $\Sph^{d-1}$ at the pole $e$ of $\Sph_e^{+}$. This projection is bijective and its inverse is given by $g_e^{-1}(y)=(y+e)/\Vert y + e  \Vert_2$ \citep{besau2018}.
Let $Q= g_e(P)$. $Q$ is a Euclidean convex polytope in $\RR^{d-1}$. Let $ \rho=g_e \# \nu$ denote the pushforward measure of $\nu$ onto $Q$, i.e., $\rho(A)=\nu(g_e^{-1}(A)) $ for all measurable set $A \subseteq Q.$ $\rho$ is a probability measure on $Q$. Let $q_i = g_e(p_i)$ for $i=1, \dots, n$. The points $q_i$ are independent samples from $\rho$. Let $Q_n = \conv(q_1, \dots, q_n)$ be their Euclidean convex hull.

A key property of the gnomonic projection is that it maps spherical geodesics (arcs of great circles) to Euclidean straight lines. Consequently, it maps the spherical convex hull of a set of points to the Euclidean convex hull of the projected points:
\[
g_e(P_n) = g_e(\convs(p_1, \dots, p_n)) = \conv(g_e(p_1), \dots, g_e(p_n)) = Q_n.
\]
Since $g_e$ is injective, it preserves the set difference
\[
g_e(P \setminus P_n) = g_e(P) \setminus g_e(P_n) = Q \setminus Q_n
\]
so that $P \setminus P_n = g_e^{-1}(Q \setminus Q_n)$.
By the definition of the pushforward measure $\rho$,
\[
\rho(Q \setminus Q_n) = \nu(P \setminus P_n).
\]
Taking the expectation yields
\begin{equation} \label{eq:expectation_equality}
\EE[\rho(Q \setminus Q_n)] = \EE[\nu(P \setminus P_n)].
\end{equation}

We now verify the conditions of Corollary \ref{cor:thinned_polytope_volume} for our measure $\rho$ on $Q$. By assumption, $\nu$ has a density $f_{P}=\mathrm{d}\nu/\mathrm{d}\omega$ w.r.t. the spherical Lebesgue measure $\omega$ satisfying $0 < c \le f_{P}(y) \le C$ for $\omega$-a.e. every $y \in P.$ The pushforward measure $\rho = g_e \# \nu$ has a density $f_{Q} = d\rho/d\lambda$ on $Q$ given by $f_{Q}(x) \propto f_{P}(g_e^{-1}(x))(1 + \|x\|^2)^{-d/2}$ for $x \in Q$ (see \citet{besau2018}, p.36). On the compact $Q$, the density $f_{Q}$ is bounded below by $c' = \min_{x \in Q} f_{Q}(x) > 0$ and above by $C' = \max_{x \in Q} f_{Q}(x) < \infty$, where $\frac{C'}{c'}=\frac{C}{c}\left(\frac{1+r^2_{\max}}{1+r^2_{\min}}\right)^{d/2}$ for $r_{\max}=\max_{x \in g_e(P)} \Vert x \Vert$, $r_{\min}=\min_{x \in g_e(P)} \Vert x \Vert$. Note that for $x=g_e(y) \in g_e(P)$ , $\Vert x \Vert^2= \Vert \frac{y}{y^{\top}e}-e\Vert^2=\frac{1}{(y^{\top}e)^2}-1$ since $y,e \in \Sph^{d-1}.$ Thus $$\left(\frac{1+r^2_{\max}}{1+r^2_{\min}}\right)^{d/2}=\left(\frac{\max_{y \in P} 1/(y^{\top}e)^2}{\min_{y \in P} 1/(y^{\top}e)^2}\right)^{d/2}=\left(\frac{\max_{y \in P} y^{\top}e}{\min_{y \in P} y^{\top}e}\right)^{d}$$ since $y^{\top}e>0$ for every $y \in P.$

Thus, the conditions of Corollary \ref{cor:thinned_polytope_volume} are satisfied for $\rho$ on $Q \subset \RR^{d-1}$. Applying it gives
\begin{equation}\label{eq:projected_expectation_bound}
\EE[\rho(Q \setminus Q_n)] \le \frac{C'}{c'} \frac{F(P)}{d^{d-2}(d-2)!} \frac{\log^{d-2}n}{n} + \bigO\left(\frac{\log^{d-3}(n)\log{\log{n}}}{n}\right).
\end{equation}
Combining inequality \eqref{eq:projected_expectation_bound} with the equality of expectations \eqref{eq:expectation_equality}, we obtain the desired result for the spherical case:
\[
\EE[\nu(P \setminus P_n)] \le \frac{C'}{c'}  \frac{F(P)}{(d+1)^{d-1}(d-1)!} \frac{\log^{d-1}n}{n} + \bigO\left(\frac{\log^{d-3}(n)\log{\log{n}}}{n}\right).
\]
\end{proof}

\begin{remark}
As can be seen in the proof of Corollary \ref{cor:spherical_polytope_volume}, the constant $K=\min_{i \in [N]}\left(\frac{\max_{y \in \mathcal{C}_i} y^{\top}e_i}{\min_{y \in \mathcal{C}_i} y^{\top}e_i}\right)^{d}$ that appears in the statement of Theorem \ref{thm:d-dimensional-regret}(b) comes from the bounds on the conditional densities after applying the gnomonic projections $g_{e_i}$ to each cell $\mathcal{C}_i.$ The $e_i$, $i \in [N]$ can be chosen to minimize $K$. In Fig. \ref{fig:gnomonic_projection}, we used $e_i=s_i,$ but this choice is not optimal in general.
\end{remark}

\subsubsection{Proof of Corollary \ref{cor:explicit_regret_bound}}\label{subsec:flag_bounds}

The bound on the number of flags that we derive is the following.

\begin{lemma}\label{lem:flag_bound_merged} The total number of flags $\sum_{i=1}^{N}F(\mathcal{C}_i)$ in the Voronoi tessellation of $\mathcal{E}$ is bounded as
\begin{itemize}
    \item  $\frac{8}{3}Nd!\left(\frac{2e(N+2d)}{d-1}\right)^{d/2}$ if $\mathcal{E}=\mathcal{I}^d$
    \item $4N(d-1)!\left(\frac{2eN}{d-2}\right)^{(d-1)/2}$ if $\mathcal{E}=\Sph^{d-1}$
\end{itemize}
\end{lemma}

We first show the following result, which is essentially a dual version of the upper bound from Section 2 of \citet{besau2018}. 
\begin{lemma}\label{lem:flag_bound_intermediary}
    If a convex polytope $P$ of $\mathbb{R}^d$ has $k$ faces of dimension $d-1$, $$F(P) \le \left\{\begin{matrix}
(2\ell)!\frac{k}{k-\ell}\binom{k-\ell}{\ell} \text{ if } d=2\ell \\ 2(2\ell+1)!\binom{k-\ell-1}{\ell} \text{ if } d=2\ell+1. 
\end{matrix}\right.$$
\end{lemma}

\begin{proof}[Proof of Lemma \ref{lem:flag_bound_intermediary}]
    Consider the dual polytope $P^o$ of $P$, which has $k$ vertices and the same number of flags as $P$. By the upper bound theorem (\citet{McMullen1970}, or Theorem 6.6 in  \citet{Billera2000}), $F(P^o) \le F(C_d(k)) $ where $C_d(k)$ is the cyclic polytope in $\mathbb{R}^d$ with $k$ vertices. Since $C_d(k)$ is simplicial, $F(C_d(k))=d!f_{d-1}(C_d(k))$,
    where $f_{d-1}(C_d(k))$ is the number of $(d-1)$-dimensional faces of $C_d(k)$. Hence, we have $$F(P) \le d!f_{d-1}(C_d(k)).$$ An exact formula for $f_{d-1}(C_d(k))$ is  known and is given by  $$f_{d-1}(C_d(k)) = \left\{\begin{matrix}
\frac{k}{k-\ell}\binom{k-\ell}{\ell} \text{ if } d=2\ell \\ 2\binom{k-\ell-1}{\ell} \text{ if } d=2\ell+1
\end{matrix}\right.$$ (see \citet{Ziegler1995}, p.25). This concludes the proof.
\end{proof}
\begin{proof}[Proof of Lemma \ref{lem:flag_bound_merged}]

First consider the case  $\mathcal{E}=\mathcal{I}^d$. As $\mathcal{C}_i$ is a convex polytope formed by at most $N-1+2d$ hyperplanes/facets of the hypercube $\mathcal{E}$ ($N-1$ for the hyperplanes, $2d$ for the facets of the hypercube), it has at most $N-1+2d$ faces of dimension $d-1$. The previous lemma applies and it gives $$\sum_{i=1}^{N}F(\mathcal{C}_i) \le \left\{\begin{matrix}
N(2\ell)!\frac{N+2d-1}{N+2d-1-\ell}\binom{N+2d-1-\ell}{\ell} \text{ if } d=2\ell \\ 2N(2\ell+1)!\binom{N+2d-1-\ell-1}{\ell} \text{ if } d=2\ell+1. 
\end{matrix}\right.$$
In particular, using $\binom{n}{k} \le (\frac{en}{k})^k$, $\sum_{i=1}^{N}F(\mathcal{C}_i) \le 2Nd!\frac{N+2d-1}{N+2d-1-\lfloor d/2 \rfloor}(\frac{e(N+2d-1-\lfloor d/2 \rfloor)}{\lfloor d/2 \rfloor}) ^{\lfloor d/2 \rfloor}.$ Since $\lfloor d/2 \rfloor \le \frac{N+2d-1}{4}$, we have $N+2d-1-\lfloor d/2 \rfloor \ge \frac{3}{4}(N+2d-1)$, so $\frac{N+2d-1}{N+2d-1-\lfloor d/2 \rfloor} \le 4/3,$ and $$\sum_{i=1}^{N}F(\mathcal{C}_i) \le  \frac{8}{3}Nd!\left(\frac{2e(N+2d)}{d-1}\right)^{d/2}.$$

When $\mathcal{E}=\Sph^{d-1}$, the argument is slightly simpler since only the $N-1$ boundaries of the Voronoi tessellation can form facets of the cells (the space $\Sph^{d-1}$ does not have boundaries to consider that might increase their number). More precisely, $\mathcal{C}_i$ is a spherically convex polytope formed by at most $N-1$ great spheres of $\mathcal{E}$, so it has at most $N-1$ faces of dimension $d-2$. We first remark that under the open halfsphere assumption, each cell has at least $d$ facets (Theorem 6.3.7 in \citet{ratcliffe2019}). Consequently, we must have $N \ge d + 1$. A spherical analogue of Lemma \ref{lem:flag_bound_intermediary} applies and it gives $$\sum_{i=1}^{N}F(\mathcal{C}_i) \le \left\{\begin{matrix}
N(2\ell)!\frac{N-1}{N-1-\ell}\binom{N-1-\ell}{\ell} \text{ if } d-1=2\ell \\ 2N(2\ell+1)!\binom{N-1-\ell-1}{\ell} \text{ if } d-1=2\ell+1 
\end{matrix}\right.$$
(note that the binomial coefficients are well defined since the open halfsphere assumption implies $N \ge d+1$).
In particular, using $\binom{n}{k} \le (\frac{en}{k})^k$, $\sum_{i=1}^{N}F(\mathcal{C}_i) \le 2N(d-1)!\frac{N-1}{N-1-\lfloor (d-1)/2 \rfloor}(\frac{e(N-1-\lfloor (d-1)/2 \rfloor)}{\lfloor (d-1)/2 \rfloor}) ^{\lfloor (d-1)/2 \rfloor}.$ Since $\lfloor (d-1)/2 \rfloor \le \frac{N-1}{2}$, we have $N-1-\lfloor (d-1)/2 \rfloor \ge \frac{1}{2}(N-1)$, so $\frac{N-1}{N-1-\lfloor (d-1)/2 \rfloor} \le 2,$ and $$\sum_{i=1}^{N}F(\mathcal{C}_i) \le  4N(d-1)!\left(\frac{2eN}{d-2}\right)^{(d-1)/2}.$$
\end{proof}
Injecting the result of Lemma \ref{lem:flag_bound_merged} in the regret bound of Theorem \ref{thm:d-dimensional-regret} directly gives Corollary \ref{cor:explicit_regret_bound}.

\subsubsection{Proof of Theorem \ref{thm:1-dimensional-regret-improved}}\label{subsec:proof_ub_dim_1}
We first show the following result.
\begin{proposition}\label{prop:reduction-to-uniform}
    Consider $n \ge 1$ points $p_1,\dots,p_n$ sampled independently at random from a distribution $\nu$  supported on an interval $P \subseteq \mathbb{R}$ with continuous cumulative distribution function. Sort the $p_i$ in increasing order as $p_{(1)} \le \dots \le p_{(n)}$ and denote  $P_n=[p_{(1)},p_{(n)}]$ their convex hull. We have $$\mathbb{E}[\nu({P \setminus P_n})]=\frac{2}{n+1}.$$
   
\end{proposition}

 \begin{proof}[Proof  of Proposition \ref{prop:reduction-to-uniform}]
 Consider $F_P$ the cumulative distribution function (c.d.f.) of $p_i$. Since $F_P$ is continuous, the $q_i:=F_P(p_i)$ are independent and uniformly distributed on $[0,1]$. Since $F_P$ is nondecreasing, we also have $q_{(1)} \le \dots \le q_{(n)}.$
Letting $Q_n=[q_{(1)},q_{(n)}]$ the convex hull of $\{q_i\}_{i=1}^{n}$, we thus have $F_P(P_n)=Q_n$ (the c.d.f. preserves convex hulls in dimension one).
Since $P$ is an interval, $P_n \subseteq P$, and \begin{align*}
\nu({P \setminus P_n})&=\nu(P)-\nu(P_n)\\
&=1-(F_P(p_{(n)})-F_P(p_{(1)})) \\
& = 1 - (q_{(n)}-q_{(1)}).
\end{align*}

The order statistics of uniform random variables on $[0,1]$ are well-known: $\EE[q_{(i)}]=\frac{i}{n+1},$ consequently $1-\EE[q_{(n)}-q_{(1)}]=1-\frac{n-1}{n+1}=\frac{2}{n+1}.$ \end{proof}
We can now apply Proposition \ref{prop:reduction-to-uniform}  to $\nu=\mu_{\mathcal{C}_i}$ for each $i \in [N]$ to derive the regret bound of Theorem \ref{thm:1-dimensional-regret-improved}.
\begin{proof}[Proof of Theorem \ref{thm:1-dimensional-regret-improved}]

Recall that $R_{\vhc}(T) = 2\sum_{t=1}^{T}\sum_{i=1}^{N} \mathbb{E}[\mu(\mathcal{C}_i \setminus \hat{\mathcal{C}}_{i,t})]$. Let  $n_i(t) \sim \Bin(t,\mu(\mathcal{C}_i))$ the number of points that landed in $\mathcal{C}_i$ up to time $t$. Conditionally to $n_i(t)=n$, $\hat{\mathcal{C}}_{i,t}$ is distributed as the convex hull $Q_n$ of $n$ points sampled at random within $\mathcal{C}_i$ with density $f_{\mathcal{C}_i}(x)=f_{\mu}(x)/\mu(\mathcal{C}_i)$. Thus
$\mathbb{E}[\mu(\mathcal{C}_i \setminus \hat{\mathcal{C}}_{i,t}) | n_i(t)]=\mathbb{E}[\mu(\mathcal{C}_i \setminus Q_{n_i(t)})|n_i(t)]= \mu(\mathcal{C}_i) \mathbb{E}[\mu_{\mathcal{C}_i}(\mathcal{C}_i \setminus Q_{n_i(t)})|n_i(t)]$. Applying Proposition \ref{prop:reduction-to-uniform} and noting that for $n=0$, $2/(n+1)$ bounds from above the L.H.S. of Proposition \ref{prop:reduction-to-uniform}, we obtain $$\mathbb{E}[\mu(\mathcal{C}_i \setminus \hat{\mathcal{C}}_{i,t}) | n_i(t)]\le \frac{2}{n_i(t)+1}\mu(\mathcal{C}_i) .$$ This gives $R_{\vhc}(T) \le 4\sum_{t=1}^{T}\sum_{i=1}^{N}\mathbb{E}[\frac{1}{n_i(t)+1}]\mu(\mathcal{C}_i)$.
If $X \sim \Bin(n,p)$, \begin{align*}\mathbb{E}\left[\frac{1}{X+1}\right]&=\sum_{k=0}^{n}\frac{1}{k+1}{n \choose k}p^k(1-p)^{n-k} \\
&=\frac{1}{(n+1)p} \sum_{k=0}^{n}{n+1 \choose k+1}p^{k+1}(1-p)^{n-k} \\
&=\frac{1}{(n+1)p} \sum_{k=1}^{n+1}{n+1 \choose k}p^{k}(1-p)^{n+1-k} \\
&=\frac{1}{(n+1)p}(1-(1-p)^{n+1}) \\
 &\le \frac{1}{(n+1)p}\end{align*}
thus $\mathbb{E}[\frac{1}{n_i(t)+1}] \le \frac{1}{(t+1)\mu(\mathcal{C}_i)}$. This yields a regret of \begin{align*}
R_{\vhc}(T) &\le 2(\beta-\alpha)\sum_{t=1}^{T}\frac{1}{t+1}\sum_{i=1}^{N}\frac{\mu(\mathcal{C}_i) }{\mu(\mathcal{C}_i)} \\
&\le 2(\beta-\alpha)N\log\left(T+1\right).
\end{align*}
\end{proof}

\begin{remark}\label{remark:exact_regret}
By carefully inspecting the proof, one can notice that  \begin{align*}
    \mathbb{E}\left[\mu(\mathcal{C}_i \setminus \hat{\mathcal{C}}_{i,t})\right] &= \mathbb{E}\left[\mathbb{E}\left[\mu(\mathcal{C}_i \setminus \hat{\mathcal{C}}_{i,t}) | n_i(t)\right]\right] \\&=\mathbb{E}\left[\mathbbm{1}\{n_i(t)=0\}+\frac{2}{n_i(t)+1}\mathbbm{1}\{n_i(t) \ne 0\}\right] \\
    &=(1-\mu(\mathcal{C}_i))^t+2\left[\frac{(1-(1-\mu(\mathcal{C}_i))^{t+1}}{(t+1)\mu(\mathcal{C}_i)}-(1-\mu(\mathcal{C}_i))^t\right] \\
    &= \frac{2(1-(1-\mu(\mathcal{C}_i))^{t+1}}{(t+1)\mu(\mathcal{C}_i)}-(1-\mu(\mathcal{C}_i))^t.
\end{align*}
An \emph{exact} expression of the regret of \vhc\ can then be obtained: \begin{align}\label{eq:exact_regret}R_{\vhc}(T) &= (\beta-\alpha) \sum_{i=1}^N \sum_{t=1}^T \left[ \frac{2(1-(1-\mu(\mathcal{C}_i))^{t+1})}{t+1} - \mu(\mathcal{C}_i)(1-\mu(\mathcal{C}_i))^t \right] \nonumber \\
&=(\beta-\alpha)\left(2N\left(\sum_{t=1}^{T}\frac{1}{t+1}\right) -2\sum_{i=1}^{N}\sum_{t=1}^{T}\frac{(1-\mu(\mathcal{C}_i))^{t+1}}{t+1}-\sum_{i=1}^{N}(1-\mu(\mathcal{C}_i)(1-(1-\mu(\mathcal{C}_i))^T)\right) \nonumber \\
&=(\beta-\alpha)\left(2N\left(\sum_{t=1}^{T}\frac{1}{t+1}\right) -2\sum_{i=1}^{N}\sum_{t=1}^{T}\frac{(1-\mu(\mathcal{C}_i))^{t+1}}{t+1}-\left[N-1-\sum_{i=1}^{N}(1-\mu(\mathcal{C}_i)^{T+1}\right]\right).\end{align}
\end{remark}

%% file: b.appendix/3.proofs_lowerbound.tex

\subsubsection{Proof of Theorem \ref{proposition:Lower bound on the minimax regret}}\label{app:lower_bound}
The proof strategy of Proposition \ref{proposition:Lower bound on the minimax regret} is to first derive a regret lower bound on the Bayesian regret for an adequately chosen prior on the set of $N$ seeds $\theta=(s_1,\dots,s_N) \in [0,1]^N$. The minimax regret bound then follows as an immediate corollary. 
First, note that given $\theta$, the correct label to $q_t \in [0,1]$ is the index of the closest seed:
\begin{align*}
	i_t(\theta) = \arg\min_{i \in [N]} \vert q_t - s_i\vert.
\end{align*}
At any time $t$, the learner may ask the expert and obtain the correct label $i_t(\theta)$. Let  $e_t = 1$ if the expert is consulted, and $e_t=0$ otherwise. 
We define $v_t = e_t  i_t(\theta)$ the expert output at time $t$.  
The learner attempts to guess the label $i_t(\theta)$ with some estimate $\hat{\imath}_t$ using the past observations.
We use the notation $v^t = (v_1,...,v_t)$, $e^t = (e_1,...,e_t)$ and $q^t = (q_1,...,q_t)$.

An algorithm is therefore comprised of two components:
\begin{enumerate}[label=(\roman*)]
	\item Estimation: $\hat{\imath}_t$ which can be an arbitrary function of $q_t,v^{t}$
	\item Consultation: $e_t$ which can be an arbitrary function of $q_t,v^{t-1}$ 
\end{enumerate}

\paragraph{Optimality of MAP estimation.}\label{sec:Optimality of MAP estimation}
Consider $\theta$ drawn according to a prior distribution with density $p_\theta$, and define the Bayesian regret
\begin{align*}
	\bar{R}(T):=\int_{[0,1]^N} R(T,\theta) p_\theta(\theta) d\theta  = (\beta-\gamma)\sum_{t=1}^{T} \mathbb{P}(\hat{\imath}_t\ne i_t(\theta)) + (\beta-\alpha)\sum_{t=1}^{T}\mathbb{E}(e_t) 
\end{align*} 
where we emphasize that $\theta$ is random and drawn according to $p_{\theta}$, and we omit the subscript $\pi$ in the regret notation for visual clarity.

\begin{proposition}\label{prop:map_optimality} The Maximum A Posteriori (MAP) estimator \begin{align*}
	\hat{\imath}_t \in \arg\max_{i \in [N]} \mathbb{P}( i_t(\theta)=i | q_t,v^{t-1} )	
\end{align*}
minimizes $\bar{R}(T).$
\end{proposition}

\begin{proof}[Proof of Proposition \ref{prop:map_optimality}]
Let us prove that there exists an optimal Bayesian algorithm such that $\hat{\imath}_t$ is the MAP estimator:
\begin{align*}
	\hat{\imath}_t \in \arg\max_{i \in [N]} \mathbb{P}( i_t(\theta)=i | q_t,v^{t-1} )	
\end{align*}
Now consider an arbitrary algorithm, its Bayesian regret is
\begin{align*}
	\bar{R}(T) 
	&= (\beta-\gamma)\sum_{t=1}^{T} \left[ 1 - \sum_{i=1}^{N} \mathbb{P}( \hat{\imath}_{t} = i ,  i_t(\theta)=i ) \right] + (\beta-\alpha)\sum_{t=1}^{T}\mathbb{E}(e_t) \\
	&= (\beta-\gamma)\sum_{t=1}^{T} \left[ 1 - \sum_{i=1}^{N} \mathbb{E}\left( \mathbb{P}( \hat{\imath}_{t} = i ,  i_t(\theta)=i | q_t,v^{t-1} )\right) \right] + (\beta-\alpha)\sum_{t=1}^{T}\mathbb{E}(e_t)  \\
	&= (\beta-\gamma)\sum_{t=1}^{T} \left[ 1 - \sum_{i=1}^{N} \mathbb{E}\left( \mathbbm{1}\{\hat{\imath}_{t} = i\} \mathbb{P}(  i_t(\theta)=i | q_t,v^{t-1} ) \right)\right] + (\beta-\alpha)\sum_{t=1}^{T}\mathbb{E}(e_t)  \\
	&\ge (\beta-\gamma)\sum_{t=1}^{T} \left[ 1 - \mathbb{E}\left( \max_{i \in [N]}  \mathbb{P}(  i_t(\theta)=i | q_t,v^{t-1} ) \right) \right] + (\beta-\alpha)\sum_{t=1}^{T}\mathbb{E}(e_t) 
\end{align*}
where we used the tower property of expectations and the fact that $\hat{\imath}_{t} \in [N]$, and where the lower bound is attained if $\hat{\imath}_{t}$ is the MAP 
	$\hat{\imath}_t \in \arg\max_{i \in [N]} \mathbb{P}( i_t(\theta)=i| q_t,v^{t-1} )$.
\end{proof}
We have proven that, given any algorithm, we may modify it by replacing its estimation part by the MAP estimator and doing so always results in lower Bayesian regret. We emphasize that changing the estimation part while keeping the consultation part constant does not change $q^t$ nor does it change $v^t$, so that it has no impact on the information available at time $t$ in order to perform the estimation. 

\paragraph{Some properties of the MAP.}\label{sec:Some properties of the MAP}

Note that if the expert is consulted  ($e_t=1$), then the MAP is simply $\hat{\imath}_t = i_t(\theta)$ since the correct label is available. 
Furthermore, if $q_t \in \conv( \{ q_{t_j}: e_{t_j} = 1, v_{t_j} = i \})$ then the MAP is $\hat{\imath}_t = i$, since the only possible label is $i$. 

Fig.~\ref{fig:figure_lower_bound} illustrates the construction behind the lower bound in dimension one.
\begin{figure}[htpb]
	\centering
	\includegraphics[width=0.8\textwidth]{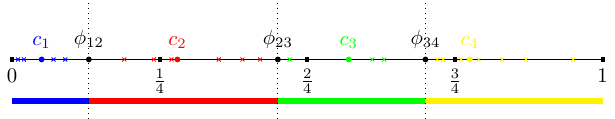}
	\caption{Lower bound construction in dimension one, $\phi_{ij}=(s_i+s_j)/2$}
	\label{fig:figure_lower_bound}
\end{figure}

\begin{proof}[Proof of Theorem \ref{proposition:Lower bound on the minimax regret}]
Consider $2 \le t \le T$ (the theorem is trivial if $T=1$) and the following prior on $\theta$: $s_1,...,s_N$ are drawn independently and $s_i$ is uniformly distributed on the interval $I_i = [(i-1)/N,i/N]$. 
Either the expert is consulted, so that $e_t = 1$ and a regret of $\beta-\alpha$ is incurred, or a regret at least equal to $\beta-\gamma$ times the error rate of the MAP estimator is incurred. 
Since $\gamma \le \alpha$, the instantaneous regret $1-r(t)$ is lower bounded as: 
\begin{align*}
	1-r(t) &\ge 
    (\beta-\gamma)\mathbb{P}(e_t = 0)   \mathbb{P}(\hat{\imath}_t \ne i_t(\theta) |  e_t = 0) + (\beta-\alpha)\mathbb{P}(e_t = 1)  \\
    &\ge (\beta-\alpha)\left(\mathbb{P}(e_t = 0)   \mathbb{P}(\hat{\imath}_t \ne i_t(\theta) |  e_t = 0) + 1-\mathbb{P}(e_t = 0)\right)   \\
    &\ge (\beta-\alpha)\mathbb{P}(\hat{\imath}_t \ne i_t(\theta) |  e_t = 0).
\end{align*}
We now focus on the error rate of the MAP knowing that $e_t=0$.
Denote by $l_{i,t} = \min \{ q_{s}: s \le t, e_{s} = 1, v_{s} = i \}$ for $2 \le i \le N$ and $l_{1,t} = 0$.
Similarly $r_{i,t} = \max \{ q_{s}: e_{s} = 1, v_{s} = i \}$ for $1 \le i \le N-1$ and $r_{N,t} = 0$.
Since $s_i \in I_i$ for all $i$ we have $s_i \le s_{i+1}$ for all $i$ and it follows that $l_{i,t} \le r_{i,t} \le l_{i+1,t}$ for all $i$. 
If $q_t \in \cup_{i \in [N]} [l_{i,t}, r_{i,t}]$ then $\hat{\imath}_t = i_t(\theta)$ so that the correct label is guessed with probability one. 
If $q_t \not\in \cup_{i \in [N]}[l_{i,t}, r_{i,t}]$ define $i$ such that $r_{i} \le q_t \le l_{i+1}$ so that there are only two possible labels: namely $i$ and $i+1$, and the learner must guess whether or not $q_t \le \phi_{i,i+1} =  (s_{i} + s_{i+1})/2$. 
Since $s_i$ and $s_{i+1}$ are uniformly distributed on $I_i$ and $I_{i+1}$ we have $\phi_{i,i+1} = (i+z-1/2)/N$ in distribution where $z$ follows the Irwin-Hall distribution with two parameters, i.e., $z \in [0,2]$ is distributed as the sum of two i.i.d uniformly distributed random variables over $[0,1]$ with density $\min(z,2-z)$.
To ease notation we define $h(z) = \min(z,2-z)$ and $I_i(a,b) = \int_{N a - i+1/2}^{N b - i+1/2} h(z) dz$.

Therefore
\begin{align*}
	\mathbb{P}(  i_t(\theta) = i | q_t,v_{t-1},e_t=0 )  
	&= \mathbb{P}( q_t \le \phi_{i,i+1} | q_t,v_{t-1} ) \\
	&= \mathbb{P}( q_t \le \phi_{i,i+1} | \phi_{i,i+1} \in [r_{i,t},l_{i+1,t}] ) \\
	&= \mathbb{P}( q_t \le (i+z-1/2)/N | (i+z-1/2)/N \in [r_{i,t},l_{i+1,t}] ) \\
	&= \mathbb{P}( N q_t - i+1/2 \le z | z \in [N r_{i,t}- i+1/2 ,N l_{i+1,t}- i+1/2 ] ) \\
	&= \frac{ I_i(q_t, l_{i+1,t}) }{ I_i( r_{i,t}, l_{i+1,t}) } 
\end{align*}
and similarly
\begin{align*}
	\mathbb{P}(   i_t(\theta) = i-1 | q_t,v_{t-1},e_t=0 ) 
	&= \frac{I_i(r_{i,t},q_t)}{ I_i( r_{i,t},l_{i+1,t}) }
\end{align*}
Furthermore, the error rate of the MAP knowing $q_t$ is given by
\begin{align*}
	\mathbb{P}(\hat{\imath}_t \ne  i_t(\theta) | q_t,v_{t-1}, e_t = 0) &= \min_{j \in \{i,i+1\}} \mathbb{P}(   i_t(\theta) = j | q_t,v_{t-1},e_t=0 )  \\
									      &= \frac{\min( I_i(q_t,l_{i+1,t}),  I_i(r_{i,t},q_t) )} {I_i( r_{i,t}, l_{i+1,t})} 
\end{align*}
and the error rate is found by integrating this function over $[0,1]$ since $q_t$ is drawn uniformly at random:
\begin{align*}
	\mathbb{P}(\hat{\imath}_t \ne  i_t(\theta) | e_t = 0) 
&=  \sum_{i=1}^{N-1}  \int_{0}^1 \mathbb{E} \left( \frac{\min( I_i(q,l_{i+1,t}),  I_i(r_{i,t},q) )} {I_i( r_{i,t}, l_{i+1,t})}  \indic(q \in [r_{i,t}, l_{i+1,t}]) \mathrm{d}q  \right).
\end{align*}
While this expression can be computed in closed form, it is unwieldy, and we will now lower bound it. Define the median point $m_t \in [r_{i,t}, l_{i+1,t}]$ such that 
\begin{align*}
	I_i(m_t,l_{i+1,t}) = I_i(r_{i,t},m_t) = \frac{1}{2} I_i(r_{i,t},l_{i+1,t})
\end{align*}
Consider $q \le m_t$, we have
\begin{align*}
	\min( I_i(q,l_{i+1,t}),  I_i(r_{i,t},q) ) &=  I_i(r_{i,t},q) \\
						  &= I_i(r_{i,t},m_t) - I_i(q,m_t) \\
						  &\ge \frac{1}{2} I_i(r_{i,t},l_{i+1,t}) - 2 |m_t-q| \\
\end{align*}
where we used the definition of $m_t$ and the fact that $h(z) \le 2$. Similarly let $q \ge m_t$, we have
\begin{align*}
	\min( I_i(q,l_{i+1,t}),  I_i(r_{i,t},q) ) &=  I_i(q,l_{i+1,t},q) \\
						  &= I_i(m_t,l_{i+1,t}) - I_i(m_t,q) \\
						  &\ge \frac{1}{2} I_i(r_{i,t},l_{i+1,t}) - 2 |m_t-q| \\
\end{align*}
Putting both cases together, and using the fact that the expression is always positive:
\begin{align*}
	\frac{\min( I_i(q,l_{i+1,t}),  I_i(r_{i,t},q))}{I_i( r_{i,t}, l_{i+1,t})}  \ge \max\left(0, \frac{1}{2} - \frac{ 2 |m_t-q|}{ I_i(r_{i,t},l_{i+1,t})}\right).
\end{align*}
Integrating this bound over $q$,
\begin{align*}
	\int_{0}^1 \frac{\min( I_i(q,l_{i+1,t}),  I_i(r_{i,t},q) )} {I_i( r_{i,t}, l_{i+1,t})}  \indic(q \in [r_{i,t}, l_{i+1,t}]) \mathrm{d}q
	&\ge \int_{0}^1  \max\left(0, \frac{1}{2} - \frac{ 2 |m_t-q|}{ I_i(r_{i,t},l_{i+1,t})}\right) \mathrm{d}q \\
	&= \frac{1}{4} I_i(r_{i,t},l_{i+1,t})
\end{align*}
so the error rate is lower bounded by 
\begin{align*}
	\mathbb{P}(\hat{\imath}_t \ne  i_t(\theta) | e_t = 0) 
&\ge  \frac{1}{4} \sum_{i=1}^{N-1}  \mathbb{E} \left(  I_i(r_{i,t},l_{i+1,t}) \right).
\end{align*}

Consider $\epsilon_t = 1/t \le 1/\sqrt{2}$ since $t \ge 2$ and consider the event $E_{i,t}$ that $\phi_{i,i+1} \in [(i-1/2+1/\sqrt{2})/N, (i-1/2 + 2 - 1/\sqrt{2})/N]$, and that $q_s \not\in [\phi_{i,i+1} - \epsilon_t/2,\phi_{i,i+1} + \epsilon_t/2]$ for all $s \le t$. 
This event has probability 
\begin{align*}
	\mathbb{P}(E_{i,t}) &= \int_{ (i-1/2+1/\sqrt{2})/N}^{(i-1/2 + 2 - 1/\sqrt{2})/N}  \mathbb{P}(  q_s \not\in [p - \epsilon_t/2,p + \epsilon_t/2] s \le t  | \phi_{i,i+1} = p) \mathbb{P}(\phi_{i,i+1} = p) dp \\
	&=  I(1/\sqrt{2},2-1/\sqrt{2}) (1-\epsilon_t)^t \\
	&=  \frac{1}{2} (1-\epsilon_t)^t \\
\end{align*}
and when it occurs $r_{i,t} \le \phi_{i,i+1} - \epsilon_t/2 \le \phi_{i,i+1} - \epsilon_t/2 \le l_{i+1,t}$ and in turn $1/\sqrt{8}\le N r_{i,t} - (i-1/2) \le N l_{i+1,t} - (i-1/2) \le 2 - 1/\sqrt{8}$, which then implies that $I_i(r_{i,t},l_{i+1,t} \ge (l_{i+1,t} - r_{i,t}) \min_{z \in [r_{i,t},l_{i+1,t}]} h(z) \ge \epsilon_t /\sqrt{8}$. Hence taking expectations
\begin{align*}
	\mathbb{E}\left(  I_i(r_{i,t},l_{i+1,t}) \right) &\ge \frac{\epsilon_t}{2\sqrt{2}} \mathbb{P}(E_{i,t}) \\
							  &=\frac{1}{t 4 \sqrt{2}} \left(1- \frac{1}{t}\right)^t
\end{align*}
Putting everything together gives a lower bound on the Bayesian regret
\begin{align*}
	\int_{[0,1]^N} R(T,\theta) p_\theta(\theta) d\theta  
	&\ge (\beta-\alpha)\sum_{t=1}^{T} \mathbb{P}(\hat{\imath}_t \ne  i_t(\theta) | e_t = 0) \\
	&\ge (\beta-\alpha)\frac{N-1}{16 \sqrt{2}} \sum_{t=2}^{T} \frac{1}{t} \left(1- \frac{1}{t}\right)^t \\
    &\ge  (\beta-\alpha)\frac{N-1}{64 \sqrt{2}}\log\left(\frac{T+1}{2}\right)
\end{align*}

where we used $(1-\frac{1}{t})^t \ge (1-\frac{1}{2})^2$ for $t \ge 2$ and $\sum_{t=2}^{T}\frac{1}{t} \ge \log\left(\frac{T+1}{2}\right)$ in the last inequality.

We have proven that, for any $T \ge 1$, the Bayesian regret of any algorithm is lower bounded by $$(\beta-\alpha)\frac{N-1}{64 \sqrt{2}} \log\left(\frac{T+1}{2}\right)= \Omega( (N-1) \log{T}).$$
By corollary, the same bound holds for the minimax regret.
\end{proof}

\begin{remark}[Adversarial Setting]\label{rem:adversarial} As mentioned in Section \ref{sec:models}, any algorithm must suffer linear regret in an adversarial setting. consider dimension $d=1$, the interval $\mathcal{E}=[0,1]$  and $N=2$ labels. Denote by $\phi$ the boundary between the two labels so that 
$\mathcal{C}_1=[0,\phi]$ and $\mathcal{C}_2=[\phi,1]$, and the adversary selects the query points as follows: $$q_t=\frac{l_t+r_t}{2},$$ obeying the following recursion:
\begin{enumerate}[label=(\roman*)]
\item If $t=1$ then $[l_1,r_1]=[0,1]$
\item If $t \ge 2$ and $\phi \le q_t$ then $[l_{t+1},r_{t+1}]=[l_t,q_t]$
\item If $t \ge 2$ and $\phi > q_t$ then $[l_{t+1},r_{t+1}]=[q_t,r_t]$
\end{enumerate}
The query points can be seen as a binary search procedure over 
$[0,1]$ in an attempt to find $\phi$. Then, for every time $t$:
\begin{enumerate}[label=(\roman*)]
    \item either the learner calls the expert, which incurs a cost of $\alpha$
    \item otherwise the learner must attempt to guess whether or not $\phi \le q_t$ and the optimal guessing strategy is the MAP estimator, as explained in the proof of Theorem \ref{proposition:Lower bound on the minimax regret}. 
\end{enumerate}

By construction, even in the most favorable case where the learner called the expert at all times from time $1$ to time $t-1$, then the error rate of the MAP estimator is exactly $1/2$, because the a posteriori distribution of $\phi$ is uniform over $[l_t,r_t]$ and the a posteriori probability of $\phi \le q_t$
and  $\phi > q_t$ are equal, i.e., both labels are equiprobable. In other words, the adversary can always make sure that the MAP does not perform better than guessing uniformly at random. This incurs an average cost of guessing of $(\beta+\gamma)/2$.

It is noted that the adversary is oblivious, so that it does not adapt to the decisions selected by the learner, and the regret is indeed linear.
\end{remark}

%% file: b.appendix/4.proofs_ETC.tex
\subsection{Proofs of Section \ref{sec:etc_algo}}\label{sec:etc_algo_proof}

In this section, we present a more general version of Theorem \ref{thm:ETc_regret} along with its proof. We recall that the expert labeling policy is given by the Voronoi tessellation with seeds $s_i$, as described in Section \ref{sec:voronoi}. We show that in the regime $T \le e^d$, if the queries are distributed as a subgaussian mixture with sufficiently separated centers, a simple Center-based Classifier (\extc) can achieve $N\log{N}$ regret after $T$ rounds when the mixture is homogeneous. \extc's general regret bound holds under the following condition, which implies \ref{as:density}(ii).

\begin{assumption}\label{as:subgaussian}
The distribution $\mu$ of the queries is given by the following mixture model:  $i \in [N]$ is chosen with probability $p_{i}$, then $q_t$ is sampled from $s_i+w_i$  where $s_i \in \mathbb{R}^d$ is the component center and $w_i$ is subgaussian with parameter $\sigma>0$: for all $\lambda \in \RR^d$, $$\mathbb{E}[\exp(\lambda^{\top}w_i)] \le \exp\left(\frac{\sigma^2\Vert \lambda \Vert_2^2}{2} \right).$$
Furthermore, the mixture centers are sufficiently separated: $$\delta^2_{\min} \ge c\sigma^2d$$ where $\delta_{\min}=\min_{i \ne j} \Vert s_i - s_j \Vert_2$ and $c \ge 32$.
\end{assumption}

On $\mathcal{E}=\mathcal{I}^d$, Assumption \ref{as:subgaussian} is for instance verified if  $w_i \in [-L,L]^d$ for some $L^2 \le \delta_{\min}^2/(cd)$. Note that we always have $\delta_{\min}^2 \le d$ on $\mathcal{I}^d$, so  the latter requires $L^2 \le 1/c.$

Below, we state our general upper bound on \extc's regret, from which Theorem \ref{thm:ETc_regret} directly follows. 

\begin{theorem}\label{thm:ETc_regret} Under Assumption \ref{as:subgaussian}, the regret of \extc\ is bounded as \begin{align*}R_{\extc}(T) &\le \frac{(\beta-\alpha)(\log{N}+1)}{p_{\min}}\left(1+\frac{192(d+2\log{T})}{cd}\right) +(2\beta-\gamma-\alpha)N+(\beta-\gamma)T(e^{-\frac{c-32}{48}d}+Te^{-\frac{c-8}{12}d}).\end{align*} 

In particular, if $\delta_{\min}^2 \ge 80\sigma^2d$ and $T \le e^d,$
\begin{align*}R_{\extc}(T) &\le \frac{41}{5}\frac{(\beta-\alpha)(\log{N}+1)}{p_{\min}}+2(\beta-\gamma)(N+1).\end{align*} 
\end{theorem}

We denote by $j_t$ the \emph{generative} label of $q_t$, i.e., $q_t=s_{j_t}+w_{j_t}$. For the analysis, it is easier to work under the assumption that the expert provides the generative labels $j_t$ instead of the Voronoi labels $i_t$. We will make use of the following lemma. 
\begin{lemma}
    Under assumption \ref{as:subgaussian}, the event $E:= \{\forall  t \in [T]  \, i_t = j_t\}$ occurs with probability $\mathbb{P}(E) \ge 1-  Te^{-\frac{c-8}{12}d}.$
\end{lemma}
\begin{proof}
$E^c$ occurs with probability \begin{align*}
\mathbb{P}(E^c) &\le \sum_{t=1}^{T} \mathbb{P}(\Vert q_t - s_{i_t} \Vert_2 \le \Vert q_t-s_{j_t}\Vert_2 ) \\
&\le \sum_{t=1}^{T} \mathbb{P}(\Vert w_{j_t} + s_{j_t} - s_{i_t} \Vert_2 \le \Vert w_{j_t}\Vert_2 ) \\
&\le \sum_{t=1}^{T}\mathbb{P}(\Vert w_{j_t} \Vert_2 \ge \delta_{\min}/2) \\
&\le \sum_{t=1}^{T}\sum_{j=1}^{N}\mathbb{P}(j_t=j)\mathbb{P}(\Vert w_j/\sigma\Vert_2^2 \ge \frac{\delta_{\min}^2}{4\sigma^2})
\end{align*}
By Lemma \ref{lem:hsu_subgaussian}, for any $j \in [N]$ and $\eta >0$, $$\mathbb{P}(\Vert w_j/\sigma\Vert_2^2 \ge d+2\sqrt{d\eta}+2\eta) \le e^{-\eta}.$$

By the AM-GM inequality, $d+2\sqrt{d\eta}+2\eta \le d+(d+\eta)+2\eta=2d+3\eta.$
Hence $$\mathbb{P}(\Vert w_j/\sigma\Vert_2^2 \ge 2d+3\eta) \le e^{-\eta}.$$ We then pick $\eta$ such that $2d+3\eta=\frac{\delta_{\min}^2}{4\sigma^2}$. This requires  $\delta_{\min}^2 \ge 8\sigma^2d$ and gives $\eta=\frac{\delta_{\min}^2}{12\sigma^2}-\frac{2}{3}d$.

Finally, $\mathbb{P}(i_t \ne j_t) \le \exp\left(-(\frac{\delta_{\min}^2}{12\sigma^2}-\frac{2}{3}d)\right)$ so that $\mathbb{P}(E^c) \le Te^{-\frac{c-8}{12}d}$.

\end{proof}

\paragraph{Generative Setting vs. Voronoi Setting} This lemma ensures that for the sake of the analysis, we can assume that the ground-truth labels are the generative labels $j_t$. More precisely, denote by $\mathcal{R}_{\extc}(T)$ the (random) regret of the \extc\ algorithm, and  by $\mathcal{R}'_{\extc}(T)$ the random regret of the \extc\ algorithm if the ground-truth labels (and those provided by the expert) were the generative labels $j_t$. Under the event $E$, the Voronoi and generative labels match, so that $$\mathcal{R}_{\extc}(T)=\mathcal{R}_{\extc}(T)\mathbbm{1}\{E\}+\mathcal{R}_{\extc}(T)\mathbbm{1}\{E^c\} \le \mathcal{R}'_{\extc}(T)\mathbbm{1}\{E\}+(\beta-\gamma)T\mathbbm{1}\{E^c\}.$$
This means that under Assumption \ref{as:subgaussian},  the Voronoi regret of the \extc\ algorithm is bounded as \begin{equation}\label{eq:voronoi_to_gen}
R_{\extc}(T) \le \mathbb{E}[\mathcal{R}'_{\extc}(T)]+(\beta-\gamma)T^2e^{-\frac{c-8}{12}d}.
\end{equation}

Note that the regret of \extc\ in the generative setting can be  expressed as \begin{align*}R'_{\extc}(T):&=\mathbb{E}[\mathcal{R}'_{\extc}(T)] =(\beta-\alpha)\mathbb{E}[T_1]+(\beta-\gamma)\sum_{t=T_1+1}^{T}\mathbb{P}(\hat{\imath}_t \ne j_t) \\
\end{align*} The rest of the analysis focuses on bounding $\mathbb{E}[T_1]$ and $\mathbb{P}(\hat{\imath}_t \ne j_t)$ for $t > T_1$. We recall that $n_i(t)=\sum_{s=1}^{t}\mathbbm{1}\{i_s=i\}$ denotes the number of queries with true label $i$ up to round $t$. For the analysis, we use the definition $\hat{s}_i(t)=\frac{1}{n_i(t)}\sum_{s=1}^{t}q_s\mathbbm{1}\{i_s=i\}$ for any $t \in [T]$. Since the expert is always called in the first phase of \extc, this matches the definition of the estimator in Section \ref{sec:etc_algo}, that was only defined for $t \le T_1.$

\begin{lemma}\label{lem:first_phase2}

 Let $\delta \in (0,1)$ and consider the \extc\ algorithm whose first phase ends at $T_1=\min\{t \le T: \forall i \in [N] \, n_i(t) \ge \hat{m}(t)\}$ where $\hat{m}(t):=\frac{108\sigma^2}{\hat{\delta}^2_{\min}(t)}\left(d+\log(NT/\delta)\right)$. Under the generative setting, the event $$A(\delta):=\{\forall i \in [N], \forall t \in [T]\,:n_i(t)\Vert s_i - \hat{s}_i(t) \Vert^2_2 \le 3\sigma^2(d+\log(NT/\delta))\}$$ holds with probability larger than $1-\delta.$ Additionally, we have under $A(\delta):$
    \begin{enumerate}[label=(\roman*)]
        \item  If the first phase terminates, then for all $i \in [N],$ $\Vert s_i - \hat{s}_i(T_1) \Vert_2 \le \frac{\delta_{\min}}{4}$
        \item $T_1 \le \tau_m:=\min\{t \in \mathbb{N} : n_i(t) \ge m\}$
    \end{enumerate}
    where  $m:=\frac{192\sigma^2(d+\log(NT/\delta))}{\delta_{\min}^2}. $
\end{lemma}
 

\begin{proof}[Proof of Lemma \ref{lem:first_phase2}]

Let $t \in [T]$ and $i \in [N].$ Conditionally on $n_i(t)=n>0,$  $\hat{s}_i(t)-s_i$ is distributed as $\frac{1}{n}\sum_{j=1}^{n}x_{j,i}$ where the $x_{j,i}$ are independent and $\sigma$-subgaussian. Thus, $\frac{1}{n}\sum_{j=1}^{n}x_{j,i}$ is ($\sigma/\sqrt{n}$)-subgaussian. 
Lemma \ref{lem:hsu_subgaussian} gives, for any $\eta > 0$ and $n>0,$
\begin{align*}&\mathbb{P}(n_i(t)\Vert s_i-\hat{s}_i(t) \Vert_2^2 \ge \sigma^2(d+2\sqrt{d\eta}+2\eta)\, | \, n_i(t)=n) \\&= \mathbb{P}(\Vert \frac{1}{n}\sum_{j=1}^{n}x_{j,i} \Vert_2^2 \ge (\sigma^2/n)(d+2\sqrt{d\eta}+2\eta))  \\
&\le e^{-\eta}.\end{align*}

and the same bound holds unconditionally by taking the expectation (if $n=0$, the conditional bound trivially holds). 
By the AM-GM inequality, $d+2\sqrt{d\eta}+2\eta \le d+(d+\eta)+2\eta \le 3(d+\eta).$ 
Thus, for $\eta=\log(NT/\delta)>0$ we have, with probability larger than $1-\frac{\delta}{NT},$ 
$$n_i(t)\Vert s_i-\hat{s}_i(t) \Vert_2^2 \le 3\sigma^2(d+\log(NT/\delta)).$$

A union bound over $i \in [N]$ and $t \in [T]$ yields that 
the event $A(\delta)$ holds with probability $1-\delta.$

\textbf{(i)} Under $A(\delta)$, if $T_1< T$ then  $n_i(T_1) \ge \hat{m}(T_1)$ for all $i \in [N]$ by definition of $T_1$. This implies that for all $i \in [N]$, $\sqrt{\hat{m}(T_1)}\Vert s_i - \hat{s}_i(T_1) \Vert_2 \le \sigma\sqrt{3(d+\log(NT/\delta))},$
i.e., $$\Vert s_i - \hat{s}_i(T_1) \Vert_2 \le \frac{\hat{\delta}_{\min}(T_1)}{6}.$$

Additionally, under this event, $$\vert \Vert s_i-s_j \Vert_2 -\Vert\hat{s}_i(T_1) -\hat{s}_j(T_1) \Vert_2 \vert \le \Vert s_i-\hat{s}_i(T_1) \Vert_2+\Vert s_j-\hat{s}_j(T_1) \Vert_2 \le 2\frac{1}{6}\hat{\delta}_{\min}(T_1).$$
This implies that $$\Vert \hat{s}_i(T_1)-\hat{s}_j(T_1) \Vert_2 \le \Vert s_i -s_j \Vert_2+ \frac{1}{3}\hat{\delta}_{\min}(T_1),$$ and by taking the minimum on the right hand side,   $$\Vert \hat{s}_i(T_1)-\hat{s}_j(T_1) \Vert_2 \le \delta_{\min}+\frac{1}{3}\hat{\delta}_{\min}(T_1)$$ for some $i,j.$ Consequently, $\hat{\delta}_{\min}(T_1) \le \delta_{\min}+\frac{1}{3}\hat{\delta}_{\min}(T_1)$ which finally gives  $$\Vert s_i - \hat{s}_i(T_1) \Vert_2 \le \frac{1}{6}\hat{\delta}_{\min}(T_1) \le \frac{1}{4}\delta_{\min}.$$

\textbf{(ii)} under $A(\delta)$, either $\tau_m \ge T$  so that $T_1 \le T \le \tau_m, $ or $\tau_m <T$ and we have $n_i(\tau_m)  \ge m$ for all $i \in [N]$, hence $\max_{i \in [N]} \Vert s_i -\hat{s}_i(\tau_m) \Vert_2 \le \sigma\sqrt{\frac{3(d+\log(NT/\delta))}{m}}=\delta_{\min}/8$. Thus, $\hat{\delta}_{\min}(\tau_m) \ge \delta_{\min}-2(\delta_{\min}/8)=3\delta_{\min}/4$, i.e., $\hat{m}(\tau_m) \le m.$ Therefore, $n_i(\tau_m) \ge m \ge \hat{m}(\tau_m)$ so the stopping condition is verified at $\tau_m$, which yields $T_1 \le \tau_m.$ 

\end{proof}


\subsubsection{Bounding the First Phase Regret}

\begin{proposition}\label{prop:first_phase_regret_ETC}
    Under the generative setting, the first phase regret is bounded as \begin{align*}
    (\beta-\alpha)\mathbb{E}[T_1] 
    &\le  \frac{(\beta-\alpha)(\log{N}+1)}{p_{\min}}\left(1+\frac{192\sigma^2(d+\log(NT/\delta))}{\delta_{\min}^2}\right)+(\beta-\alpha)T\delta.\end{align*}
\end{proposition}

\begin{proof}[Proof of Proposition \ref{prop:first_phase_regret_ETC}]

By Lemma \ref{lem:first_phase2}, under $A(\delta),$ we have $T_1 \le \tau_m$. In general, $T_1 \le T$. Thus

\begin{align*}T_1&=T_1\mathbbm{1}\{A(\delta)\}+T_1\mathbbm{1}\{A(\delta)^c\} \\
&\le \tau_m+T\mathbbm{1}\{A(\delta)^c\}. \end{align*}
As $A(\delta)$ holds with probability larger than $1-\delta$, this yields \begin{align*}\mathbb{E}[T_1]
&\le \mathbb{E}[\tau_m]+T\delta.\end{align*}


To bound $\mathbb{E}[\tau_m]$, first note that it is trivially upper bounded by $\lceil m \rceil \mathbb{E}[\tau_1]$ where $\tau_1=\min\{t \in \mathbb{N} : n_i(t) \ge 1\}$ is the number of rounds required to observe each label once. Computing $\mathbb{E}[\tau_1]$ is known as a coupon-collector problem. Let $\tau_{1,i}$ denote the time to discover a $i$-th new label after $i-1$ are already found. Then $\tau_1=\sum_{i=1}^{N}\tau_{1,i}.$ The probability of observing a given label is bounded below by $p_{\min}=\min_{i \in [N]}p_i$. Thus the probability of observing a $i$-th new label is $q_i \ge (N-(i-1))p_{\min}.$ $\tau_{1,i}$ is geometric with success probability $q_i$, so $\mathbb{E}[\tau_{1,i}]= \frac{1}{q_i} \le \frac{1}{(N-(i-1))p_{\min}}$. Consequently, \begin{align*}\mathbb{E}[\tau_{1}] &\le \frac{1}{p_{\min}}\sum_{i=1}^{N}\frac{1}{N-(i-1)} \\
&=\frac{\sum_{i=1}^{N}1/i}{p_{\min}} \\
&\le \frac{\log{N}+1}{p_{\min}}. \end{align*}

This entails that $\mathbb{E}[\tau_m] \le \lceil m\rceil\frac{\log{N}+1}{p_{\min}}$, so that $$\mathbb{E}[T_1] \le (\lceil m\rceil/p_{\min})(\log{N}+1)+T\delta.$$


This yields \begin{align*}
 (\beta-\alpha)\mathbb{E}[T_1] &\le (\beta-\alpha)(\lceil m\rceil/p_{\min})(1+\log{N}) +(\beta-\alpha)T\delta\\
  &\le  \frac{(\beta-\alpha)(\log{N}+1)}{p_{\min}}\left(1+\frac{192\sigma^2(d+\log(NT/\delta))}{\delta_{\min}^2}\right)+(\beta-\alpha)T\delta.
\end{align*}

\end{proof}

\subsubsection{Bounding the Second Phase Regret}
\begin{proposition}\label{prop:second_phase_regret_etc}
Let $\delta \in (0,1)$. Under the generative setting, for $t > T_1$, the probability of making a wrong guess is bounded as $$\PP(\hat{\imath}_t \ne j_t) \le e^{-\frac{c-32}{48}d}+\delta.$$ Consequently, the second phase regret is bounded as
$$(\beta-\gamma)\sum_{t=T_1+1}^{T}\mathbb{P}(\hat{\imath}_t \ne j_t)\le (\beta-\gamma)T\left[e^{-\frac{c-32}{48}d}+\delta\right].$$\end{proposition}

\begin{proof}[Proof of Proposition \ref{prop:second_phase_regret_etc}.]
We first write $\mathbb{P}(\hat{\imath}_t \ne j_t) = \mathbb{P}(\hat{\imath}_t \ne j_t |A(\delta))\mathbb{P}(A(\delta))+\mathbb{P}(\hat{\imath}_t \ne j_t | A(\delta)^c)\mathbb{P}(A(\delta)^c) \le \mathbb{P}(\hat{\imath}_t \ne j_t |A(\delta))+\delta.$
Let $e_i:=s_i-\hat{s}_i(T_1)$. By Lemma \ref{lem:first_phase2}, under $A(\delta),$ we have $\Vert e_i \Vert_2 \le \delta_{\min}/4$ for all $i \in [N].$
We show that if $\Vert w_{j_t} \Vert_2 < \delta_{\min}/4$, we must have $\hat{\imath}_t = j_t.$ First, note that for any $j \ne  j_t,$ $$\Vert q_t-s_{j} \Vert_2 = \Vert w_{j_t}+s_{j_t}-s_j\Vert_2 \ge \Vert s_{j_t}-s_j\Vert_2-\Vert w_{j_t} \Vert_2 \ge \delta_{\min}-\delta_{\min}/4 \ge 3\delta_{\min}/4 > \Vert w_{j_t}\Vert_2 .$$ In particular, $j_t=i_t$. Note further that for any $j\ne j_t$, $$\Vert q_t - \hat{s}_j \Vert_2 = \Vert q_t-s_j+e_j\Vert_2 \ge \Vert q_t - s_j \Vert_2 - \Vert e_j \Vert_2 \ge 3\delta_{\min}/4-\delta_{\min}/4=\delta_{\min}/2.$$ Additionally, $$\Vert q_t -\hat{s}_{j_t} \Vert_2 =\Vert w_{j_t}+e_{j_t}\Vert_2 \le \delta_{\min}/4+\delta_{\min}/4=\delta_{\min}/2,$$ so for any $j \in [N]$, $$\Vert q_t - \hat{s}_j \Vert_2 \ge \Vert q_t - \hat{s}_{j_t}\Vert_2.$$ This yields $\hat{\imath}_t=j_t=i_t.$ Consequently, $$\mathbb{P}(\hat{\imath}_t \ne j_t | A(\delta)) \le \mathbb{P}(\Vert w_{j_t}\Vert _2\ge \delta_{\min}/4 | A(\delta))=\mathbb{P}(\Vert w_{j_t}\Vert_2 \ge \delta_{\min}/4)$$ since $q_t$ is independent from $q_s$ for $s < t.$  Hence $$\mathbb{P}( \hat{\imath}_t \ne j_t | A(\delta)) \le \sum_{j=1}^{N}\mathbb{P}(j_t=j)\mathbb{P}(\Vert w_j/\sigma\Vert_2^2 \ge \frac{\delta_{\min}^2}{16\sigma^2}).$$ By Lemma \ref{lem:hsu_subgaussian}, for any $j \in [N]$ and $\eta >0$, $$\mathbb{P}(\Vert w_j/\sigma\Vert_2^2 \ge 2d+3\eta)\le e^{-\eta}.$$ We then pick $\eta$ such that $2d+3\eta=\frac{\delta_{\min}^2}{16\sigma^2}$. This requires  $\delta_{\min}^2 \ge 32\sigma^2d$ and gives $\eta=\frac{\delta_{\min}^2}{48\sigma^2}-\frac{2}{3}d$.

Thus, $\mathbb{P}(\hat{\imath}_t \ne j_t | A(\delta)) \le \exp\left(-(\frac{\delta_{\min}^2}{48\sigma^2}-\frac{2}{3}d)\right)$. Using $\delta^2_{\min} \ge c\sigma^2d$, this finally yields $$\mathbb{P}(\hat{\imath}_t \ne j_t)\le e^{-\frac{c-32}{48}d}+\delta.$$

\end{proof}

We can now derive the full regret bound of the \extc\ algorithm.
\begin{proof}[Proof of Theorem \ref{thm:ETc_regret}.]
Combining Propositions \ref{prop:first_phase_regret_ETC} and \ref{prop:second_phase_regret_etc} and choosing $\delta=N/T,$ we obtain (under the generative setting) $$R'_{\extc}(T) \le \frac{(\beta-\alpha)(\log{N}+1)}{p_{\min}}\left(1+\frac{192\sigma^2(d+2\log{T})}{\delta_{\min}^2}\right)+(2\beta-\gamma-\alpha)N+(\beta-\gamma)Te^{-\frac{c-32}{48}d}.$$

Using  $\delta_{\min}^2 \ge c\sigma^2d$ and combining the previous inequality with  \eqref{eq:voronoi_to_gen}, we get under the Voronoi setting $$R_{\extc}(T) \le \frac{(\beta-\alpha)(\log{N}+1)}{p_{\min}}\left(1+\frac{192(d+2\log{T})}{cd}\right)+(2\beta-\gamma-\alpha)N+(\beta-\gamma)Te^{-\frac{c-32}{48}d}+(\beta-\gamma)T^2e^{-\frac{c-8}{12}d}.$$ 

In particular, if $\delta_{\min}^2 \ge 80\sigma^2d$, this implies
\begin{align*}R_{\extc}(T) \le \frac{(\beta-\alpha)(\log{N}+1)}{p_{\min}}\left(1+\frac{192(d+2\log{T})}{80d}\right)+(2\beta-\gamma-\alpha)N+(\beta-\gamma)Te^{-d}+(\beta-\gamma)T^2e^{-6d}\end{align*} 
When we additionally have $T \le e^d,$
\begin{align*}R_{\extc}(T) &\le \frac{41}{5}\frac{(\beta-\alpha)(\log{N}+1)}{p_{\min}}+(2\beta-\gamma-\alpha)N+(\beta-\gamma)(1+e^{-4d}) \\ &\le \frac{41}{5}\frac{(\beta-\alpha)(\log{N}+1)}{p_{\min}}+2(\beta-\gamma)(N+1).\end{align*} 

\end{proof}

\begin{remark}
In the proof above, under the condition $\delta_{\min}^2 \ge 80\sigma^2d,$ we assumed $T \le e^d$ to obtain a simple bound that scales with $(\log{N})/p_{\min}.$ We remark that weaker conditions on $T$ can still yield interesting bounds. For instance, $T \le e^d\sqrt{(\log{N})/p_{\min}} $ yields a bound that scales with  $\frac{\log{N}}{p_{\min}}(1+\frac{\log((\log{N})/p_{\min})}{d})$: \begin{align*}R_{\extc}(T) &\le \frac{(\beta-\alpha)(\log{N}+1)}{p_{\min}}\left(\frac{41}{5}+\frac{12}{5}\frac{\log((\log{N})/p_{\min})}{d}\right)+(2\beta-\gamma-\alpha)N+(\beta-\gamma)(1+e^{-4d})(\log{N})/p_{\min} \\ &\le \frac{(\beta-\alpha)(\log{N}+1)}{p_{\min}}\left(\frac{41}{5}+\frac{12}{5}\frac{\log((\log{N})/p_{\min})}{d}\right)+2(\beta-\gamma)(N+(\log{N})/p_{\min}).\end{align*} 
\end{remark}
\begin{remark}
In the proof of Proposition \ref{prop:first_phase_regret_ETC}, we used the simple upper bound $\mathbb{E}[\tau_1] \le  (\log{N}+1)/p_{\min}.$ $\mathbb{E}[\tau_1]$ can in fact be computed exactly as $I:=\int_{0}^{\infty}\left(1-\prod_{i=1}^{N}(1-e^{-p_it})\right)\mathrm{d}t$ (see example 5.17 in \citet{ross2010}). Using this formula would yield a sharper but less readable regret bound that scales with $I$ instead of $(\log{N})/p_{\min}$ in the $T \le e^d$ regime. We also note that an alternative bound is $I \le \sum_{i=1}^{N}1/p_i,$ which can be sharper than $I \le (\log{N}+1)/p_{\min}$ depending on the value of the weights $p_i.$
\end{remark}

%% file: b.appendix/5.proofs_GHC.tex
\subsection{Proofs of Section \ref{sec:threshold}}\label{sec:proofs_threshold}

In this section, we show that the Euclidean projection on the spherical convex hull of a set of points $\{p_i\}_{i=1}^{n}$ is the normalized projection on the positive hull of $\{p_i\}_{i=1}^{n}$ if this projection is non-zero, and is one of the $p_i$ otherwise.  
\begin{proposition}\label{prop:dist_sp}
Let $\Hcone_n$ be the positive hull generated by points $p_1,\dots,p_n \in \Sph^{d-1}$, i.e.,
$\Hcone_n = \left\{\sum_{i=1}^n \alpha_i p_i, \alpha \ge 0  \right\}$, and let $P_n=\Sph^{d-1} \cap \Hcone_n$ their spherical convex hull. For any $q \in \Sph^{d-1}$: \begin{enumerate}[label=(\roman*)]
    \item If $\operatorname{proj}_n:=\arg \min_{x \in \Hcone_n} \Vert q-x \Vert \ne 0,$ then $\displaystyle \arg \max_{x \in P_n} \; q^{\top}x=\frac{\operatorname{proj}_n}{\Vert \operatorname{proj}_n\Vert }$ and $$d(q,P_n)=\sqrt{2-2\Vert \operatorname{proj}_n \Vert}.$$
    \item If $\operatorname{proj}_n = 0,$ then $\displaystyle \max_{x \in P_n} q^{\top}x=\max_{1 \le i \le n} q^{\top}p_i$ and $$d(q,P_n)=\sqrt{2-2\max_{1 \le i \le n} q^{\top}p_i}.$$
\end{enumerate} 
\end{proposition}

\begin{proof}[Proof of Proposition \ref{prop:dist_sp}]

First note that for any $x \in \Sph^{d-1},$ $\Vert q-x\Vert^2_2=2-2q^{\top}x.$. Thus, computing  
$$\displaystyle \min_{\alpha \in \mathbb{R}^n} \; \Vert q- \sum_{i=1}^{n}\alpha_i p_i \Vert^2_2 \quad  \text{ subject to } \quad \alpha_i \ge 0,  \Vert \sum_{i=1}^{n} \alpha_ip_i \Vert_2=1$$ amounts to computing
$$\displaystyle \max_{\alpha \in \mathbb{R}^n} \; q^{\top}(\sum_{i=1}^{n} \alpha_ip_i) \quad  \text{ subject to } \quad \alpha_i \ge 0,  \Vert \sum_{i=1}^{n} \alpha_ip_i \Vert_2=1.$$
    Since $\Hcone_n$ is a closed convex set, by the closest point property, $\forall x \in \Hcone_n$, $(q-\operatorname{proj}_n)^{\top}(x-\operatorname{proj}_n) \le 0$. Applying this to $x=0 \in \Hcone_n$ and $x=2\operatorname{proj}_n \in \Hcone_n,$ we obtain $(q-\operatorname{proj}_n)^{\top}\operatorname{proj}_n = 0$, i.e., $q^{\top}\operatorname{proj}_n=\Vert \operatorname{proj}_n \Vert^2$. \\
    First assume that $\operatorname{proj}_n \ne  0$ and let $x \in P_n$. From $(q-\operatorname{proj}_n)^{\top}(x-\operatorname{proj}_n) \le 0$, $q^{\top}\operatorname{proj}_n=\Vert \operatorname{proj}_n \Vert^2$ and Cauchy-Schwartz, we get $$q^{\top}x \le \operatorname{proj}_n^{\top}x \le \Vert \operatorname{proj}_n \Vert.$$
    Conversely, this upper bound is attained for $x^*=\frac{\operatorname{proj}_n}{\Vert \operatorname{proj}_n\Vert } \in P_n$, as $$q^{\top}x^*=\frac{q^{\top}\operatorname{proj}_n}{\Vert \operatorname{proj}_n \Vert}=\Vert \operatorname{proj}_n \Vert.$$
    Thus, if $\operatorname{proj}_n \ne 0,$ $d(q,P_n)=\Vert q- \frac{\operatorname{proj}_n}{\Vert \operatorname{proj}_n \Vert}\Vert=\sqrt{2-2\frac{q^{\top}\operatorname{proj}_n}{\Vert \operatorname{proj}_n \Vert}}=\sqrt{2-2\Vert \operatorname{proj}_n \Vert}$.

    If $\operatorname{proj}_n = 0$, the closest point property gives $q^{\top}x \le 0$ for all $x \in \Hcone_n$. Let $x=\sum_{i=1}^{n} \alpha_i p_i \in P_n$. Then 
    $$q^{\top}x \le \sum_{i=1}^{n}\alpha_i(q^{\top}p_i) \le M\sum_{i=1}^{n}\alpha_i \le M$$ because $1=\Vert \sum_{i=1}^{n} \alpha_i p_i \Vert  \le \sum_{i=1}^{n} \alpha_i \Vert p_i \Vert=\sum_{i=1}^{n}\alpha_i,$ and $M \le 0.$ This is obviously attained by one of the $p_i \in P_n$. This means that $\max_{x \in P_n} q^{\top}x=\max_{1 \le i \le n} q^{\top}p_i$, so that $$d(q,P_n)=\sqrt{2-2\max_{1 \le i \le n} q^{\top}p_i} \ge \sqrt{2}.$$
\end{proof}

%% file: b.appendix/6.lemmas.tex
\section{CONCENTRATION TOOLS}


In this section, we provide concentration tools that are useful for our analysis. First, we present a lemma that controls the expectation of specific functions of a binomial random variable, which is useful for handling the randomness in the number of points landing per cell prior to applying our intermediary results on random polytopes, that are stated for a deterministic number of points (see e.g. Corollary \ref{cor:thinned_polytope_volume}).
In the statement below, $\bigO(\cdot)$ is asymptotic in $n$ and hides dependencies in $d$ and $p$.
\begin{lemma}\label{lem:binomial-concentration}
     Let  $X_n \sim \Bin(n,p)$ for $p \in [0,1]$. For
     \begin{enumerate}[label=(\roman*)] \item $\varphi_1:x \mapsto \frac{\log^{d-1}x}{x}$ with $\varphi_1(0)=0$ 
     \item $\varphi_2:x \mapsto \frac{\log^{d-2}x \log\log x}{x}$ with $\varphi_2(0)=\alpha \in \mathbb{R},\varphi_2(1)=0$
     \end{enumerate}
     we have  $$\mathbb{E}[\varphi_i(X_n)] \le \varphi_i(\mathbb{E}(X_n))+\bigO\left(\frac{1}{n}\right).$$
 \end{lemma}

 \begin{proof}[Proof of Lemma \ref{lem:binomial-concentration}]
 
Let $p \in (0,1)$ (otherwise $X_n$ is constant almost surely and the result follows). Since $\mathbb{E}[X_n]=np,$ we write $\varphi_1(X_n)=\varphi_1(X_n)\mathbbm{1}\{\vert X_n-np \vert < np/2 \}+\varphi_1(X_n)\mathbbm{1}\{\vert X_n-np \vert \ge np/2\}$.
From the multiplicative Chernoff bound, we know that $\mathbb{P}(\vert X_n-np \vert \ge np/2) \le 2e^{-np/12}$. For $x \ge 1,$ $\varphi_1$ is maximized at $x=e^{d-1}$, for which it is equal to $(\frac{d-1}{e})^{d-1},$ and we have $$ \mathbb{E}[\varphi_1(X_n)\mathbbm{1}\{\vert X_n-np \vert \ge np/2 \}] \le 2\left(\frac{d-1}{e}\right)^{d-1}e^{-np/12}.$$ 

On the other interval, namely $[np/2,3np/2],$ we perform a Taylor approximation. Around $x=np$, we write $\varphi_1(x)=\varphi_1(np)+\varphi_1'(\xi)(x-np)$ where $\xi$ is between $x$ and $np$.
$\varphi_1'$ is given by  $\varphi_1'(x)=\frac{(d-\log(x)-1)\log^{d-2}(x)}{x^2}$. For $n$ large enough $\vert \varphi_1' \vert$ is decreasing on $[np/2,3np/2]$ and $\vert \varphi_1'(np/2)\vert \le 4\vert d-\log(np/2)-1\vert\frac{\log^{d-2}(np/2)}{(np)^2} \le 4\frac{\log^{d-1}(np)}{(np)^2}.$ Then \begin{align*}\mathbb{E}[\varphi_1(X_n)\mathbbm{1}\{\vert X_n-np \vert < np/2 \}] &\le \varphi_1(np)+\vert \varphi_1'(np/2)\vert\mathbb{E}[\vert X_n-np \vert]  \\
&\le \varphi_1(np)+4\frac{\log^{d-1}(np)}{(np)^2}\sqrt{np(1-p)} \\
&\le \varphi_1(np)+4\sqrt{1-p}\frac{\log^{d-1}(np)}{(np)^{3/2}}.\end{align*} 

Thus for $n$ large enough, \begin{align*}\mathbb{E}[\varphi_1(X_n)] &\le \varphi_1(np) +4\sqrt{1-p}\frac{\log^{d-1}(np)}{(np)^{3/2}} + 2(\frac{d-1}{e})^{d-1}e^{-np/12} \\
&= \varphi_1(np)+\bigO(1/n).\\\end{align*}

For the function $\varphi_2$, adapting the previous arguments gives $\mathbb{E}[\varphi_2(X_n)\mathbbm{1}\{\vert X_n-np \vert \ge np/2 \}] \le Me^{-np/12}+\varphi_2(0)\mathbb{P}(X_n=0)=Me^{-np/12}+\alpha(1-p)^n$ where $M$ is the maximum of $\varphi_2$ for $x \ge 2,$ and for $n$ large enough, $\mathbb{E}[\varphi_2(X_n)\mathbbm{1}\{\vert X_n-np \vert < np/2 \}] \le \varphi_2(np)+4\frac{\log^{d-3}(np)}{(np)^2}[\log(\log(np))\log(np)+1]\sqrt{np(1-p)}.$ This gives $$\mathbb{E}[\varphi_2(X_n)] \le \varphi_2(np)+\bigO(1/n).$$

 \end{proof}

 The following subgaussian concentration inequality will also be useful for the analysis of \extc.
 \begin{lemma}[Theorem 2.1 in \cite{Hsu2012}]\label{lem:hsu_subgaussian}
    if $X$ is $\sigma$-subgaussian, then for any $\lambda>0$, $$\mathbb{P}(\Vert AX \Vert_2^2 \ge \sigma^2(\operatorname{tr}(\Sigma)+2\sqrt{\operatorname{tr}(\Sigma^2)\lambda}+2\Vert \Sigma \Vert_{op}\lambda)) \le e^{-\lambda}$$ for $\Sigma=A^{\top}A.$
\end{lemma}

%% file: b.appendix/7.additional_related_work.tex
\section{ADDITIONAL RELATED WORK}\label{sec:additional_related_work}

In this section, we provide further context on areas central to our theoretical analysis and experimental methodology, complementing Section \ref{sec:relatedwork}.

\paragraph{Convex Geometry} Our theoretical analysis relies on results regarding the approximation of convex bodies by random polytopes. We refer to the recent survey of \citet{prochno2022}, and to the books of \citet{Ziegler1995} and \citet{ratcliffe2019} for general background on Euclidean and spherical geometry.

\paragraph{Embedding-based Retrieval} Our problem formulation (Section \ref{sec:models}) assumes queries are represented by embeddings, and our real-world experiments (Section \ref{sec:experiments}, Appendix \ref{sec:additional_real_world_experiments}) utilize state-of-the-art text embedding models. Such embeddings are foundational to many modern Natural Language Processing (NLP) applications: for example, in contemporary applications of NLP to real-world question-answering environments, where knowledge may reside in documentation, Retrieval-Augmented Generation (RAG) \citep{DBLP:conf/nips/LewisPPPKGKLYR020} has emerged as a pivotal technique. A RAG system comprises two main components: a Retriever and a Large Language Model. Retrievers excel at representing similar words and sentences closely in the embedding space, thereby understanding language patterns effectively. 
Notable retriever models, such as Dense Passage Retrieval (DPR) \citep{karpukhin-etal-2020-dense}, EmbEddings from bidirEctional Encoder rEpresentations (E5) \citep{wang2022text}, and General Text Embedding (GTE) \citep{li2023general}, leverage pretrained architectures such as BERT \citep{devlin-etal-2019-bert} to initialize encoders. These encoders 
are fine-tuned to ensure that the cosine similarity  between the query and passage accurately captures their true relationship.
 
Recently, the NLP community has increasingly favored decoder architectures for creating embeddings \citep{Springer2024RepetitionIL,behnamghader2024llm2vec}, as these approaches have demonstrated superior performance over traditional encoder-based methods. Among these, Mistral-E5 \citep{wang-etal-2024-improving-text} stands out as the state-of-the-art embedding model for text retrieval. Retriever models are commonly evaluated using benchmarks such as BEIR \citep{thakur2021beir} and MTEB \citep{muennighoff2022mteb}.

%% file: b.appendix/8.additional_discussions.tex
\section{ADDITIONAL RESULTS AND DISCUSSIONS}

In this section, we provide additional results and discussions: 
\begin{enumerate}[label=(\roman*)]
\item We discuss a slight refinement of \vhc's guessing rule; 
\item We show that our analysis of \vhc\ is sharp; 
\item We discuss Theorems \ref{thm:d-dimensional-regret} and \ref{thm:1-dimensional-regret-improved};
\item We discuss the center separation condition used in the \extc\ analysis; 
\item We explain why the regret analysis of \vht($\tau$) for $\tau>0$ is challenging. 
\end{enumerate}

\subsection{\vhc\ Guessing Rule Refinement}
If all hulls $\hat{\mathcal{C}}_{i,t}$ for $i \in [N]$ are non-empty, there are cases where the cell (hence the label) of a query that lands outside of the convex hulls can be deduced. Indeed, if $\operatorname{hull}_{\mathcal{E}}(\{q_t\} \cup \hat{\mathcal{C}}_{i,t}) \cap \hat{\mathcal{C}}_{j,t} \neq \emptyset$ for some $i \ne j,$ then $q_t$ cannot have label $i$. The decision rule of \vhc\ may thus be slightly refined by returning label $i$ when for all $k \in [N] \setminus \{i\},$ $$\operatorname{hull}_{\mathcal{E}}(\{q_t\} \cup \hat{\mathcal{C}}_{k,t}) \cap \hat{\mathcal{C}}_{j,t} \neq \emptyset$$ for some $j \ne k$. For example, in Fig. \ref{fig:mixg}, $q$ is necessarily in the top-right cell, and we know its correct label. We do not leverage this minor refinement in the analysis nor the implementation of the algorithm.

\subsection{Asymptotic Lower bound on the Regret of \vhc}\label{app:dist_independent_lower_bound}

Here, we make a slightly stronger assumption on the density of $\mu$ than Assumption \ref{as:density}(i):

\begin{assumption}\label{as:cont_density}
The distribution $\mu$ is absolutely continuous with respect to $V$ with density
$f_{\mu}=\frac{d\mu}{dV}$. Furthermore, $f_{\mu}$ is continuous on $\mathcal{E}$ and $\inf_{x\in\mathcal{E}} f_{\mu}(x)> 0\,$.
\end{assumption}
\begin{theorem}\label{thm:dist_independent_regret_bound}
Under Assumptions \ref{as:cells} and \ref{as:cont_density}, the following result holds either if $\mathcal{E}=\mathcal{I}^d$ and $d \ge 2$ with $\eta=d$, or if $\mathcal{E}=\Sph^{d-1}$, $d \ge 3$ and each cell $\mathcal{C}_i$ is contained in an open half-sphere with $\eta=d-1$:
$$\lim \inf_{T \to \infty} \frac{R_{\vhc}(T)}{\log^\eta{T}} \ge (\beta-\alpha)\frac{\sum_{i=1}^N F(\mathcal{C}_i)}{4{\eta! \eta^{\eta}}}.$$
\end{theorem}

The lower bounds established in Theorem \ref{thm:dist_independent_regret_bound}, are independent of the distribution $\mu$ and match the upper bounds of Theorem \ref{thm:d-dimensional-regret} in the uniform case ($C/c = 1$), up to a multiplicative factor linear in the dimension $d$. This suggests that \vhc\ does not exploit favorable scenarios where $\mu$ is highly concentrated around the seed queries of the Voronoi tessellation. It is important to stress that this lower bound is specific to the \vhc\ algorithm and does not extend to all possible algorithms. The proof, presented below, relies on the connection between random polytopes and the so-called \emph{floating body} of the polytope they approximate, a notion developed in \citet{schutt1990}.

\begin{proof}[Proof of Theorem \ref{thm:dist_independent_regret_bound}]

Recall that the expected number of expert calls is driven by the probability of a query $q_t$ landing outside the current estimated hulls: 
$$\mathbb{E}[C_T] = \sum_{t=1}^T \mathbb{E}[\mathbb{P}(q_t \notin \cup_{i} \hat{\mathcal{C}}_{i,t} | \mathcal{F}_{t-1})] .$$ By equality \eqref{eq:expert-proba}, this probability is the expected measure of the uncovered region, $\sum_{i=1}^{N}\mathbb{E}[\mu(\mathcal{C}_i \setminus \hat{\mathcal{C}}_{i,t})] = \sum_{i=1}^N \mu(\mathcal{C}_i) \mathbb{E}[\mu_{\mathcal{C}_i}(\mathcal{C}_i \setminus \hat{\mathcal{C}}_{i,t})]$, where $\mu_{\mathcal{C}_i}$ is the conditional distribution of $q_t$ given that $q_t \in \mathcal{C}_i.$
This means that \begin{equation}\label{eq:regret_reformulation}R_{\vhc}(T) = (\beta-\alpha)\sum_{t=1}^{T}\sum_{i=1}^N \mu(\mathcal{C}_i) \mathbb{E}[\mu_{\mathcal{C}_i}(\mathcal{C}_i \setminus \hat{\mathcal{C}}_{i,t})].\end{equation}


\paragraph{Case $\mathcal{E}=\mathcal{I}^d$.}
Instead of applying Corollary \ref{cor:thinned_polytope_volume}, we may alternatively use the following lemma.

\begin{lemma}\label{lem:missing_measure_no_density}
Let $P$ a convex polytope in $\mathbb{R}^d$ with $d \ge 2$ and $n \ge 1$ points $p_1,\dots,p_n$ sampled  independently from a distribution $\nu$ supported on $P$ with a density w.r.t. $\lambda$ that is continuous and never zero on $P$. Denote by $P_n$ the convex hull of $p_1,\dots,p_n$. Then 
\begin{align*}\mathbb{E}[\nu(P \setminus P_n)] \ge \frac{F(P)}{4d! d^{d-1}}\frac{ \log^{d-1}{n}}{n}  +o\left(\frac{ \log^{d-1}{n}}{n}\right).\end{align*}
\end{lemma}

Lemma \ref{lem:missing_measure_no_density} leverages the relationship between the uncovered region $P \setminus P_n$ and the wet part of $\nu$, which we now define. For a distribution $\nu$, the wet part of $\nu$ is defined in \citet{barany2020} as
\begin{align*}W_{\nu}^t&=\{x \in \mathbb{R}^d, \text{there is a halfspace $h$ with $x \in h$ and $\nu(h) \le t$}\},\end{align*} and the $\nu$-measure of the wet part is $$w^{\nu}(t)=\nu(W^{\nu}_t).$$ To understand this notion intuitively, consider the case where $\nu$ is uniform on the unit ball, and assume that the ball is filled with a volume $t$ of water. The wet part $W^t_{\nu}$ represents the ``trace'' left by the water as the ball rolls on a flat surface, and $w^{\nu}(t)$ is the volume of that wet part. The $\nu$-measure of the wet part is tightly linked to the expected missing volume $\mathbb{E}[1-\nu(P_n)]$, as shown in e.g. \citet{schutt1991} 
 and \citet{barany1993} for the case where $\nu$ is uniform on some convex polytope of $\mathbb{R}^d$, and in \citet{barany2020} for arbitrary measures in $\mathbb{R}^d.$ Such results are sometimes stated in terms of the \emph{floating body} of $\nu$, which is the relative complement of its wet part in its support, see \citet{schutt1990}.

\begin{proof}[Proof of Lemma \ref{lem:missing_measure_no_density}]
    By the lower bound of Theorem 1.2 in \citet{barany2020}:
\begin{equation}\label{eq:barany_bound} \mathbb{E}[\nu(P \setminus P_n)]  \ge \frac{1}{4} w^{\nu}\left(\frac{1}{n}\right). \end{equation}

This lower bound is straightfoward to derive: one can note that $\PP(x \notin P_n) \ge (1-t)^n$ for any $x \in W^t_{\nu}$, thus $\EE[\nu(P \setminus P_n)]=\int_{0}^{\infty} \PP(x \notin P_n)d\mu(x) \ge (1-t)^nw^{\nu}(t)$. Setting $t=1/n$ yields the result (see Section 3.1 in \citet{barany2020} for a more detailed proof).

Then, by Corollary 1.2 in \citet{besau2018}, under Assumption \ref{as:density}(ii),  the asymptotic behavior of $w^{\nu}$ is independent of $\nu$:  \[ w^{\nu}(\delta) \underset{\delta \to 0}{\sim} \frac{F(P)}{d! d^{d-1}} \delta\log^{d-1}(1/\delta) \]
where we recall that $F(P)$ is the number of flags of $P$. 
Consequently, the right hand side in inequality \eqref{eq:barany_bound} is asymptotically equivalent to $\frac{F(P)}{d! d^{d-1}}\frac{1}{4n}  \log^{d-1}{n}$.
\end{proof}

As in the proof of Theorem \ref{thm:d-dimensional-regret}, we may write $\mathbb{E}[\mu_{\mathcal{C}_i}(\mathcal{C}_i \setminus \hat{\mathcal{C}}_{i,t})]=\mathbb{E}[\mathbb{E}[\mu_{\mathcal{C}_i}(\mathcal{C}_i \setminus Q_{n_i(t)}) | n_i(t)]]$ where $n_i(t)$ is the number of points that landed in $\mathcal{C}_i$ up to time $t$ and $Q_n$ is the convex hull of $n$ points sampled independently from $\mu_{\mathcal{C}_i}$. Applying Lemma \ref{lem:missing_measure_no_density} to each $\mathcal{C}_i$ gives \begin{align*}\EE[\mu_{\mathcal{C}_i}(\mathcal{C}_i \setminus \hat{\mathcal{C}}_{i,t}) | n_i(t)=n] \ge\frac{F(\mathcal{C}_i)}{4d! d^{d-1}}\frac{ \log^{d-1}{n}}{n}  +o\left(\frac{ \log^{d-1}{n}}{n}\right) .\end{align*}
  
As $n_i(t)$ follows a $\Bin(t,\mu(\mathcal{C}_i))$ distribution, Lemma \ref{lem:binomial-concentration} gives that as $t \to \infty$, 
$\frac{F(\mathcal{C}_i)}{4d! d^{d-1}}\EE\left[\frac{\log^{d-1}(n_i(t))}{n_i(t)}\right]   \underset{t \to \infty}{\sim} \frac{F(\mathcal{C}_i)}{4d! d^{d-1}}\frac{ \log^{d-1}{t}}{t\mu(\mathcal{C}_i)}$, so that asymptotically in $t$, \begin{align*}\mathbb{E}[\mu_{\mathcal{C}_i}(\mathcal{C}_i \setminus \hat{\mathcal{C}}_{i,t})] \ge \frac{F(\mathcal{C}_i)}{4d! d^{d-1}}\frac{ \log^{d-1}{t}}{t\mu(\mathcal{C}_i)}  +o\left(\frac{ \log^{d-1}{t}}{t}\right)   .\end{align*} 
As $\sum_{t=1}^{T}\frac{ \log^{d-1}{t}}{t}=\frac{\log^d{T}}{d}+\bigO(1)$  for any $d \ge 1$ (with the convention $0^0=1$), combining the previous identity with equality \eqref{eq:regret_reformulation}  yields \begin{align*}R_{\vhc}(T) \ge (\beta-\alpha)\frac{\sum_{i=1}^{N}F(\mathcal{C}_i)}{4d! d^{d-1}}\frac{\log^d{T}}{d}  +o\left(\log^d{T}\right)\end{align*}  
which concludes the proof in the case  $\mathcal{E}=\mathcal{I}^d$.

\paragraph{Case $\mathcal{E}=\Sph^{d-1}$.}
The exact same gnomonic projection argument as in the proof of Corollary \ref{cor:spherical_polytope_volume} can be used to obtain a spherical analogue of Lemma \ref{lem:missing_measure_no_density}, simply noting that the density of the pushforward measure on $\mathbb{R}^d$ will remain continuous and never zero on the projected convex polytope $g_e(P)$ if the original measure on $\Sph^{d-1}$ is:
\begin{lemma}\label{lem:missing_measure_no_density_spherical}
Let $P$ a spherically convex polytope in $\Sph^{d-1}$ contained in an open halfsphere with $d \ge 2$ and $n \ge 1$ points $p_1,\dots,p_n$ sampled  independently from a distribution $\nu$ supported on $P$ with a  density w.r.t. $\omega$ that is continuous and never zero on $P$. Denote by $P_n$ the convex hull of $p_1,\dots,p_n$. Then 
\begin{align*}\mathbb{E}[\nu(P \setminus P_n)] \ge \frac{F(P)}{4(d-1)! (d-1)^{d-2}}\frac{ \log^{d-2}{n}}{n}  +o\left(\frac{ \log^{d-2}{n}}{n}\right).\end{align*}
\end{lemma}
The regret analysis is then exactly the same as in the case $\mathcal{E}=\mathcal{I}^d$.
\end{proof}

\input{b.appendix/wet_part_volume}

\input{b.appendix/dim_1_vs_dim_d}
\subsection{Center Separation Assumption when $\mathcal{E}=\mathcal{I}^d$ and $\mathcal{E}=\Sph^{d-1}$}\label{app:separation_assumption}
The \extc\ regret bound of Theorem \ref{thm:ETc_regret_main} relies on the  separability assumption $\delta_{\min}^2 \ge c\sigma^2 d$ for some $c\ge 32$.

When $\mathcal{E}=\mathbb{R}^d$, this assumption is reasonable: since the expected squared Euclidean distance between two uniformly sampled points scales linearly with $d$ \citep{anderssen1976concerning}, it is plausible for $\delta_{\min}^2$ to also scale with $d$. If $\mathcal{E}=\mathcal{I}^d$, it is reasonable if $\sigma$ is small. In this context, the $d$ in the numerator of the threshold formula $$\hat{m}(t)=\frac{108\sigma^2(d+2\log{T})}{\hat{\delta}^2_{\min}(t)}$$is effectively balanced by the implicit scaling of $d$ in the denominator, making the threshold rule behave in a ``dimensionless'' manner.

However, this assumption may become unrealistic when $\mathcal{E}=\Sph^{d-1}$, as the average distance between uniformly sampled seeds is bounded by $2$, irrespective of the dimension. This average distance in fact approaches $\sqrt{2}$ as $d \to \infty$, because the probability that two seeds are orthogonal approaches one as $d \to \infty$. Therefore, we cannot expect $\delta_{\min}^2$ to scale linearly with $d$. In this setting, the numerator of $\hat{m}(t)$ scales with $d$ while the denominator does not.

Although Theorems \ref{thm:ETc_regret} and \ref{thm:ETc_regret_main} still hold as stated for subgaussian mixtures supported on the sphere, a tailored analysis would employ concentration inequalities native to that geometry. They would require a weaker separability assumption and result in a different threshold rule, likely both without a linear dependence on $d$. Such an analysis is beyond the scope of this work. Our empirical evaluation explores this by treating the constant in the definition of the threshold $\hat{m}(t)$ as a tunable hyperparameter, so that we can find a practical balance suitable for each geometric setting.

\input{b.appendix/tau_ge_0_discussion}

%% file: b.appendix/wet_part_volume.tex
\subsection{On the Error Term in Theorem \ref{thm:d-dimensional-regret}}\label{sec:wet_part}

The error term in Theorem \ref{thm:d-dimensional-regret}, e.g. $\bigO(\log^{d-1}(T)\log\log{T})$ when $\mathcal{E}=\mathcal{I}^d$, depends on the following model parameters.
\begin{enumerate}[label=(\roman*)]
\item The geometry of each cell $\mathcal{C}_i$. This dependency comes from our application of Proposition \ref{prop:polytope_volume} (Theorem 2 from \cite{barany1993}). 
\item The bounds on the density of $\mu$ (\ref{as:density}), and the volume of each $\mathcal{C}_i$. This dependency mainly comes from both the rejection sampling trick (Corollary \ref{cor:thinned_polytope_volume}) and the application of Lemma \ref{lem:binomial-concentration}—it was applied for $p=c\lambda(\mathcal{C}_i)$ in the proof of Corollary \ref{cor:thinned_polytope_volume}, and for $p=\mu(\mathcal{C}_i) \in [c\lambda(\mathcal{C}_i),C\lambda(\mathcal{C}_i)]$ in the proof of Theorem \ref{thm:d-dimensional-regret}.
\end{enumerate}

Below, we support the claim that it is very difficult to identify the precise dependency in the cell geometry (and in particular in the dimension) induced by Theorem $2$ from \cite{barany1993}, which was not tracked in their original proof. One way to attempt to identify this dependency is to exploit the relationship between random polytopes and the floating body of the polytope they approximate, as already leveraged in the proof of Theorem \ref{thm:dist_independent_regret_bound}. For example, \cite{barany2020} show in their Theorem 1.2 that if $p_1,\dots,p_n$ are sampled according to some arbitrary $\nu$ on $\mathbb{R}^d$ with convex hull $P_n$, $$\mathbb{E}[1-\nu(P_n)] \le w^{\nu}\left((d+2)\frac{\log{n}}{n}\right)+2^{d+3}e\frac{\lceil \log{n} \rceil^d}{n^2}.$$
Here, the dependency in $(\log{n})/n$ in the wet part measure is necessary because $\nu$ is arbitrary, but for more specific distributions (e.g. log-concave), this dependency may potentially be improved to $1/n$, as discussed in Section 2.4 of \cite{barany2020}. Since $\mathbb{E}[1-\nu(P_n)]$ is precisely the measure of the region uncovered by $P_n$, we may attempt to upper bound $w^{\nu}(t)$ for fixed $t$ to obtain a fully explicit upper bound on our quantity of interest. This is not too difficult when $\nu$ is uniformly distributed on the hypercube $\mathcal{I}^d$, as we show below.

\begin{proposition}
Denote by $w(t)$ the volume of the wet part of parameter $t$ for $\nu$ uniform on  $\mathcal{I}^d$. For any $n \ge 2$, $w(1/n) \le \frac{2}{n}\sum_{k=0}^{d-1}\frac{\log(n/2)^k}{k!}.$
\end{proposition}

\begin{proof}
Let $v(x)=\min\{\lambda(\mathcal{I}^d \cap H), x \in H \text{ and $H$ is a halfspace}\}$, and $u(x)=\lambda(\mathcal{I}^d \setminus (2x-\mathcal{I}^d))$. This is the volume of a \emph{Macbeath} region, see e.g. \citet{barany2008random}, where they recall in their Lemma 4.3 that $u(x) \le 2v(x)$. Note also that the wet part for the uniform distribution on $\mathcal{I}^d$ is $\{x \in K, v(x) \le t\lambda(\mathcal{I}^d)\}.$ Then $$u(x)=\lambda(\mathcal{I}^d \cap (2x-\mathcal{I}^d))=\lambda([0,1] 
^d \cap ([2x_1-1,2x_1]\times \dots \times [2x_d-1,2x_d]).$$ By symmetry we may decompose the hypercube into $2^d$ subcubes based on which side of $1/2$ each coordinate lies and only consider the case $x \in S_0:=[0,1/2]^d$. For $x \in S_0$, we have $u(x)=\lambda([0,2x_1]\times \dots \times [0,2x_d])=2^d\prod_{i}x_i$.
Thus $$w(t)=\int_{\mathcal{I}^d} \mathbbm{1}_{v(x) \le t}\mathrm{d}x \le \int_{\mathcal{I}^d} \mathbbm{1}_{u(x) \le 2t}\mathrm{d}x = 2^d\int_{S_0} \mathbbm{1}_{u(x) \le 2t}\mathrm{d}x.$$ Setting $y=2x,$  this is equal to \begin{align*} 2^d\int_{S_0} \mathbbm{1}_{\prod x_i \le 2t/2^d}\mathrm{d}x &= 2^d\int_{\mathcal{I}^d} \mathbbm{1}_{\prod (y_i/2) \le 2t/2^d}(1/2^d)\mathrm{d}y \\
&= \int_{\mathcal{I}^d} \mathbbm{1}_{\prod y_i \le 2t}\mathrm{d}y \\
&=  \lambda(\{y \in [0,1]^d, \prod y_i \le 2t\}).
\end{align*} Now for any $0 < a \le 1$, 
$\lambda(\{y \in [0,1]^d, \prod y_i \le a\})=a\sum_{k=0}^{d-1}\log(1/a)^k/k!$ (see \cite{barany1988}, p.288).
Consequently, $w(t) \le 2t\sum_{k=0}^{d-1}\log(1/(2t))^k/k!$ whenever $0 < t \le 1/2$. Applying this to $t=1/n$ yields the result.
\end{proof}


When $\nu$ is uniform on some arbitrary polytope $P$, explicitly bounding the volume of the wet part of $P$ becomes significantly more difficult. \citet{schutt1991} proved in their Lemma 1.8 that for some $t_d$ and $c_d$ that only depends on $d$, if $0< t < t_d$,  $$w^{\nu}(t) \le \frac{F(P)}{d!d^{d-1}}t\log(\lambda(P)/t)^{d-1}+c_dF(P)t\log(\lambda(P)/t)^{d-2},$$ but no explicit expressions for $t_d$ and $c_d$ are provided. Their proof strategy is to first obtain an upper bound when $P$ is a simplex,  then form a partition of $P$ by $F(P)$ simplices, and apply this bound to each simplex in the partition. By carefully tracking the constant appearing in the proof of their Lemma 1.3, one can check that even in the simplex case, their arguments only yield a likely very loose  value of $c_d$ that scales roughly as $2^{d^2}.$ As such, more refined arguments are required to identify the precise dependency in $d$ of $c_d$, and to the best of our knowledge, this has not been elucidated in subsequent works.

%% file: b.appendix/dim_1_vs_dim_d.tex
\subsection{On Generalizing Theorem \ref{thm:1-dimensional-regret-improved} to $d \ge 2$}\label{app:dim_1_vs_dim_d}

In dimension one, we derived the sharp density-free regret bound of Theorem \ref{thm:1-dimensional-regret-improved} via a change of variable involving cumulative distribution functions (c.d.f.). Specifically, we used that:
\begin{enumerate}[label=(\roman*)]
\item If $q$ has a continuous c.d.f. $F$, $F(q)$ is uniformly distributed in $[0,1]$
\item This $F$ maps convex hulls to convex hulls: $F(\conv\{p_1,\dots,p_n\})=\conv\{F(p_1),\dots,F(p_n)\}$.

\end{enumerate}
A natural idea to generalize Theorem \ref{thm:1-dimensional-regret-improved} to higher dimensions  $d \ge 2$ is to identify a transformation $F$ in $\mathbb{R}^d$ with analogous properties. However, preserving convex hulls is a much more restrictive condition in dimensions $d \ge 2$ that for instance forces $F$ to be affine if it is one-to-one: this is the so-called fundamental theorem of affine geometry, see e.g. Theorem 2.16 in \citet{Artstein-Avidan2012}.
One may hope that a weaker but still sufficient condition for the analysis would hold: namely, by inspecting the proof of Theorem \ref{thm:1-dimensional-regret-improved}, it can be noted that the condition \begin{equation}\label{eq:weaker_conv_cond}
\lambda(F(\conv\{p_1,\dots,p_n\})) \ge\lambda(\conv\{F(p_1),\dots,F(p_n)\})
\end{equation}
would suffice.  
A natural candidate for $F$ is the Rosenblatt transformation, which maps continuous random vectors to vectors uniformly distributed on the hypercube by inductively projecting their coordinates with their conditional c.d.f. \citep{rosenblatt1952}.
We show below that unfortunately, \eqref{eq:weaker_conv_cond} does not hold for the Rosenblatt transformation.

\textbf{Counter-example for \eqref{eq:weaker_conv_cond}}. Consider a random variable with support $[0,1]^2$ and density $f(x_1,x_2)=2x_1$, $p_1=(0,0),$ $p_2=(1,0)$, $p_3=(0,1).$ Then $\conv\{p_1,p_2,p_3\}=\{(x,y) \in [0,1]^2 | x+y \le 1\}$ is a triangle.
Furthermore, the marginal and conditional c.d.f. are given by $F_{1}(x_1)=x_1^2$ and $F_{2 | 1}(x_2 | x_1)=x_2.$ The Rosenblatt transformation is then given by $F(x_1,x_2)=(x_1^2,x_2)$. In particular, $F(p_i)=p_i$ so $\conv\{p_1,p_2,p_3\}=\conv\{F(p_1),F(p_2),F(p_3)\}$, yet $F(\conv\{p_1,p_2,p_3\})=\{(x^2,y) | (x,y) \in [0,1]^2, x+y \le 1\}=\{(x,y) | (x,y) \in [0,1]^2, \sqrt{x}+y \le 1\} \subset \conv\{p_1,p_2,p_3\}.$ In fact, $\lambda(\conv\{F(p_1),F(p_2),F(p_3)\})=1/2,$ and $\lambda(F(\conv\{p_1,p_2,p_3\}))=\int_{0}^{1} (1-\sqrt{u})\mathrm{d}u=1/3.$

This suggests that this transformation argument is not fruitful in higher dimensions, and that a more refined analysis is required.

%% file: b.appendix/tau_ge_0_discussion.tex
\subsection{Challenges of the \vht\ Regret Analysis}\label{sec:analysis_threshold}

Analyzing the regret of \vht($\tau$) (Algorithm \ref{alg:vht}) for $\tau > 0$ introduces significant complexities not present in the  analysis of the \vhc\ algorithm (which corresponds to \vht\ with $\tau=0$). 

\paragraph{Non i.i.d. Hull Construction} It can no longer be said that  $\hat{\mathcal{C}}_{i,t} = \text{hull}_{\mathcal{E}}(\mathcal{Q}_{i,t})$ is the convex hull of the i.i.d. samples drawn from $\mu$ that landed in $\mathcal{C}_i$, which is crucial for applying results from random polytope theory, as they are typically stated for independently sampled points.  
Indeed, when $\tau>0$, if $q_t$ results in a guess (whether correct or incorrect), since no reward feedback is observed, $q_t$ will not be used to form the hulls. $\hat{\mathcal{C}}_{i,t}$ is \emph{only} the convex hull of the queries that either triggered an expert call or landed inside the current hulls.
The queries are selected through a complex history-dependent process that depends on the geometric configuration of all estimated hulls $\hat{\mathcal{C}}_{j,t}$.

\paragraph{Bounding the Number of Incorrect Guesses} When $\tau>0$, the \vht\ may make incorrect guesses, and the regret is no longer directly proportional to the number of expert calls: rather, $$R_{\vht}(T)=(\beta-\gamma)\mathbb{E}[M_T]+(\beta-\alpha)\mathbb{E}[C_T]$$  where $M_T$ is the total number of incorrect guesses (mistakes). A mistake occurs if $q_t \in C_k$ for some $k \in [N]$ while the algorithm guesses  $i \ne k$. It can be checked that this implies $d(q_t, C_i) \le \varepsilon_{i,t} + \tau \varepsilon_{k,t}$, where $\varepsilon_{j,t} = d_H(\mathcal{C}_j, \hat{\mathcal{C}}_{j,t})$ is the Hausdorff distance between the cell $\mathcal{C}_j$ and its estimate. Bounding these Hausdorff distances is already difficult due to the non-i.i.d. hull formation. Furthermore, in the simpler i.i.d. uniform case, bounds on the Hausdorff distance between a random polytope and the polytope it approximates scale as $n^{-1/d}$ where $n$ is the number of points \citep{barany1989} . This analysis approach would yield bounds on the number of mistakes scaling with $T^{(d-1)/d}$, which, asymptotically in $T$, is significantly worse than the $\log^{d}{T}$ bound of Theorem \ref{thm:d-dimensional-regret}.

\paragraph{Bounding the Number of Expert Calls} The region triggering an expert call is also more intricate: the convex hulls $\hat{\mathcal{C}}_{i,t}$ depend on $\tau$. For a fixed hull, increasing $\tau$ reduces the probability of an expert call, but the hulls vary with $\tau$: as $\tau$ increases fewer points are used to form them. For a fixed $\tau$, reducing the hulls makes expert calls more likely. These two effects counteract each other and it is not directly obvious how to handle them.

These dynamics prevent a straightforward extension of the \vhc\ analysis. Deriving tight regret bounds for \vht\ with $\tau > 0$ would likely require more involved techniques to handle these history-dependent processes.

\paragraph{Decision Regions} To better visualize the impact of $\tau$ and the distribution of the queries in the decision rule of \vht\, we complement Fig. \ref{fig:guessing_regions} from Section \ref{sec:threshold} with Fig.  \ref{fig:guessing_regions_full}. The decision regions in round $t=250$ for a run of \vht($\tau$)  for $\tau \in \{0.1,0.2,\dots,0.9\}$ are shown, for a uniform distribution and a mixture of truncated Gaussians with small covariance.

\begin{figure}[ht]
    \centering
    \begin{subfigure}{0.83\textwidth}
        \includegraphics[width=\linewidth]{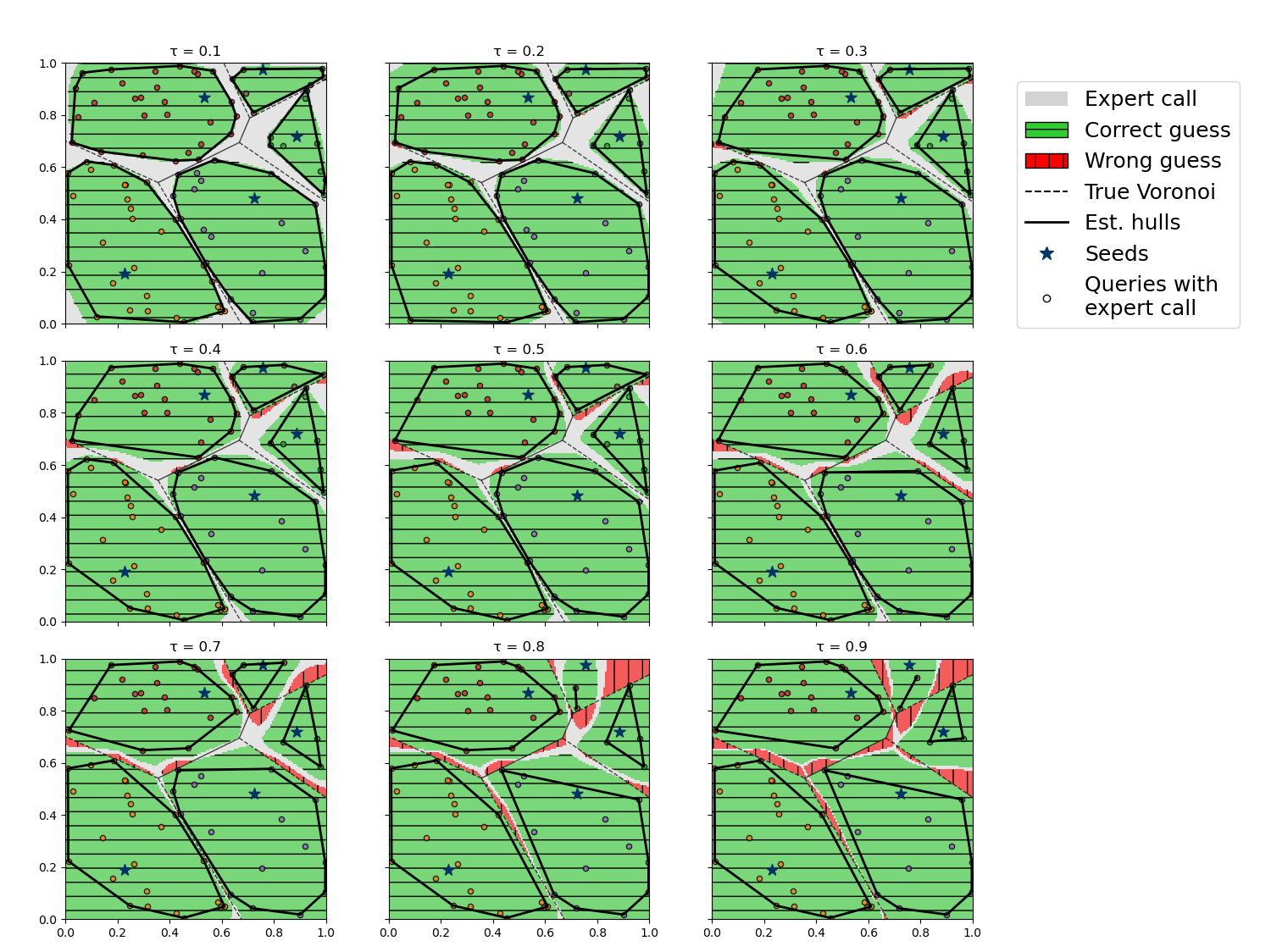}
        \caption{Uniform}
        \label{fig:misunif}
    
    \end{subfigure}\hfill
    \begin{subfigure}{0.83\textwidth}
        \includegraphics[width=\linewidth]{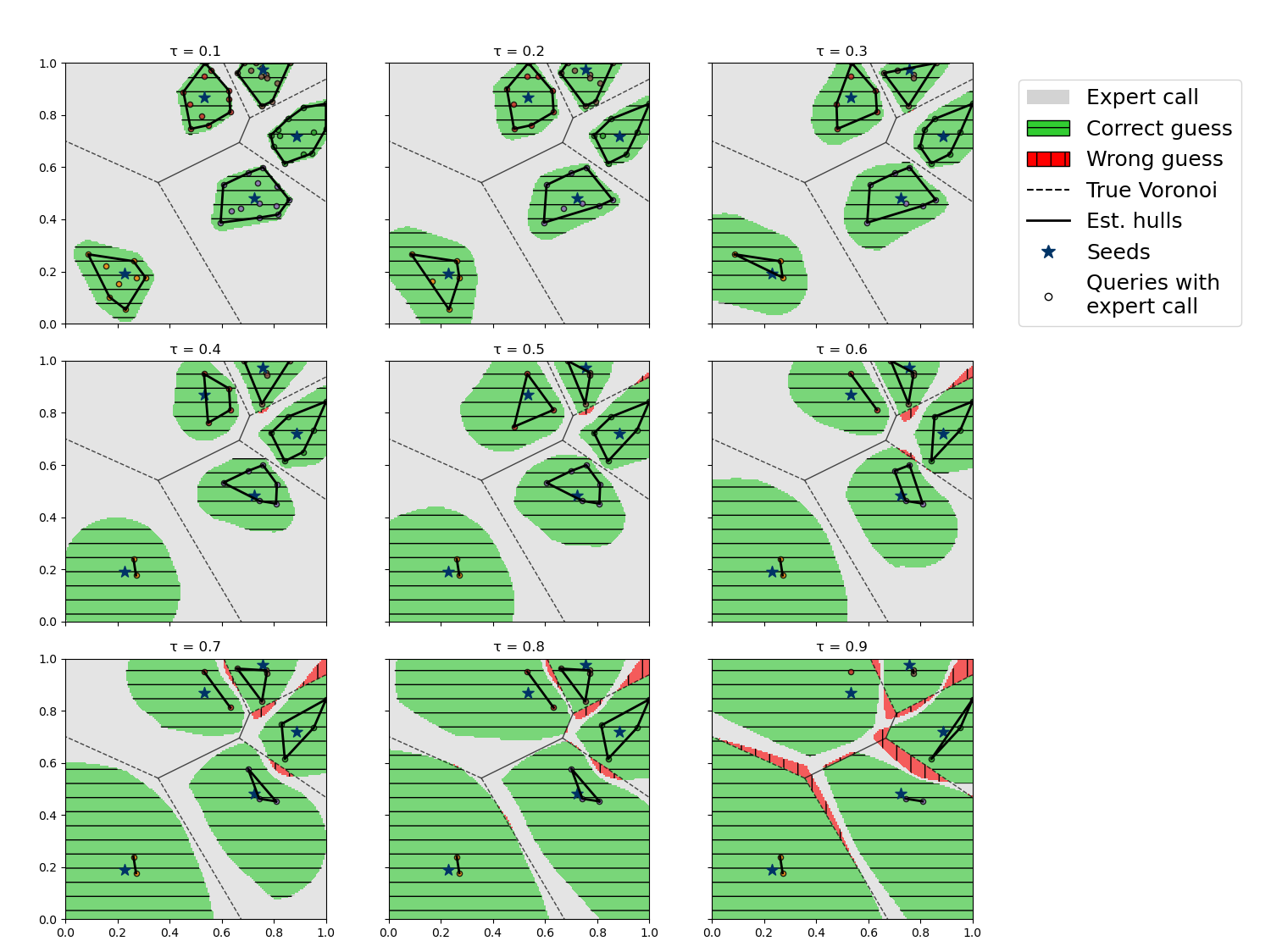}
        \caption{Mixture of truncated Gaussians, covariance matrix $0.0025I.$}
        \label{fig:mismix}
    
    \end{subfigure}\hfill
    \caption{Decision regions of \vht($\tau$) for $t=250$ and $\mathcal{E}=[0,1]^2.$}
    \label{fig:guessing_regions_full}
    \end{figure}

\clearpage

%% file: b.appendix/9.synthetic_experiments.tex
\section{EXPERIMENTS DETAILS}\label{app:experiments}

In this section, we provide the results of numerous additional synthetic and real-world experiments, along with a more detailed description of the experiments of Section \ref{sec:experiments}. Tthe reward values are always instantiated as $\alpha=-1$, $\beta=+1$, $\gamma=-10$.

\subsection{Synthetic Experiments}\label{sec:additional_synthetic_experiments}

To evaluate our algorithms in low dimensions, we perform synthetic experiments on two query domains for $N=5$ labels:
\begin{enumerate}[label=(\roman*)]
  \item $\mathcal{I}^d$, $d\in\{1,4,10,50\}$, where the seeds $s_1,\dots,s_N$ are drawn uniformly on $\mathcal{I}^d$ and $q_t$ is either drawn from the uniform distribution on $\mathcal{I}^d$ or from a homogeneous mixture of truncated Gaussians with covariance matrix $0.01I$ In this setting, \vht\ is applied using the distance to Euclidean convex hulls.
  \item $\Sph^{d-1}$, $d\in\{2,4,10,50\}$ where the seeds $s_1,\dots,s_N$ are drawn uniformly on $\Sph^{d-1}$. The query $q_t$ is sampled either uniformly on $\Sph^{d-1}$ or from a mixture, specifically
$i \in [N]$ is drawn uniformly at random and  $q_t=y_t/\Vert y_t \Vert$ where $y_t \sim \mathcal{N}(s_i,0.01I)$. In this setting, \vht\ is applied using the distance to spherical convex hulls.
\end{enumerate}

For each configuration, five independent datasets of size $T=5000$ are generated. All evaluated algorithms are run on the same datasets. The expert labeling policy is assumed to be given by the Voronoi tessellation with seeds $s_i.$ We report detailed metrics, namely:
\begin{enumerate}[label=(\roman*)]
  \item the average number of expert calls after every label has been observed at least once;
  \item the average number of incorrect guesses;
  \item the average Voronoi regret
\end{enumerate}
over the $T=5000$ rounds.

\subsubsection{Evaluation of \vhc, \vht\ and Comparison with Baselines}\label{app:synthetic_ghc_skm} First, we report the performance of \vht\ and \vhc. We recall that \vhc\ corresponds to \vht\ for $\tau=0.$  We also present the results of two baseline algorithms. 
\paragraph{Sequential $k$-Means (\texttt{SKM}) baseline.} The first baseline is based on sequential $k$-means \citep{MacQueen1967}. At each round, it stores $N$ estimated centroids $\{\hat{c}_{i}\}_{i \in [N]}$. If the algorithm guesses that the label of $q$ is $i$, the corresponding centroid is updated as $\hat{c}_{i} \leftarrow \hat{c}_{i}+(q-\hat{c}_{i})/\kappa_i$ where $\kappa_i$ is the number of times label $i$ was guessed up to the current round. This algorithm has two phases:
\begin{enumerate}[label=(\roman*)]
    \item In phase 1 (initialization phase), the expert is called until all labels are observed.
    \item In phase 2 (unsupervised phase), the algorithm predicts the label of the nearest centroid: $\hat{\imath} = \arg \min_{i \in [N]} \Vert q - \hat{c}_i \Vert_2$
\end{enumerate}
 Since this baseline never calls the expert after all labels are observed, it is a relevant comparison to \vht\ with threshold $\tau=1.$ 
 
\paragraph{Active Multiclass Perceptron (\texttt{AMP}) baseline.} The second baseline is a selective variant of the multiclass perceptron. 
The standard multiclass perceptron was studied in \cite{crammer2003}, and in the binary classification setting, an active variant was proposed in \cite{cesa2004}. Although these works study an adversarial setting under a margin assumption, whereas we study the stochastic i.i.d. regime, such a method is a reasonable baseline in practice if the classes are separated by hyperplanes passing through the origin. This is the case when the query domain is 
$\mathcal{S}^{d-1}$. \\

The \texttt{AMP} baseline takes as input a threshold $\tau \in [0,1]$ and first initializes weights $w_1,\dots,w_N$ at $0$. Then at each round $t=1,\dots,T$: 
\begin{enumerate}[label=(\roman*)]
\item It calculates normalized scores $q_t^{\top}v_i$ where $v_i = w_i / \Vert w_i \Vert$ for the query $q_t$ and identifies the best candidate label $j=\arg \max_{i} q_t^{\top}v_i$ and second-best candidate label $k= \arg \max_{i \ne j} q_t^{\top}v_{j}$.
\item If $q_t^\top(v_j-v_k)\le 2(1-\tau)$, it queries the expert for the true label $i_t$. It then updates weights by setting $w_{j} \leftarrow w_{j}+q_t$ if $j=i_t$, and $w_j \leftarrow w_j-q_t$ if $j \ne i_t$. 
\item If $q_t^\top(v_j-v_k) > 2(1-\tau)$, it predicts $\hat{\imath}_t=j$. 
\end{enumerate}
We note that typically, Perceptron (and its active variants) only update weights when the prediction is wrong \citep{cesa2004}, i.e. it would set $w_{i_t} \leftarrow w_{i_t} + q_t$ and $w_j \leftarrow w_j - q_t$ only if $j \ne i_t$. However, we found better empirical performance by also updating weights when the prediction matches the true label.

For each setting, we report the performance metrics of \texttt{AMP} on the sphere for the threshold value with lowest regret among $\{0.05k: k \in \{0,\dots,20\}\}$.
As for \vht, we do not report the number of calls in the first phase, although they are included in the regret calculation. Results are reported in Tables \ref{table:full_results_table_cube} and \ref{table:full_results_table_sphere} for $\mathcal{I}^d$ and $\Sph^{d-1}$ respectively. ``C'', ``M'' and ``R'' denote the mean $\pm$ standard deviation of the number of expert calls in the second phase, the number of incorrect guesses and the cumulative regret respectively. 

We observe that the optimal threshold increases with the dimension: in dimension one, keeping a low threshold is optimal and the regret, mostly driven by the number of expert calls, is limited. Already in dimension $10$, \vhc\ almost calls the expert for the $5,000$ queries, as expected from the discussion following Corollary \ref{cor:explicit_regret_bound}. As the dimension grows large, the optimal threshold appears to rapidly approach one. In the mixture case, the performance of \vht\ is surprisingly better in large dimension. A likely explanation in the $\mathcal{E}=\mathcal{I}^d$ is the $\sqrt{d}$ scaling of the average distance between the uniformly sampled seeds in $\mathcal{I}^d$ \citep{anderssen1976concerning}, making the Gaussian clusters more separated since their covariance is fixed. The same regret improvement is observed when $\mathcal{E}=\Sph^{d-1}$, even though the average distance between the uniformly sampled seeds will remain bounded by $2$. This average distance in fact approaches $\sqrt{2}$ as $d \to \infty$, because the probability that two seeds are orthogonal approaches one as $d \to \infty$. This entails that the centers are more evenly separated in higher dimensions and could explain why better regret is observed in these cases. 
The performance of \texttt{SKM} is worse than \vht\ for any threshold in the uniform settings, and comparable to \vht\ with high threshold in the mixture settings. \vht\ also always outperforms the \texttt{AMP} baseline, especially in lower dimensions.


\begin{table}[!htbp]
\centering
    \caption{Performance metrics of \vht\ and \texttt{SKM} on $\mathcal{I}^{d}$ ($T=5000$).}
    \label{table:full_results_table_cube}
    \begin{adjustbox}{max width=\textwidth,center}
    \begin{tblr}{colspec = { | Q[c,m]| Q[l,m]| Q[c,m]| Q[c,m]| Q[c,m]| Q[c,m]| Q[c,m]| Q[c,m]| Q[c,m]| Q[c,m]| Q[c,m]| },row{1} = {font=\small\bfseries,m},row{2-Z} = {font=\small,m},hlines,colsep=3pt, rowsep=3pt}
    Dim & Dist. & $\tau=0.00$ & $\tau=0.10$ & $\tau=0.40$ & $\tau=0.60$ & $\tau=0.80$ & $\tau=0.90$ & $\tau=0.95$ & $\tau=1.00$ & \texttt{SKM}\\
    \SetCell[r=2]{m} 1 & Unif. & \begin{tabular}{l@{\:}r@{}}
    \scriptsize C: & \scriptsize $59\,\pm\,7$ \\
    \scriptsize M: & \scriptsize $0\,\pm\,0$ \\
    \scriptsize R: & \scriptsize $142\,\pm\,11$
\end{tabular} & \SetCell{bg=gray!20} \begin{tabular}{l@{\:}r@{}}
    \scriptsize C: & \scriptsize $42\,\pm\,6$ \\
    \scriptsize M: & \scriptsize $0\,\pm\,0$ \\
    \scriptsize\textbf{R:} & \scriptsize\bm{$110\,\pm\,7$}
\end{tabular} & \begin{tabular}{l@{\:}r@{}}
    \scriptsize C: & \scriptsize $29\,\pm\,3$ \\
    \scriptsize M: & \scriptsize $3\,\pm\,3$ \\
    \scriptsize R: & \scriptsize $119\,\pm\,32$
\end{tabular} & \begin{tabular}{l@{\:}r@{}}
    \scriptsize C: & \scriptsize $26\,\pm\,2$ \\
    \scriptsize M: & \scriptsize $12\,\pm\,3$ \\
    \scriptsize R: & \scriptsize $211\,\pm\,39$
\end{tabular} & \begin{tabular}{l@{\:}r@{}}
    \scriptsize C: & \scriptsize $20\,\pm\,3$ \\
    \scriptsize M: & \scriptsize $34\,\pm\,10$ \\
    \scriptsize R: & \scriptsize $434\,\pm\,113$
\end{tabular} & \begin{tabular}{l@{\:}r@{}}
    \scriptsize C: & \scriptsize $16\,\pm\,3$ \\
    \scriptsize M: & \scriptsize $68\,\pm\,20$ \\
    \scriptsize R: & \scriptsize $799\,\pm\,212$
\end{tabular} & \begin{tabular}{l@{\:}r@{}}
    \scriptsize C: & \scriptsize $12\,\pm\,2$ \\
    \scriptsize M: & \scriptsize $101\,\pm\,33$ \\
    \scriptsize R: & \scriptsize $1161\,\pm\,365$
\end{tabular} & \begin{tabular}{l@{\:}r@{}}
    \scriptsize C: & \scriptsize $0\,\pm\,0$ \\
    \scriptsize M: & \scriptsize $375\,\pm\,126$ \\
    \scriptsize R: & \scriptsize $4152\,\pm\,1387$
\end{tabular} & \begin{tabular}{l@{\:}r@{}}
    \scriptsize C: & \scriptsize $0\,\pm\,0$ \\
    \scriptsize M: & \scriptsize $1282\,\pm\,389$ \\
    \scriptsize R: & \scriptsize $14128\,\pm\,4275$
\end{tabular}\\
     & Mix. & \begin{tabular}{l@{\:}r@{}}
    \scriptsize C: & \scriptsize $55\,\pm\,10$ \\
    \scriptsize M: & \scriptsize $0\,\pm\,0$ \\
    \scriptsize R: & \scriptsize $130\,\pm\,21$
\end{tabular} & \SetCell{bg=gray!20} \begin{tabular}{l@{\:}r@{}}
    \scriptsize C: & \scriptsize $41\,\pm\,9$ \\
    \scriptsize M: & \scriptsize $0\,\pm\,0$ \\
    \scriptsize\textbf{R:} & \scriptsize\bm{$106\,\pm\,18$}
\end{tabular} & \begin{tabular}{l@{\:}r@{}}
    \scriptsize C: & \scriptsize $34\,\pm\,4$ \\
    \scriptsize M: & \scriptsize $9\,\pm\,5$ \\
    \scriptsize R: & \scriptsize $181\,\pm\,53$
\end{tabular} & \begin{tabular}{l@{\:}r@{}}
    \scriptsize C: & \scriptsize $29\,\pm\,1$ \\
    \scriptsize M: & \scriptsize $21\,\pm\,7$ \\
    \scriptsize R: & \scriptsize $310\,\pm\,78$
\end{tabular} & \begin{tabular}{l@{\:}r@{}}
    \scriptsize C: & \scriptsize $24\,\pm\,1$ \\
    \scriptsize M: & \scriptsize $42\,\pm\,14$ \\
    \scriptsize R: & \scriptsize $525\,\pm\,148$
\end{tabular} & \begin{tabular}{l@{\:}r@{}}
    \scriptsize C: & \scriptsize $20\,\pm\,1$ \\
    \scriptsize M: & \scriptsize $87\,\pm\,22$ \\
    \scriptsize R: & \scriptsize $1013\,\pm\,246$
\end{tabular} & \begin{tabular}{l@{\:}r@{}}
    \scriptsize C: & \scriptsize $15\,\pm\,2$ \\
    \scriptsize M: & \scriptsize $155\,\pm\,31$ \\
    \scriptsize R: & \scriptsize $1752\,\pm\,336$
\end{tabular} & \begin{tabular}{l@{\:}r@{}}
    \scriptsize C: & \scriptsize $0\,\pm\,0$ \\
    \scriptsize M: & \scriptsize $823\,\pm\,389$ \\
    \scriptsize R: & \scriptsize $9070\,\pm\,4275$
\end{tabular} & \begin{tabular}{l@{\:}r@{}}
    \scriptsize C: & \scriptsize $0\,\pm\,0$ \\
    \scriptsize M: & \scriptsize $1011\,\pm\,340$ \\
    \scriptsize R: & \scriptsize $11141\,\pm\,3741$
\end{tabular}\\
    \SetCell[r=2]{m} 4 & Unif. & \begin{tabular}{l@{\:}r@{}}
    \scriptsize C: & \scriptsize $1476\,\pm\,35$ \\
    \scriptsize M: & \scriptsize $0\,\pm\,0$ \\
    \scriptsize R: & \scriptsize $2972\,\pm\,72$
\end{tabular} & \begin{tabular}{l@{\:}r@{}}
    \scriptsize C: & \scriptsize $1657\,\pm\,45$ \\
    \scriptsize M: & \scriptsize $0\,\pm\,0$ \\
    \scriptsize R: & \scriptsize $3336\,\pm\,92$
\end{tabular} & \SetCell{bg=gray!20} \begin{tabular}{l@{\:}r@{}}
    \scriptsize C: & \scriptsize $722\,\pm\,29$ \\
    \scriptsize M: & \scriptsize $12\,\pm\,4$ \\
    \scriptsize\textbf{R:} & \scriptsize\bm{$1593\,\pm\,35$}
\end{tabular} & \begin{tabular}{l@{\:}r@{}}
    \scriptsize C: & \scriptsize $478\,\pm\,25$ \\
    \scriptsize M: & \scriptsize $60\,\pm\,7$ \\
    \scriptsize R: & \scriptsize $1633\,\pm\,55$
\end{tabular} & \begin{tabular}{l@{\:}r@{}}
    \scriptsize C: & \scriptsize $291\,\pm\,13$ \\
    \scriptsize M: & \scriptsize $191\,\pm\,15$ \\
    \scriptsize R: & \scriptsize $2706\,\pm\,159$
\end{tabular} & \begin{tabular}{l@{\:}r@{}}
    \scriptsize C: & \scriptsize $194\,\pm\,10$ \\
    \scriptsize M: & \scriptsize $376\,\pm\,20$ \\
    \scriptsize R: & \scriptsize $4550\,\pm\,201$
\end{tabular} & \begin{tabular}{l@{\:}r@{}}
    \scriptsize C: & \scriptsize $132\,\pm\,5$ \\
    \scriptsize M: & \scriptsize $576\,\pm\,46$ \\
    \scriptsize R: & \scriptsize $6618\,\pm\,511$
\end{tabular} & \begin{tabular}{l@{\:}r@{}}
    \scriptsize C: & \scriptsize $0\,\pm\,0$ \\
    \scriptsize M: & \scriptsize $2415\,\pm\,324$ \\
    \scriptsize R: & \scriptsize $26584\,\pm\,3559$
\end{tabular} & \begin{tabular}{l@{\:}r@{}}
    \scriptsize C: & \scriptsize $0\,\pm\,0$ \\
    \scriptsize M: & \scriptsize $2730\,\pm\,333$ \\
    \scriptsize R: & \scriptsize $30055\,\pm\,3659$
\end{tabular}\\
     & Mix. & \begin{tabular}{l@{\:}r@{}}
    \scriptsize C: & \scriptsize $1159\,\pm\,7$ \\
    \scriptsize M: & \scriptsize $0\,\pm\,0$ \\
    \scriptsize R: & \scriptsize $2337\,\pm\,19$
\end{tabular} & \begin{tabular}{l@{\:}r@{}}
    \scriptsize C: & \scriptsize $933\,\pm\,48$ \\
    \scriptsize M: & \scriptsize $0\,\pm\,0$ \\
    \scriptsize R: & \scriptsize $1884\,\pm\,100$
\end{tabular} & \begin{tabular}{l@{\:}r@{}}
    \scriptsize C: & \scriptsize $334\,\pm\,15$ \\
    \scriptsize M: & \scriptsize $2\,\pm\,2$ \\
    \scriptsize R: & \scriptsize $712\,\pm\,28$
\end{tabular} & \SetCell{bg=gray!20} \begin{tabular}{l@{\:}r@{}}
    \scriptsize C: & \scriptsize $214\,\pm\,13$ \\
    \scriptsize M: & \scriptsize $12\,\pm\,3$ \\
    \scriptsize\textbf{R:} & \scriptsize\bm{$573\,\pm\,52$}
\end{tabular} & \begin{tabular}{l@{\:}r@{}}
    \scriptsize C: & \scriptsize $130\,\pm\,10$ \\
    \scriptsize M: & \scriptsize $43\,\pm\,11$ \\
    \scriptsize R: & \scriptsize $754\,\pm\,113$
\end{tabular} & \begin{tabular}{l@{\:}r@{}}
    \scriptsize C: & \scriptsize $83\,\pm\,8$ \\
    \scriptsize M: & \scriptsize $86\,\pm\,15$ \\
    \scriptsize R: & \scriptsize $1136\,\pm\,152$
\end{tabular} & \begin{tabular}{l@{\:}r@{}}
    \scriptsize C: & \scriptsize $55\,\pm\,7$ \\
    \scriptsize M: & \scriptsize $120\,\pm\,16$ \\
    \scriptsize R: & \scriptsize $1453\,\pm\,173$
\end{tabular} & \begin{tabular}{l@{\:}r@{}}
    \scriptsize C: & \scriptsize $0\,\pm\,0$ \\
    \scriptsize M: & \scriptsize $444\,\pm\,304$ \\
    \scriptsize R: & \scriptsize $4902\,\pm\,3346$
\end{tabular} & \begin{tabular}{l@{\:}r@{}}
    \scriptsize C: & \scriptsize $0\,\pm\,0$ \\
    \scriptsize M: & \scriptsize $70\,\pm\,7$ \\
    \scriptsize R: & \scriptsize $793\,\pm\,94$
\end{tabular}\\
    \SetCell[r=2]{m} 10 & Unif. & \begin{tabular}{l@{\:}r@{}}
    \scriptsize C: & \scriptsize $4725\,\pm\,25$ \\
    \scriptsize M: & \scriptsize $0\,\pm\,0$ \\
    \scriptsize R: & \scriptsize $9489\,\pm\,35$
\end{tabular} & \begin{tabular}{l@{\:}r@{}}
    \scriptsize C: & \scriptsize $4783\,\pm\,24$ \\
    \scriptsize M: & \scriptsize $0\,\pm\,0$ \\
    \scriptsize R: & \scriptsize $9605\,\pm\,33$
\end{tabular} & \begin{tabular}{l@{\:}r@{}}
    \scriptsize C: & \scriptsize $3762\,\pm\,47$ \\
    \scriptsize M: & \scriptsize $1\,\pm\,0$ \\
    \scriptsize R: & \scriptsize $7569\,\pm\,73$
\end{tabular} & \begin{tabular}{l@{\:}r@{}}
    \scriptsize C: & \scriptsize $2665\,\pm\,50$ \\
    \scriptsize M: & \scriptsize $22\,\pm\,5$ \\
    \scriptsize R: & \scriptsize $5607\,\pm\,74$
\end{tabular} & \SetCell{bg=gray!20} \begin{tabular}{l@{\:}r@{}}
    \scriptsize C: & \scriptsize $1517\,\pm\,30$ \\
    \scriptsize M: & \scriptsize $196\,\pm\,10$ \\
    \scriptsize\textbf{R:} & \scriptsize\bm{$5233\,\pm\,94$}
\end{tabular} & \begin{tabular}{l@{\:}r@{}}
    \scriptsize C: & \scriptsize $903\,\pm\,23$ \\
    \scriptsize M: & \scriptsize $489\,\pm\,17$ \\
    \scriptsize R: & \scriptsize $7226\,\pm\,186$
\end{tabular} & \begin{tabular}{l@{\:}r@{}}
    \scriptsize C: & \scriptsize $528\,\pm\,22$ \\
    \scriptsize M: & \scriptsize $816\,\pm\,48$ \\
    \scriptsize R: & \scriptsize $10068\,\pm\,487$
\end{tabular} & \begin{tabular}{l@{\:}r@{}}
    \scriptsize C: & \scriptsize $0\,\pm\,0$ \\
    \scriptsize M: & \scriptsize $2599\,\pm\,512$ \\
    \scriptsize R: & \scriptsize $28628\,\pm\,5600$
\end{tabular} & \begin{tabular}{l@{\:}r@{}}
    \scriptsize C: & \scriptsize $0\,\pm\,0$ \\
    \scriptsize M: & \scriptsize $3533\,\pm\,389$ \\
    \scriptsize R: & \scriptsize $38900\,\pm\,4268$
\end{tabular}\\
     & Mix. & \begin{tabular}{l@{\:}r@{}}
    \scriptsize C: & \scriptsize $4400\,\pm\,23$ \\
    \scriptsize M: & \scriptsize $0\,\pm\,0$ \\
    \scriptsize R: & \scriptsize $8821\,\pm\,48$
\end{tabular} & \begin{tabular}{l@{\:}r@{}}
    \scriptsize C: & \scriptsize $2837\,\pm\,43$ \\
    \scriptsize M: & \scriptsize $0\,\pm\,0$ \\
    \scriptsize R: & \scriptsize $5695\,\pm\,88$
\end{tabular} & \begin{tabular}{l@{\:}r@{}}
    \scriptsize C: & \scriptsize $132\,\pm\,6$ \\
    \scriptsize M: & \scriptsize $0\,\pm\,0$ \\
    \scriptsize R: & \scriptsize $284\,\pm\,9$
\end{tabular} & \begin{tabular}{l@{\:}r@{}}
    \scriptsize C: & \scriptsize $23\,\pm\,3$ \\
    \scriptsize M: & \scriptsize $0\,\pm\,0$ \\
    \scriptsize R: & \scriptsize $66\,\pm\,9$
\end{tabular} & \begin{tabular}{l@{\:}r@{}}
    \scriptsize C: & \scriptsize $3\,\pm\,2$ \\
    \scriptsize M: & \scriptsize $0\,\pm\,0$ \\
    \scriptsize R: & \scriptsize $27\,\pm\,4$
\end{tabular} & \begin{tabular}{l@{\:}r@{}}
    \scriptsize C: & \scriptsize $1\,\pm\,1$ \\
    \scriptsize M: & \scriptsize $0\,\pm\,0$ \\
    \scriptsize R: & \scriptsize $23\,\pm\,4$
\end{tabular} & \SetCell{bg=gray!20} \begin{tabular}{l@{\:}r@{}}
    \scriptsize C: & \scriptsize $1\,\pm\,1$ \\
    \scriptsize M: & \scriptsize $0\,\pm\,0$ \\
    \scriptsize\textbf{R:} & \scriptsize\bm{$23\,\pm\,4$}
\end{tabular} & \begin{tabular}{l@{\:}r@{}}
    \scriptsize C: & \scriptsize $0\,\pm\,0$ \\
    \scriptsize M: & \scriptsize $1\,\pm\,1$ \\
    \scriptsize R: & \scriptsize $31\,\pm\,7$
\end{tabular} & \begin{tabular}{l@{\:}r@{}}
    \scriptsize C: & \scriptsize $0\,\pm\,0$ \\
    \scriptsize M: & \scriptsize $0\,\pm\,0$ \\
    \scriptsize R: & \scriptsize $20\,\pm\,5$
\end{tabular}\\
    \SetCell[r=2]{m} 50 & Unif. & \begin{tabular}{l@{\:}r@{}}
    \scriptsize C: & \scriptsize $4981\,\pm\,6$ \\
    \scriptsize M: & \scriptsize $0\,\pm\,0$ \\
    \scriptsize R: & \scriptsize $10000\,\pm\,0$
\end{tabular} & \begin{tabular}{l@{\:}r@{}}
    \scriptsize C: & \scriptsize $4981\,\pm\,6$ \\
    \scriptsize M: & \scriptsize $0\,\pm\,0$ \\
    \scriptsize R: & \scriptsize $10000\,\pm\,0$
\end{tabular} & \begin{tabular}{l@{\:}r@{}}
    \scriptsize C: & \scriptsize $4981\,\pm\,6$ \\
    \scriptsize M: & \scriptsize $0\,\pm\,0$ \\
    \scriptsize R: & \scriptsize $10000\,\pm\,0$
\end{tabular} & \begin{tabular}{l@{\:}r@{}}
    \scriptsize C: & \scriptsize $4981\,\pm\,6$ \\
    \scriptsize M: & \scriptsize $0\,\pm\,0$ \\
    \scriptsize R: & \scriptsize $10000\,\pm\,0$
\end{tabular} & \begin{tabular}{l@{\:}r@{}}
    \scriptsize C: & \scriptsize $4976\,\pm\,6$ \\
    \scriptsize M: & \scriptsize $0\,\pm\,0$ \\
    \scriptsize R: & \scriptsize $9991\,\pm\,3$
\end{tabular} & \SetCell{bg=gray!20} \begin{tabular}{l@{\:}r@{}}
    \scriptsize C: & \scriptsize $4569\,\pm\,11$ \\
    \scriptsize M: & \scriptsize $35\,\pm\,5$ \\
    \scriptsize\textbf{R:} & \scriptsize\bm{$9559\,\pm\,36$}
\end{tabular} & \begin{tabular}{l@{\:}r@{}}
    \scriptsize C: & \scriptsize $3284\,\pm\,31$ \\
    \scriptsize M: & \scriptsize $356\,\pm\,21$ \\
    \scriptsize R: & \scriptsize $10525\,\pm\,177$
\end{tabular} & \begin{tabular}{l@{\:}r@{}}
    \scriptsize C: & \scriptsize $0\,\pm\,0$ \\
    \scriptsize M: & \scriptsize $3155\,\pm\,101$ \\
    \scriptsize R: & \scriptsize $34741\,\pm\,1097$
\end{tabular} & \begin{tabular}{l@{\:}r@{}}
    \scriptsize C: & \scriptsize $0\,\pm\,0$ \\
    \scriptsize M: & \scriptsize $3875\,\pm\,307$ \\
    \scriptsize R: & \scriptsize $42659\,\pm\,3378$
\end{tabular}\\
     & Mix. & \begin{tabular}{l@{\:}r@{}}
    \scriptsize C: & \scriptsize $4989\,\pm\,3$ \\
    \scriptsize M: & \scriptsize $0\,\pm\,0$ \\
    \scriptsize R: & \scriptsize $10000\,\pm\,0$
\end{tabular} & \begin{tabular}{l@{\:}r@{}}
    \scriptsize C: & \scriptsize $4989\,\pm\,3$ \\
    \scriptsize M: & \scriptsize $0\,\pm\,0$ \\
    \scriptsize R: & \scriptsize $10000\,\pm\,0$
\end{tabular} & \begin{tabular}{l@{\:}r@{}}
    \scriptsize C: & \scriptsize $7\,\pm\,2$ \\
    \scriptsize M: & \scriptsize $0\,\pm\,0$ \\
    \scriptsize R: & \scriptsize $37\,\pm\,2$
\end{tabular} & \begin{tabular}{l@{\:}r@{}}
    \scriptsize C: & \scriptsize $0\,\pm\,0$ \\
    \scriptsize M: & \scriptsize $0\,\pm\,0$ \\
    \scriptsize R: & \scriptsize $23\,\pm\,5$
\end{tabular} & \begin{tabular}{l@{\:}r@{}}
    \scriptsize C: & \scriptsize $0\,\pm\,0$ \\
    \scriptsize M: & \scriptsize $0\,\pm\,0$ \\
    \scriptsize R: & \scriptsize $23\,\pm\,5$
\end{tabular} & \begin{tabular}{l@{\:}r@{}}
    \scriptsize C: & \scriptsize $0\,\pm\,0$ \\
    \scriptsize M: & \scriptsize $0\,\pm\,0$ \\
    \scriptsize R: & \scriptsize $23\,\pm\,5$
\end{tabular} & \begin{tabular}{l@{\:}r@{}}
    \scriptsize C: & \scriptsize $0\,\pm\,0$ \\
    \scriptsize M: & \scriptsize $0\,\pm\,0$ \\
    \scriptsize R: & \scriptsize $23\,\pm\,5$
\end{tabular} & \SetCell{bg=gray!20} \begin{tabular}{l@{\:}r@{}}
    \scriptsize C: & \scriptsize $0\,\pm\,0$ \\
    \scriptsize M: & \scriptsize $0\,\pm\,0$ \\
    \scriptsize\textbf{R:} & \scriptsize\bm{$23\,\pm\,5$}
\end{tabular} & \begin{tabular}{l@{\:}r@{}}
    \scriptsize C: & \scriptsize $0\,\pm\,0$ \\
    \scriptsize M: & \scriptsize $0\,\pm\,0$ \\
    \scriptsize R: & \scriptsize $23\,\pm\,5$
\end{tabular}\\
    \end{tblr}
    \end{adjustbox}
\end{table}

\begin{table}[!htbp]
\centering
    \caption{Performance metrics of \vht, \texttt{SKM} and \texttt{AMP} on $\Sph^{d-1}$ ($T=5000$).}
    \label{table:full_results_table_sphere}
    \begin{adjustbox}{max width=\textwidth,center}
    \begin{tblr}{colspec = { | Q[c,m]| Q[l,m]| Q[c,m]| Q[c,m]| Q[c,m]| Q[c,m]| Q[c,m]| Q[c,m]| Q[c,m]| Q[c,m]| Q[c,m]| },row{1} = {font=\small\bfseries,m},row{2-Z} = {font=\small,m},hlines,colsep=3pt, rowsep=3pt}
    Dim & Dist. & $\tau=0.00$ & $\tau=0.10$ & $\tau=0.20$ & $\tau=0.40$ & $\tau=0.60$ & $\tau=0.80$ & $\tau=1.00$ & \texttt{SKM} & \texttt{AMP}\\
    \SetCell[r=2]{m} 2 & Unif. & \begin{tabular}{l@{\:}r@{}}
    \scriptsize C: & \scriptsize $50\,\pm\,7$ \\
    \scriptsize M: & \scriptsize $0\,\pm\,0$ \\
    \scriptsize R: & \scriptsize $132\,\pm\,12$
\end{tabular} & \SetCell{bg=gray!20} \begin{tabular}{l@{\:}r@{}}
    \scriptsize C: & \scriptsize $43\,\pm\,5$ \\
    \scriptsize M: & \scriptsize $1\,\pm\,1$ \\
    \scriptsize\textbf{R:} & \scriptsize\bm{$125\,\pm\,6$}
\end{tabular} & \begin{tabular}{l@{\:}r@{}}
    \scriptsize C: & \scriptsize $41\,\pm\,3$ \\
    \scriptsize M: & \scriptsize $2\,\pm\,1$ \\
    \scriptsize R: & \scriptsize $142\,\pm\,7$
\end{tabular} & \begin{tabular}{l@{\:}r@{}}
    \scriptsize C: & \scriptsize $35\,\pm\,2$ \\
    \scriptsize M: & \scriptsize $5\,\pm\,3$ \\
    \scriptsize R: & \scriptsize $157\,\pm\,30$
\end{tabular} & \begin{tabular}{l@{\:}r@{}}
    \scriptsize C: & \scriptsize $28\,\pm\,1$ \\
    \scriptsize M: & \scriptsize $18\,\pm\,6$ \\
    \scriptsize R: & \scriptsize $288\,\pm\,69$
\end{tabular} & \begin{tabular}{l@{\:}r@{}}
    \scriptsize C: & \scriptsize $22\,\pm\,2$ \\
    \scriptsize M: & \scriptsize $44\,\pm\,11$ \\
    \scriptsize R: & \scriptsize $563\,\pm\,128$
\end{tabular} & \begin{tabular}{l@{\:}r@{}}
    \scriptsize C: & \scriptsize $0\,\pm\,0$ \\
    \scriptsize M: & \scriptsize $315\,\pm\,148$ \\
    \scriptsize R: & \scriptsize $3495\,\pm\,1633$
\end{tabular} & \begin{tabular}{l@{\:}r@{}}
    \scriptsize C: & \scriptsize $0\,\pm\,0$ \\
    \scriptsize M: & \scriptsize $2431\,\pm\,174$ \\
    \scriptsize R: & \scriptsize $26804\,\pm\,1923$
\end{tabular} & \begin{tabular}{l@{\:}r@{}}
    \scriptsize C: & \scriptsize $1856\,\pm\,46$ \\
    \scriptsize M: & \scriptsize $114\,\pm\,42$ \\
    \scriptsize R: & \scriptsize $4967\,\pm\,421$
\end{tabular}\\
     & Mix. & \begin{tabular}{l@{\:}r@{}}
    \scriptsize C: & \scriptsize $61\,\pm\,3$ \\
    \scriptsize M: & \scriptsize $0\,\pm\,0$ \\
    \scriptsize R: & \scriptsize $138\,\pm\,8$
\end{tabular} & \begin{tabular}{l@{\:}r@{}}
    \scriptsize C: & \scriptsize $47\,\pm\,4$ \\
    \scriptsize M: & \scriptsize $0\,\pm\,0$ \\
    \scriptsize R: & \scriptsize $114\,\pm\,8$
\end{tabular} & \SetCell{bg=gray!20} \begin{tabular}{l@{\:}r@{}}
    \scriptsize C: & \scriptsize $38\,\pm\,2$ \\
    \scriptsize M: & \scriptsize $1\,\pm\,1$ \\
    \scriptsize\textbf{R:} & \scriptsize\bm{$104\,\pm\,13$}
\end{tabular} & \begin{tabular}{l@{\:}r@{}}
    \scriptsize C: & \scriptsize $29\,\pm\,1$ \\
    \scriptsize M: & \scriptsize $3\,\pm\,1$ \\
    \scriptsize R: & \scriptsize $110\,\pm\,17$
\end{tabular} & \begin{tabular}{l@{\:}r@{}}
    \scriptsize C: & \scriptsize $24\,\pm\,1$ \\
    \scriptsize M: & \scriptsize $9\,\pm\,3$ \\
    \scriptsize R: & \scriptsize $166\,\pm\,31$
\end{tabular} & \begin{tabular}{l@{\:}r@{}}
    \scriptsize C: & \scriptsize $18\,\pm\,2$ \\
    \scriptsize M: & \scriptsize $22\,\pm\,9$ \\
    \scriptsize R: & \scriptsize $299\,\pm\,98$
\end{tabular} & \begin{tabular}{l@{\:}r@{}}
    \scriptsize C: & \scriptsize $0\,\pm\,0$ \\
    \scriptsize M: & \scriptsize $345\,\pm\,77$ \\
    \scriptsize R: & \scriptsize $3810\,\pm\,847$
\end{tabular} & \begin{tabular}{l@{\:}r@{}}
    \scriptsize C: & \scriptsize $0\,\pm\,0$ \\
    \scriptsize M: & \scriptsize $70\,\pm\,58$ \\
    \scriptsize R: & \scriptsize $791\,\pm\,645$
\end{tabular} & \begin{tabular}{l@{\:}r@{}}
    \scriptsize C: & \scriptsize $5000\,\pm\,0$ \\
    \scriptsize M: & \scriptsize $0\,\pm\,0$ \\
    \scriptsize R: & \scriptsize $10000\,\pm\,0$
\end{tabular}\\
    \SetCell[r=2]{m} 4 & Unif. & \begin{tabular}{l@{\:}r@{}}
    \scriptsize C: & \scriptsize $602\,\pm\,15$ \\
    \scriptsize M: & \scriptsize $0\,\pm\,0$ \\
    \scriptsize R: & \scriptsize $1225\,\pm\,31$
\end{tabular} & \begin{tabular}{l@{\:}r@{}}
    \scriptsize C: & \scriptsize $514\,\pm\,11$ \\
    \scriptsize M: & \scriptsize $0\,\pm\,0$ \\
    \scriptsize R: & \scriptsize $1053\,\pm\,26$
\end{tabular} & \begin{tabular}{l@{\:}r@{}}
    \scriptsize C: & \scriptsize $438\,\pm\,11$ \\
    \scriptsize M: & \scriptsize $2\,\pm\,1$ \\
    \scriptsize R: & \scriptsize $922\,\pm\,25$
\end{tabular} & \SetCell{bg=gray!20} \begin{tabular}{l@{\:}r@{}}
    \scriptsize C: & \scriptsize $329\,\pm\,10$ \\
    \scriptsize M: & \scriptsize $14\,\pm\,3$ \\
    \scriptsize\textbf{R:} & \scriptsize\bm{$836\,\pm\,40$}
\end{tabular} & \begin{tabular}{l@{\:}r@{}}
    \scriptsize C: & \scriptsize $246\,\pm\,6$ \\
    \scriptsize M: & \scriptsize $55\,\pm\,11$ \\
    \scriptsize R: & \scriptsize $1122\,\pm\,108$
\end{tabular} & \begin{tabular}{l@{\:}r@{}}
    \scriptsize C: & \scriptsize $177\,\pm\,7$ \\
    \scriptsize M: & \scriptsize $137\,\pm\,4$ \\
    \scriptsize R: & \scriptsize $1878\,\pm\,52$
\end{tabular} & \begin{tabular}{l@{\:}r@{}}
    \scriptsize C: & \scriptsize $0\,\pm\,0$ \\
    \scriptsize M: & \scriptsize $1909\,\pm\,339$ \\
    \scriptsize R: & \scriptsize $21016\,\pm\,3728$
\end{tabular} & \begin{tabular}{l@{\:}r@{}}
    \scriptsize C: & \scriptsize $0\,\pm\,0$ \\
    \scriptsize M: & \scriptsize $2329\,\pm\,742$ \\
    \scriptsize R: & \scriptsize $25645\,\pm\,8157$
\end{tabular} & \begin{tabular}{l@{\:}r@{}}
    \scriptsize C: & \scriptsize $869\,\pm\,23$ \\
    \scriptsize M: & \scriptsize $96\,\pm\,10$ \\
    \scriptsize R: & \scriptsize $2799\,\pm\,110$
\end{tabular}\\
     & Mix. & \begin{tabular}{l@{\:}r@{}}
    \scriptsize C: & \scriptsize $531\,\pm\,29$ \\
    \scriptsize M: & \scriptsize $0\,\pm\,0$ \\
    \scriptsize R: & \scriptsize $1083\,\pm\,56$
\end{tabular} & \begin{tabular}{l@{\:}r@{}}
    \scriptsize C: & \scriptsize $127\,\pm\,4$ \\
    \scriptsize M: & \scriptsize $0\,\pm\,0$ \\
    \scriptsize R: & \scriptsize $275\,\pm\,5$
\end{tabular} & \begin{tabular}{l@{\:}r@{}}
    \scriptsize C: & \scriptsize $51\,\pm\,4$ \\
    \scriptsize M: & \scriptsize $0\,\pm\,0$ \\
    \scriptsize R: & \scriptsize $124\,\pm\,11$
\end{tabular} & \begin{tabular}{l@{\:}r@{}}
    \scriptsize C: & \scriptsize $12\,\pm\,2$ \\
    \scriptsize M: & \scriptsize $0\,\pm\,0$ \\
    \scriptsize R: & \scriptsize $46\,\pm\,7$
\end{tabular} & \begin{tabular}{l@{\:}r@{}}
    \scriptsize C: & \scriptsize $3\,\pm\,2$ \\
    \scriptsize M: & \scriptsize $0\,\pm\,0$ \\
    \scriptsize R: & \scriptsize $28\,\pm\,3$
\end{tabular} & \begin{tabular}{l@{\:}r@{}}
    \scriptsize C: & \scriptsize $0\,\pm\,0$ \\
    \scriptsize M: & \scriptsize $0\,\pm\,0$ \\
    \scriptsize R: & \scriptsize $22\,\pm\,6$
\end{tabular} & \SetCell{bg=gray!20} \begin{tabular}{l@{\:}r@{}}
    \scriptsize C: & \scriptsize $0\,\pm\,0$ \\
    \scriptsize M: & \scriptsize $0\,\pm\,0$ \\
    \scriptsize\textbf{R:} & \scriptsize\bm{$22\,\pm\,7$}
\end{tabular} & \begin{tabular}{l@{\:}r@{}}
    \scriptsize C: & \scriptsize $0\,\pm\,0$ \\
    \scriptsize M: & \scriptsize $0\,\pm\,0$ \\
    \scriptsize R: & \scriptsize $22\,\pm\,7$
\end{tabular} & \begin{tabular}{l@{\:}r@{}}
    \scriptsize C: & \scriptsize $70\,\pm\,46$ \\
    \scriptsize M: & \scriptsize $16\,\pm\,8$ \\
    \scriptsize R: & \scriptsize $312\,\pm\,170$
\end{tabular}\\
    \SetCell[r=2]{m} 10 & Unif. & \begin{tabular}{l@{\:}r@{}}
    \scriptsize C: & \scriptsize $2428\,\pm\,72$ \\
    \scriptsize M: & \scriptsize $0\,\pm\,0$ \\
    \scriptsize R: & \scriptsize $4878\,\pm\,150$
\end{tabular} & \begin{tabular}{l@{\:}r@{}}
    \scriptsize C: & \scriptsize $2167\,\pm\,74$ \\
    \scriptsize M: & \scriptsize $0\,\pm\,0$ \\
    \scriptsize R: & \scriptsize $4357\,\pm\,154$
\end{tabular} & \begin{tabular}{l@{\:}r@{}}
    \scriptsize C: & \scriptsize $1919\,\pm\,74$ \\
    \scriptsize M: & \scriptsize $0\,\pm\,0$ \\
    \scriptsize R: & \scriptsize $3862\,\pm\,155$
\end{tabular} & \begin{tabular}{l@{\:}r@{}}
    \scriptsize C: & \scriptsize $1430\,\pm\,54$ \\
    \scriptsize M: & \scriptsize $6\,\pm\,3$ \\
    \scriptsize R: & \scriptsize $2952\,\pm\,138$
\end{tabular} & \SetCell{bg=gray!20} \begin{tabular}{l@{\:}r@{}}
    \scriptsize C: & \scriptsize $1011\,\pm\,42$ \\
    \scriptsize M: & \scriptsize $46\,\pm\,6$ \\
    \scriptsize\textbf{R:} & \scriptsize\bm{$2550\,\pm\,123$}
\end{tabular} & \begin{tabular}{l@{\:}r@{}}
    \scriptsize C: & \scriptsize $632\,\pm\,21$ \\
    \scriptsize M: & \scriptsize $194\,\pm\,21$ \\
    \scriptsize R: & \scriptsize $3423\,\pm\,225$
\end{tabular} & \begin{tabular}{l@{\:}r@{}}
    \scriptsize C: & \scriptsize $0\,\pm\,0$ \\
    \scriptsize M: & \scriptsize $2990\,\pm\,94$ \\
    \scriptsize R: & \scriptsize $32914\,\pm\,1023$
\end{tabular} & \begin{tabular}{l@{\:}r@{}}
    \scriptsize C: & \scriptsize $0\,\pm\,0$ \\
    \scriptsize M: & \scriptsize $3630\,\pm\,358$ \\
    \scriptsize R: & \scriptsize $39956\,\pm\,3931$
\end{tabular} & \begin{tabular}{l@{\:}r@{}}
    \scriptsize C: & \scriptsize $1415\,\pm\,30$ \\
    \scriptsize M: & \scriptsize $97\,\pm\,14$ \\
    \scriptsize R: & \scriptsize $3902\,\pm\,158$
\end{tabular}\\
     & Mix. & \begin{tabular}{l@{\:}r@{}}
    \scriptsize C: & \scriptsize $3735\,\pm\,38$ \\
    \scriptsize M: & \scriptsize $0\,\pm\,0$ \\
    \scriptsize R: & \scriptsize $7490\,\pm\,74$
\end{tabular} & \begin{tabular}{l@{\:}r@{}}
    \scriptsize C: & \scriptsize $1439\,\pm\,36$ \\
    \scriptsize M: & \scriptsize $0\,\pm\,0$ \\
    \scriptsize R: & \scriptsize $2898\,\pm\,63$
\end{tabular} & \begin{tabular}{l@{\:}r@{}}
    \scriptsize C: & \scriptsize $430\,\pm\,13$ \\
    \scriptsize M: & \scriptsize $0\,\pm\,0$ \\
    \scriptsize R: & \scriptsize $879\,\pm\,21$
\end{tabular} & \begin{tabular}{l@{\:}r@{}}
    \scriptsize C: & \scriptsize $55\,\pm\,3$ \\
    \scriptsize M: & \scriptsize $0\,\pm\,0$ \\
    \scriptsize R: & \scriptsize $130\,\pm\,7$
\end{tabular} & \begin{tabular}{l@{\:}r@{}}
    \scriptsize C: & \scriptsize $6\,\pm\,3$ \\
    \scriptsize M: & \scriptsize $0\,\pm\,0$ \\
    \scriptsize R: & \scriptsize $31\,\pm\,7$
\end{tabular} & \begin{tabular}{l@{\:}r@{}}
    \scriptsize C: & \scriptsize $0\,\pm\,0$ \\
    \scriptsize M: & \scriptsize $0\,\pm\,0$ \\
    \scriptsize R: & \scriptsize $20\,\pm\,11$
\end{tabular} & \SetCell{bg=gray!20} \begin{tabular}{l@{\:}r@{}}
    \scriptsize C: & \scriptsize $0\,\pm\,0$ \\
    \scriptsize M: & \scriptsize $0\,\pm\,0$ \\
    \scriptsize\textbf{R:} & \scriptsize\bm{$20\,\pm\,11$}
\end{tabular} & \begin{tabular}{l@{\:}r@{}}
    \scriptsize C: & \scriptsize $0\,\pm\,0$ \\
    \scriptsize M: & \scriptsize $0\,\pm\,0$ \\
    \scriptsize R: & \scriptsize $20\,\pm\,11$
\end{tabular} & \begin{tabular}{l@{\:}r@{}}
    \scriptsize C: & \scriptsize $23\,\pm\,9$ \\
    \scriptsize M: & \scriptsize $0\,\pm\,0$ \\
    \scriptsize R: & \scriptsize $50\,\pm\,21$
\end{tabular}\\
    \SetCell[r=2]{m} 50 & Unif. & \begin{tabular}{l@{\:}r@{}}
    \scriptsize C: & \scriptsize $4979\,\pm\,5$ \\
    \scriptsize M: & \scriptsize $0\,\pm\,0$ \\
    \scriptsize R: & \scriptsize $9983\,\pm\,4$
\end{tabular} & \begin{tabular}{l@{\:}r@{}}
    \scriptsize C: & \scriptsize $4964\,\pm\,6$ \\
    \scriptsize M: & \scriptsize $0\,\pm\,0$ \\
    \scriptsize R: & \scriptsize $9953\,\pm\,6$
\end{tabular} & \begin{tabular}{l@{\:}r@{}}
    \scriptsize C: & \scriptsize $4936\,\pm\,7$ \\
    \scriptsize M: & \scriptsize $0\,\pm\,0$ \\
    \scriptsize R: & \scriptsize $9898\,\pm\,8$
\end{tabular} & \begin{tabular}{l@{\:}r@{}}
    \scriptsize C: & \scriptsize $4720\,\pm\,9$ \\
    \scriptsize M: & \scriptsize $0\,\pm\,0$ \\
    \scriptsize R: & \scriptsize $9464\,\pm\,18$
\end{tabular} & \begin{tabular}{l@{\:}r@{}}
    \scriptsize C: & \scriptsize $4119\,\pm\,20$ \\
    \scriptsize M: & \scriptsize $0\,\pm\,0$ \\
    \scriptsize R: & \scriptsize $8262\,\pm\,41$
\end{tabular} & \SetCell{bg=gray!20} \begin{tabular}{l@{\:}r@{}}
    \scriptsize C: & \scriptsize $2825\,\pm\,13$ \\
    \scriptsize M: & \scriptsize $42\,\pm\,8$ \\
    \scriptsize\textbf{R:} & \scriptsize\bm{$6135\,\pm\,73$}
\end{tabular} & \begin{tabular}{l@{\:}r@{}}
    \scriptsize C: & \scriptsize $0\,\pm\,0$ \\
    \scriptsize M: & \scriptsize $3502\,\pm\,53$ \\
    \scriptsize R: & \scriptsize $38549\,\pm\,582$
\end{tabular} & \begin{tabular}{l@{\:}r@{}}
    \scriptsize C: & \scriptsize $0\,\pm\,0$ \\
    \scriptsize M: & \scriptsize $3747\,\pm\,99$ \\
    \scriptsize R: & \scriptsize $41246\,\pm\,1091$
\end{tabular} & \begin{tabular}{l@{\:}r@{}}
    \scriptsize C: & \scriptsize $2879\,\pm\,32$ \\
    \scriptsize M: & \scriptsize $94\,\pm\,17$ \\
    \scriptsize R: & \scriptsize $6796\,\pm\,156$
\end{tabular}\\
     & Mix. & \begin{tabular}{l@{\:}r@{}}
    \scriptsize C: & \scriptsize $4990\,\pm\,2$ \\
    \scriptsize M: & \scriptsize $0\,\pm\,0$ \\
    \scriptsize R: & \scriptsize $10000\,\pm\,0$
\end{tabular} & \begin{tabular}{l@{\:}r@{}}
    \scriptsize C: & \scriptsize $4990\,\pm\,2$ \\
    \scriptsize M: & \scriptsize $0\,\pm\,0$ \\
    \scriptsize R: & \scriptsize $10000\,\pm\,0$
\end{tabular} & \begin{tabular}{l@{\:}r@{}}
    \scriptsize C: & \scriptsize $4971\,\pm\,4$ \\
    \scriptsize M: & \scriptsize $0\,\pm\,0$ \\
    \scriptsize R: & \scriptsize $9961\,\pm\,12$
\end{tabular} & \begin{tabular}{l@{\:}r@{}}
    \scriptsize C: & \scriptsize $1650\,\pm\,28$ \\
    \scriptsize M: & \scriptsize $0\,\pm\,0$ \\
    \scriptsize R: & \scriptsize $3319\,\pm\,56$
\end{tabular} & \begin{tabular}{l@{\:}r@{}}
    \scriptsize C: & \scriptsize $109\,\pm\,7$ \\
    \scriptsize M: & \scriptsize $0\,\pm\,0$ \\
    \scriptsize R: & \scriptsize $237\,\pm\,15$
\end{tabular} & \begin{tabular}{l@{\:}r@{}}
    \scriptsize C: & \scriptsize $5\,\pm\,2$ \\
    \scriptsize M: & \scriptsize $0\,\pm\,0$ \\
    \scriptsize R: & \scriptsize $28\,\pm\,5$
\end{tabular} & \SetCell{bg=gray!20} \begin{tabular}{l@{\:}r@{}}
    \scriptsize C: & \scriptsize $0\,\pm\,0$ \\
    \scriptsize M: & \scriptsize $0\,\pm\,0$ \\
    \scriptsize\textbf{R:} & \scriptsize\bm{$19\,\pm\,5$}
\end{tabular} & \begin{tabular}{l@{\:}r@{}}
    \scriptsize C: & \scriptsize $0\,\pm\,0$ \\
    \scriptsize M: & \scriptsize $0\,\pm\,0$ \\
    \scriptsize R: & \scriptsize $19\,\pm\,5$
\end{tabular} & \begin{tabular}{l@{\:}r@{}}
    \scriptsize C: & \scriptsize $10\,\pm\,0$ \\
    \scriptsize M: & \scriptsize $1\,\pm\,1$ \\
    \scriptsize R: & \scriptsize $28\,\pm\,8$
\end{tabular}\\
    \end{tblr}
    \end{adjustbox}
\end{table}


\subsubsection{Evaluation of \extc\ and Comparison with \vht}\label{app:synthetic_etc_ghc}  Then, we evaluate the performance of the \extc\ algorithm described in Section \ref{sec:etc_algo}. Its exploration phase length is determined by the threshold
$$\hat{m}(t)=\frac{C\sigma^2}{{\hat{\delta}^2_{\min}(t)}}(d+2\log{T}),$$
where $\sigma=0.1$ and $C=108$. As discussed in Appendix \ref{app:separation_assumption}, this rule is best suited for the Euclidean setting $\mathcal{E}=\mathcal{I}^d$. We therefore treat $C$ as a tunable parameter to find a practical balance in all settings. The results for various values of $C$ are reported in Tables \ref{table:etc_results_table_cube} and \ref{table:etc_results_table_sphere}. To allow a direct comparison with the performance of previous algorithms, we only report in ``C'' the number of expert calls after each label has been observed at least once (this corresponds to the number of expert calls in the second phase for the \vht\ and \texttt{SKM} algorithms).

Interestingly, our experiments suggest that a much more aggressive choice for $C$ is preferable in practice, even in the hypercube setting where the theoretical analysis suggested $C=108$. We also observe that as for the other algorithms, in the mixture settings, performance generally improves with dimension. Additionally, even in low dimensions, performance is poor in the uniform setting, which is expected since \extc\ is tailored to settings where the sample mean of expert-labeled queries is a reliable estimate of the corresponding seed.

\begin{table}[!htbp]
\centering
    \caption{Performance Metrics of \extc\ on $\mathcal{I}^d$ ($T=5000$).}
    \label{table:etc_results_table_cube}
    \begin{adjustbox}{max width=\textwidth,center}
    \begin{tblr}{colspec = { | Q[c,m]| Q[l,m]| Q[c,m]| Q[c,m]| Q[c,m]| Q[c,m]| Q[c,m]| Q[c,m]| Q[c,m]| Q[c,m]| },row{1} = {font=\small\bfseries,m},row{2-Z} = {font=\small,m},hlines,colsep=3pt, rowsep=3pt}
    Dim & Dist. & $C=0.25$ & $C=1$ & $C=5$ & $C=10$ & $C=25$ & $C=50$ & $C=100$ & $C=108$\\
    \SetCell[r=2]{m} 1 & Unif. & \SetCell{bg=gray!20} \begin{tabular}{l@{\:}r@{}}
    \scriptsize C: & \scriptsize $29\,\pm\,20$ \\
    \scriptsize M: & \scriptsize $549\,\pm\,38$ \\
    \scriptsize\textbf{R:} & \scriptsize\bm{$6123\,\pm\,390$}
\end{tabular} & \begin{tabular}{l@{\:}r@{}}
    \scriptsize C: & \scriptsize $98\,\pm\,12$ \\
    \scriptsize M: & \scriptsize $541\,\pm\,29$ \\
    \scriptsize R: & \scriptsize $6174\,\pm\,333$
\end{tabular} & \begin{tabular}{l@{\:}r@{}}
    \scriptsize C: & \scriptsize $445\,\pm\,50$ \\
    \scriptsize M: & \scriptsize $483\,\pm\,40$ \\
    \scriptsize R: & \scriptsize $6231\,\pm\,376$
\end{tabular} & \begin{tabular}{l@{\:}r@{}}
    \scriptsize C: & \scriptsize $851\,\pm\,138$ \\
    \scriptsize M: & \scriptsize $447\,\pm\,43$ \\
    \scriptsize R: & \scriptsize $6650\,\pm\,243$
\end{tabular} & \begin{tabular}{l@{\:}r@{}}
    \scriptsize C: & \scriptsize $2120\,\pm\,83$ \\
    \scriptsize M: & \scriptsize $310\,\pm\,21$ \\
    \scriptsize R: & \scriptsize $7676\,\pm\,113$
\end{tabular} & \begin{tabular}{l@{\:}r@{}}
    \scriptsize C: & \scriptsize $4173\,\pm\,114$ \\
    \scriptsize M: & \scriptsize $90\,\pm\,16$ \\
    \scriptsize R: & \scriptsize $9364\,\pm\,94$
\end{tabular} & \begin{tabular}{l@{\:}r@{}}
    \scriptsize C: & \scriptsize $4985\,\pm\,5$ \\
    \scriptsize M: & \scriptsize $0\,\pm\,0$ \\
    \scriptsize R: & \scriptsize $10000\,\pm\,0$
\end{tabular} & \begin{tabular}{l@{\:}r@{}}
    \scriptsize C: & \scriptsize $4985\,\pm\,5$ \\
    \scriptsize M: & \scriptsize $0\,\pm\,0$ \\
    \scriptsize R: & \scriptsize $10000\,\pm\,0$
\end{tabular}\\
     & Mix. & \SetCell{bg=gray!20} \begin{tabular}{l@{\:}r@{}}
    \scriptsize C: & \scriptsize $19\,\pm\,9$ \\
    \scriptsize M: & \scriptsize $386\,\pm\,71$ \\
    \scriptsize\textbf{R:} & \scriptsize\bm{$4305\,\pm\,770$}
\end{tabular} & \begin{tabular}{l@{\:}r@{}}
    \scriptsize C: & \scriptsize $71\,\pm\,16$ \\
    \scriptsize M: & \scriptsize $407\,\pm\,65$ \\
    \scriptsize R: & \scriptsize $4643\,\pm\,718$
\end{tabular} & \begin{tabular}{l@{\:}r@{}}
    \scriptsize C: & \scriptsize $329\,\pm\,22$ \\
    \scriptsize M: & \scriptsize $349\,\pm\,30$ \\
    \scriptsize R: & \scriptsize $4517\,\pm\,320$
\end{tabular} & \begin{tabular}{l@{\:}r@{}}
    \scriptsize C: & \scriptsize $657\,\pm\,37$ \\
    \scriptsize M: & \scriptsize $309\,\pm\,14$ \\
    \scriptsize R: & \scriptsize $4734\,\pm\,156$
\end{tabular} & \begin{tabular}{l@{\:}r@{}}
    \scriptsize C: & \scriptsize $1566\,\pm\,73$ \\
    \scriptsize M: & \scriptsize $252\,\pm\,22$ \\
    \scriptsize R: & \scriptsize $5923\,\pm\,163$
\end{tabular} & \begin{tabular}{l@{\:}r@{}}
    \scriptsize C: & \scriptsize $3119\,\pm\,146$ \\
    \scriptsize M: & \scriptsize $141\,\pm\,11$ \\
    \scriptsize R: & \scriptsize $7812\,\pm\,246$
\end{tabular} & \begin{tabular}{l@{\:}r@{}}
    \scriptsize C: & \scriptsize $4990\,\pm\,2$ \\
    \scriptsize M: & \scriptsize $0\,\pm\,0$ \\
    \scriptsize R: & \scriptsize $10000\,\pm\,0$
\end{tabular} & \begin{tabular}{l@{\:}r@{}}
    \scriptsize C: & \scriptsize $4990\,\pm\,2$ \\
    \scriptsize M: & \scriptsize $0\,\pm\,0$ \\
    \scriptsize R: & \scriptsize $10000\,\pm\,0$
\end{tabular}\\
    \SetCell[r=2]{m} 4 & Unif. & \begin{tabular}{l@{\:}r@{}}
    \scriptsize C: & \scriptsize $0\,\pm\,0$ \\
    \scriptsize M: & \scriptsize $2080\,\pm\,344$ \\
    \scriptsize R: & \scriptsize $22913\,\pm\,3771$
\end{tabular} & \begin{tabular}{l@{\:}r@{}}
    \scriptsize C: & \scriptsize $3\,\pm\,3$ \\
    \scriptsize M: & \scriptsize $1890\,\pm\,408$ \\
    \scriptsize R: & \scriptsize $20827\,\pm\,4476$
\end{tabular} & \begin{tabular}{l@{\:}r@{}}
    \scriptsize C: & \scriptsize $60\,\pm\,29$ \\
    \scriptsize M: & \scriptsize $1184\,\pm\,170$ \\
    \scriptsize R: & \scriptsize $13175\,\pm\,1814$
\end{tabular} & \begin{tabular}{l@{\:}r@{}}
    \scriptsize C: & \scriptsize $123\,\pm\,26$ \\
    \scriptsize M: & \scriptsize $1017\,\pm\,170$ \\
    \scriptsize R: & \scriptsize $11464\,\pm\,1843$
\end{tabular} & \begin{tabular}{l@{\:}r@{}}
    \scriptsize C: & \scriptsize $333\,\pm\,44$ \\
    \scriptsize M: & \scriptsize $825\,\pm\,81$ \\
    \scriptsize R: & \scriptsize $9771\,\pm\,857$
\end{tabular} & \begin{tabular}{l@{\:}r@{}}
    \scriptsize C: & \scriptsize $627\,\pm\,51$ \\
    \scriptsize M: & \scriptsize $773\,\pm\,44$ \\
    \scriptsize R: & \scriptsize $9785\,\pm\,448$
\end{tabular} & \begin{tabular}{l@{\:}r@{}}
    \scriptsize C: & \scriptsize $1307\,\pm\,101$ \\
    \scriptsize M: & \scriptsize $632\,\pm\,34$ \\
    \scriptsize R: & \scriptsize $9598\,\pm\,390$
\end{tabular} & \SetCell{bg=gray!20} \begin{tabular}{l@{\:}r@{}}
    \scriptsize C: & \scriptsize $1411\,\pm\,113$ \\
    \scriptsize M: & \scriptsize $610\,\pm\,44$ \\
    \scriptsize\textbf{R:} & \scriptsize\bm{$9563\,\pm\,486$}
\end{tabular}\\
     & Mix. & \begin{tabular}{l@{\:}r@{}}
    \scriptsize C: & \scriptsize $0\,\pm\,1$ \\
    \scriptsize M: & \scriptsize $252\,\pm\,89$ \\
    \scriptsize R: & \scriptsize $2802\,\pm\,976$
\end{tabular} & \begin{tabular}{l@{\:}r@{}}
    \scriptsize C: & \scriptsize $9\,\pm\,4$ \\
    \scriptsize M: & \scriptsize $178\,\pm\,53$ \\
    \scriptsize R: & \scriptsize $2006\,\pm\,587$
\end{tabular} & \begin{tabular}{l@{\:}r@{}}
    \scriptsize C: & \scriptsize $49\,\pm\,12$ \\
    \scriptsize M: & \scriptsize $98\,\pm\,20$ \\
    \scriptsize R: & \scriptsize $1207\,\pm\,239$
\end{tabular} & \SetCell{bg=gray!20} \begin{tabular}{l@{\:}r@{}}
    \scriptsize C: & \scriptsize $102\,\pm\,23$ \\
    \scriptsize M: & \scriptsize $76\,\pm\,11$ \\
    \scriptsize\textbf{R:} & \scriptsize\bm{$1064\,\pm\,142$}
\end{tabular} & \begin{tabular}{l@{\:}r@{}}
    \scriptsize C: & \scriptsize $238\,\pm\,30$ \\
    \scriptsize M: & \scriptsize $70\,\pm\,8$ \\
    \scriptsize R: & \scriptsize $1279\,\pm\,124$
\end{tabular} & \begin{tabular}{l@{\:}r@{}}
    \scriptsize C: & \scriptsize $491\,\pm\,52$ \\
    \scriptsize M: & \scriptsize $51\,\pm\,8$ \\
    \scriptsize R: & \scriptsize $1568\,\pm\,174$
\end{tabular} & \begin{tabular}{l@{\:}r@{}}
    \scriptsize C: & \scriptsize $944\,\pm\,58$ \\
    \scriptsize M: & \scriptsize $43\,\pm\,9$ \\
    \scriptsize R: & \scriptsize $2395\,\pm\,193$
\end{tabular} & \begin{tabular}{l@{\:}r@{}}
    \scriptsize C: & \scriptsize $1037\,\pm\,47$ \\
    \scriptsize M: & \scriptsize $43\,\pm\,8$ \\
    \scriptsize R: & \scriptsize $2571\,\pm\,150$
\end{tabular}\\
    \SetCell[r=2]{m} 10 & Unif. & \begin{tabular}{l@{\:}r@{}}
    \scriptsize C: & \scriptsize $0\,\pm\,0$ \\
    \scriptsize M: & \scriptsize $2481\,\pm\,356$ \\
    \scriptsize R: & \scriptsize $27328\,\pm\,3902$
\end{tabular} & \begin{tabular}{l@{\:}r@{}}
    \scriptsize C: & \scriptsize $0\,\pm\,0$ \\
    \scriptsize M: & \scriptsize $2481\,\pm\,356$ \\
    \scriptsize R: & \scriptsize $27328\,\pm\,3902$
\end{tabular} & \begin{tabular}{l@{\:}r@{}}
    \scriptsize C: & \scriptsize $45\,\pm\,17$ \\
    \scriptsize M: & \scriptsize $1657\,\pm\,205$ \\
    \scriptsize R: & \scriptsize $18348\,\pm\,2213$
\end{tabular} & \begin{tabular}{l@{\:}r@{}}
    \scriptsize C: & \scriptsize $68\,\pm\,20$ \\
    \scriptsize M: & \scriptsize $1452\,\pm\,205$ \\
    \scriptsize R: & \scriptsize $16143\,\pm\,2212$
\end{tabular} & \begin{tabular}{l@{\:}r@{}}
    \scriptsize C: & \scriptsize $199\,\pm\,43$ \\
    \scriptsize M: & \scriptsize $1034\,\pm\,93$ \\
    \scriptsize R: & \scriptsize $11808\,\pm\,976$
\end{tabular} & \begin{tabular}{l@{\:}r@{}}
    \scriptsize C: & \scriptsize $395\,\pm\,51$ \\
    \scriptsize M: & \scriptsize $893\,\pm\,100$ \\
    \scriptsize R: & \scriptsize $10644\,\pm\,1026$
\end{tabular} & \SetCell{bg=gray!20} \begin{tabular}{l@{\:}r@{}}
    \scriptsize C: & \scriptsize $834\,\pm\,92$ \\
    \scriptsize M: & \scriptsize $735\,\pm\,45$ \\
    \scriptsize\textbf{R:} & \scriptsize\bm{$9782\,\pm\,462$}
\end{tabular} & \begin{tabular}{l@{\:}r@{}}
    \scriptsize C: & \scriptsize $880\,\pm\,95$ \\
    \scriptsize M: & \scriptsize $731\,\pm\,46$ \\
    \scriptsize R: & \scriptsize $9831\,\pm\,510$
\end{tabular}\\
     & Mix. & \SetCell{bg=gray!20} \begin{tabular}{l@{\:}r@{}}
    \scriptsize C: & \scriptsize $0\,\pm\,0$ \\
    \scriptsize M: & \scriptsize $1\,\pm\,1$ \\
    \scriptsize\textbf{R:} & \scriptsize\bm{$33\,\pm\,7$}
\end{tabular} & \begin{tabular}{l@{\:}r@{}}
    \scriptsize C: & \scriptsize $0\,\pm\,0$ \\
    \scriptsize M: & \scriptsize $1\,\pm\,1$ \\
    \scriptsize R: & \scriptsize $33\,\pm\,7$
\end{tabular} & \begin{tabular}{l@{\:}r@{}}
    \scriptsize C: & \scriptsize $8\,\pm\,4$ \\
    \scriptsize M: & \scriptsize $0\,\pm\,0$ \\
    \scriptsize R: & \scriptsize $44\,\pm\,7$
\end{tabular} & \begin{tabular}{l@{\:}r@{}}
    \scriptsize C: & \scriptsize $25\,\pm\,5$ \\
    \scriptsize M: & \scriptsize $0\,\pm\,0$ \\
    \scriptsize R: & \scriptsize $75\,\pm\,9$
\end{tabular} & \begin{tabular}{l@{\:}r@{}}
    \scriptsize C: & \scriptsize $62\,\pm\,13$ \\
    \scriptsize M: & \scriptsize $0\,\pm\,0$ \\
    \scriptsize R: & \scriptsize $149\,\pm\,24$
\end{tabular} & \begin{tabular}{l@{\:}r@{}}
    \scriptsize C: & \scriptsize $116\,\pm\,13$ \\
    \scriptsize M: & \scriptsize $0\,\pm\,0$ \\
    \scriptsize R: & \scriptsize $257\,\pm\,27$
\end{tabular} & \begin{tabular}{l@{\:}r@{}}
    \scriptsize C: & \scriptsize $229\,\pm\,10$ \\
    \scriptsize M: & \scriptsize $0\,\pm\,0$ \\
    \scriptsize R: & \scriptsize $483\,\pm\,24$
\end{tabular} & \begin{tabular}{l@{\:}r@{}}
    \scriptsize C: & \scriptsize $248\,\pm\,7$ \\
    \scriptsize M: & \scriptsize $0\,\pm\,0$ \\
    \scriptsize R: & \scriptsize $520\,\pm\,14$
\end{tabular}\\
    \SetCell[r=2]{m} 50 & Unif. & \begin{tabular}{l@{\:}r@{}}
    \scriptsize C: & \scriptsize $0\,\pm\,0$ \\
    \scriptsize M: & \scriptsize $3020\,\pm\,115$ \\
    \scriptsize R: & \scriptsize $33260\,\pm\,1267$
\end{tabular} & \begin{tabular}{l@{\:}r@{}}
    \scriptsize C: & \scriptsize $0\,\pm\,0$ \\
    \scriptsize M: & \scriptsize $3020\,\pm\,115$ \\
    \scriptsize R: & \scriptsize $33260\,\pm\,1267$
\end{tabular} & \begin{tabular}{l@{\:}r@{}}
    \scriptsize C: & \scriptsize $35\,\pm\,18$ \\
    \scriptsize M: & \scriptsize $2579\,\pm\,205$ \\
    \scriptsize R: & \scriptsize $28483\,\pm\,2223$
\end{tabular} & \begin{tabular}{l@{\:}r@{}}
    \scriptsize C: & \scriptsize $160\,\pm\,78$ \\
    \scriptsize M: & \scriptsize $1985\,\pm\,364$ \\
    \scriptsize R: & \scriptsize $22197\,\pm\,3857$
\end{tabular} & \begin{tabular}{l@{\:}r@{}}
    \scriptsize C: & \scriptsize $466\,\pm\,82$ \\
    \scriptsize M: & \scriptsize $1277\,\pm\,91$ \\
    \scriptsize R: & \scriptsize $15027\,\pm\,887$
\end{tabular} & \begin{tabular}{l@{\:}r@{}}
    \scriptsize C: & \scriptsize $955\,\pm\,96$ \\
    \scriptsize M: & \scriptsize $909\,\pm\,52$ \\
    \scriptsize R: & \scriptsize $11957\,\pm\,541$
\end{tabular} & \begin{tabular}{l@{\:}r@{}}
    \scriptsize C: & \scriptsize $2176\,\pm\,164$ \\
    \scriptsize M: & \scriptsize $497\,\pm\,44$ \\
    \scriptsize R: & \scriptsize $9866\,\pm\,252$
\end{tabular} & \SetCell{bg=gray!20} \begin{tabular}{l@{\:}r@{}}
    \scriptsize C: & \scriptsize $2380\,\pm\,203$ \\
    \scriptsize M: & \scriptsize $436\,\pm\,38$ \\
    \scriptsize\textbf{R:} & \scriptsize\bm{$9597\,\pm\,143$}
\end{tabular}\\
     & Mix. & \SetCell{bg=gray!20} \begin{tabular}{l@{\:}r@{}}
    \scriptsize C: & \scriptsize $0\,\pm\,0$ \\
    \scriptsize M: & \scriptsize $0\,\pm\,0$ \\
    \scriptsize\textbf{R:} & \scriptsize\bm{$24\,\pm\,9$}
\end{tabular} & \begin{tabular}{l@{\:}r@{}}
    \scriptsize C: & \scriptsize $0\,\pm\,0$ \\
    \scriptsize M: & \scriptsize $0\,\pm\,0$ \\
    \scriptsize R: & \scriptsize $24\,\pm\,9$
\end{tabular} & \begin{tabular}{l@{\:}r@{}}
    \scriptsize C: & \scriptsize $0\,\pm\,0$ \\
    \scriptsize M: & \scriptsize $0\,\pm\,0$ \\
    \scriptsize R: & \scriptsize $24\,\pm\,9$
\end{tabular} & \begin{tabular}{l@{\:}r@{}}
    \scriptsize C: & \scriptsize $7\,\pm\,4$ \\
    \scriptsize M: & \scriptsize $0\,\pm\,0$ \\
    \scriptsize R: & \scriptsize $39\,\pm\,15$
\end{tabular} & \begin{tabular}{l@{\:}r@{}}
    \scriptsize C: & \scriptsize $24\,\pm\,8$ \\
    \scriptsize M: & \scriptsize $0\,\pm\,0$ \\
    \scriptsize R: & \scriptsize $72\,\pm\,15$
\end{tabular} & \begin{tabular}{l@{\:}r@{}}
    \scriptsize C: & \scriptsize $43\,\pm\,12$ \\
    \scriptsize M: & \scriptsize $0\,\pm\,0$ \\
    \scriptsize R: & \scriptsize $110\,\pm\,20$
\end{tabular} & \begin{tabular}{l@{\:}r@{}}
    \scriptsize C: & \scriptsize $81\,\pm\,8$ \\
    \scriptsize M: & \scriptsize $0\,\pm\,0$ \\
    \scriptsize R: & \scriptsize $186\,\pm\,17$
\end{tabular} & \begin{tabular}{l@{\:}r@{}}
    \scriptsize C: & \scriptsize $90\,\pm\,12$ \\
    \scriptsize M: & \scriptsize $0\,\pm\,0$ \\
    \scriptsize R: & \scriptsize $205\,\pm\,21$
\end{tabular}\\
    \end{tblr}
    \end{adjustbox}
\end{table}

\begin{table}[!htbp]
\centering
    \caption{Performance Metrics of \extc\ on $\Sph^{d-1}$ ($T=5000$).}
    \label{table:etc_results_table_sphere}
    \begin{adjustbox}{max width=\textwidth,center}
    \begin{tblr}{colspec = { | Q[c,m]| Q[l,m]| Q[c,m]| Q[c,m]| Q[c,m]| Q[c,m]| Q[c,m]| Q[c,m]| Q[c,m]| Q[c,m]| },row{1} = {font=\small\bfseries,m},row{2-Z} = {font=\small,m},hlines,colsep=3pt, rowsep=3pt}
    Dim & Dist. & $C=0.25$ & $C=1$ & $C=5$ & $C=10$ & $C=25$ & $C=50$ & $C=100$ & $C=108$\\
    \SetCell[r=2]{m} 2 & Unif. & \begin{tabular}{l@{\:}r@{}}
    \scriptsize C: & \scriptsize $0\,\pm\,0$ \\
    \scriptsize M: & \scriptsize $983\,\pm\,117$ \\
    \scriptsize R: & \scriptsize $10877\,\pm\,1270$
\end{tabular} & \begin{tabular}{l@{\:}r@{}}
    \scriptsize C: & \scriptsize $25\,\pm\,17$ \\
    \scriptsize M: & \scriptsize $947\,\pm\,68$ \\
    \scriptsize R: & \scriptsize $10525\,\pm\,778$
\end{tabular} & \begin{tabular}{l@{\:}r@{}}
    \scriptsize C: & \scriptsize $131\,\pm\,61$ \\
    \scriptsize M: & \scriptsize $905\,\pm\,88$ \\
    \scriptsize R: & \scriptsize $10274\,\pm\,920$
\end{tabular} & \begin{tabular}{l@{\:}r@{}}
    \scriptsize C: & \scriptsize $299\,\pm\,63$ \\
    \scriptsize M: & \scriptsize $829\,\pm\,69$ \\
    \scriptsize R: & \scriptsize $9775\,\pm\,714$
\end{tabular} & \SetCell{bg=gray!20} \begin{tabular}{l@{\:}r@{}}
    \scriptsize C: & \scriptsize $807\,\pm\,85$ \\
    \scriptsize M: & \scriptsize $697\,\pm\,54$ \\
    \scriptsize\textbf{R:} & \scriptsize\bm{$9347\,\pm\,472$}
\end{tabular} & \begin{tabular}{l@{\:}r@{}}
    \scriptsize C: & \scriptsize $1574\,\pm\,104$ \\
    \scriptsize M: & \scriptsize $579\,\pm\,39$ \\
    \scriptsize R: & \scriptsize $9573\,\pm\,354$
\end{tabular} & \begin{tabular}{l@{\:}r@{}}
    \scriptsize C: & \scriptsize $2895\,\pm\,213$ \\
    \scriptsize M: & \scriptsize $355\,\pm\,27$ \\
    \scriptsize R: & \scriptsize $9754\,\pm\,319$
\end{tabular} & \begin{tabular}{l@{\:}r@{}}
    \scriptsize C: & \scriptsize $3135\,\pm\,211$ \\
    \scriptsize M: & \scriptsize $316\,\pm\,26$ \\
    \scriptsize R: & \scriptsize $9810\,\pm\,331$
\end{tabular}\\
     & Mix. & \begin{tabular}{l@{\:}r@{}}
    \scriptsize C: & \scriptsize $10\,\pm\,5$ \\
    \scriptsize M: & \scriptsize $256\,\pm\,122$ \\
    \scriptsize R: & \scriptsize $2853\,\pm\,1341$
\end{tabular} & \begin{tabular}{l@{\:}r@{}}
    \scriptsize C: & \scriptsize $46\,\pm\,24$ \\
    \scriptsize M: & \scriptsize $171\,\pm\,169$ \\
    \scriptsize R: & \scriptsize $1993\,\pm\,1825$
\end{tabular} & \SetCell{bg=gray!20} \begin{tabular}{l@{\:}r@{}}
    \scriptsize C: & \scriptsize $223\,\pm\,21$ \\
    \scriptsize M: & \scriptsize $45\,\pm\,15$ \\
    \scriptsize\textbf{R:} & \scriptsize\bm{$963\,\pm\,173$}
\end{tabular} & \begin{tabular}{l@{\:}r@{}}
    \scriptsize C: & \scriptsize $422\,\pm\,67$ \\
    \scriptsize M: & \scriptsize $33\,\pm\,8$ \\
    \scriptsize R: & \scriptsize $1224\,\pm\,176$
\end{tabular} & \begin{tabular}{l@{\:}r@{}}
    \scriptsize C: & \scriptsize $1020\,\pm\,94$ \\
    \scriptsize M: & \scriptsize $26\,\pm\,3$ \\
    \scriptsize R: & \scriptsize $2347\,\pm\,209$
\end{tabular} & \begin{tabular}{l@{\:}r@{}}
    \scriptsize C: & \scriptsize $1976\,\pm\,75$ \\
    \scriptsize M: & \scriptsize $19\,\pm\,4$ \\
    \scriptsize R: & \scriptsize $4177\,\pm\,115$
\end{tabular} & \begin{tabular}{l@{\:}r@{}}
    \scriptsize C: & \scriptsize $3962\,\pm\,125$ \\
    \scriptsize M: & \scriptsize $6\,\pm\,2$ \\
    \scriptsize R: & \scriptsize $8011\,\pm\,230$
\end{tabular} & \begin{tabular}{l@{\:}r@{}}
    \scriptsize C: & \scriptsize $4304\,\pm\,103$ \\
    \scriptsize M: & \scriptsize $4\,\pm\,1$ \\
    \scriptsize R: & \scriptsize $8666\,\pm\,207$
\end{tabular}\\
    \SetCell[r=2]{m} 4 & Unif. & \begin{tabular}{l@{\:}r@{}}
    \scriptsize C: & \scriptsize $0\,\pm\,0$ \\
    \scriptsize M: & \scriptsize $1714\,\pm\,437$ \\
    \scriptsize R: & \scriptsize $18873\,\pm\,4810$
\end{tabular} & \begin{tabular}{l@{\:}r@{}}
    \scriptsize C: & \scriptsize $0\,\pm\,0$ \\
    \scriptsize M: & \scriptsize $1714\,\pm\,437$ \\
    \scriptsize R: & \scriptsize $18873\,\pm\,4810$
\end{tabular} & \begin{tabular}{l@{\:}r@{}}
    \scriptsize C: & \scriptsize $8\,\pm\,4$ \\
    \scriptsize M: & \scriptsize $1359\,\pm\,308$ \\
    \scriptsize R: & \scriptsize $14987\,\pm\,3391$
\end{tabular} & \begin{tabular}{l@{\:}r@{}}
    \scriptsize C: & \scriptsize $25\,\pm\,13$ \\
    \scriptsize M: & \scriptsize $959\,\pm\,202$ \\
    \scriptsize R: & \scriptsize $10622\,\pm\,2195$
\end{tabular} & \begin{tabular}{l@{\:}r@{}}
    \scriptsize C: & \scriptsize $42\,\pm\,13$ \\
    \scriptsize M: & \scriptsize $823\,\pm\,123$ \\
    \scriptsize R: & \scriptsize $9160\,\pm\,1328$
\end{tabular} & \begin{tabular}{l@{\:}r@{}}
    \scriptsize C: & \scriptsize $79\,\pm\,15$ \\
    \scriptsize M: & \scriptsize $635\,\pm\,83$ \\
    \scriptsize R: & \scriptsize $7165\,\pm\,904$
\end{tabular} & \begin{tabular}{l@{\:}r@{}}
    \scriptsize C: & \scriptsize $152\,\pm\,21$ \\
    \scriptsize M: & \scriptsize $529\,\pm\,98$ \\
    \scriptsize R: & \scriptsize $6148\,\pm\,1053$
\end{tabular} & \SetCell{bg=gray!20} \begin{tabular}{l@{\:}r@{}}
    \scriptsize C: & \scriptsize $159\,\pm\,18$ \\
    \scriptsize M: & \scriptsize $524\,\pm\,83$ \\
    \scriptsize\textbf{R:} & \scriptsize\bm{$6099\,\pm\,889$}
\end{tabular}\\
     & Mix. & \SetCell{bg=gray!20} \begin{tabular}{l@{\:}r@{}}
    \scriptsize C: & \scriptsize $0\,\pm\,0$ \\
    \scriptsize M: & \scriptsize $0\,\pm\,0$ \\
    \scriptsize\textbf{R:} & \scriptsize\bm{$22\,\pm\,7$}
\end{tabular} & \begin{tabular}{l@{\:}r@{}}
    \scriptsize C: & \scriptsize $0\,\pm\,0$ \\
    \scriptsize M: & \scriptsize $0\,\pm\,0$ \\
    \scriptsize R: & \scriptsize $22\,\pm\,7$
\end{tabular} & \begin{tabular}{l@{\:}r@{}}
    \scriptsize C: & \scriptsize $4\,\pm\,4$ \\
    \scriptsize M: & \scriptsize $0\,\pm\,0$ \\
    \scriptsize R: & \scriptsize $30\,\pm\,9$
\end{tabular} & \begin{tabular}{l@{\:}r@{}}
    \scriptsize C: & \scriptsize $12\,\pm\,6$ \\
    \scriptsize M: & \scriptsize $0\,\pm\,0$ \\
    \scriptsize R: & \scriptsize $46\,\pm\,12$
\end{tabular} & \begin{tabular}{l@{\:}r@{}}
    \scriptsize C: & \scriptsize $34\,\pm\,10$ \\
    \scriptsize M: & \scriptsize $0\,\pm\,0$ \\
    \scriptsize R: & \scriptsize $89\,\pm\,15$
\end{tabular} & \begin{tabular}{l@{\:}r@{}}
    \scriptsize C: & \scriptsize $75\,\pm\,17$ \\
    \scriptsize M: & \scriptsize $0\,\pm\,0$ \\
    \scriptsize R: & \scriptsize $172\,\pm\,34$
\end{tabular} & \begin{tabular}{l@{\:}r@{}}
    \scriptsize C: & \scriptsize $140\,\pm\,16$ \\
    \scriptsize M: & \scriptsize $0\,\pm\,0$ \\
    \scriptsize R: & \scriptsize $301\,\pm\,30$
\end{tabular} & \begin{tabular}{l@{\:}r@{}}
    \scriptsize C: & \scriptsize $154\,\pm\,14$ \\
    \scriptsize M: & \scriptsize $0\,\pm\,0$ \\
    \scriptsize R: & \scriptsize $329\,\pm\,25$
\end{tabular}\\
    \SetCell[r=2]{m} 10 & Unif. & \begin{tabular}{l@{\:}r@{}}
    \scriptsize C: & \scriptsize $0\,\pm\,0$ \\
    \scriptsize M: & \scriptsize $2920\,\pm\,269$ \\
    \scriptsize R: & \scriptsize $32141\,\pm\,2957$
\end{tabular} & \begin{tabular}{l@{\:}r@{}}
    \scriptsize C: & \scriptsize $0\,\pm\,0$ \\
    \scriptsize M: & \scriptsize $2920\,\pm\,269$ \\
    \scriptsize R: & \scriptsize $32141\,\pm\,2957$
\end{tabular} & \begin{tabular}{l@{\:}r@{}}
    \scriptsize C: & \scriptsize $11\,\pm\,7$ \\
    \scriptsize M: & \scriptsize $2271\,\pm\,434$ \\
    \scriptsize R: & \scriptsize $25020\,\pm\,4758$
\end{tabular} & \begin{tabular}{l@{\:}r@{}}
    \scriptsize C: & \scriptsize $52\,\pm\,15$ \\
    \scriptsize M: & \scriptsize $1563\,\pm\,238$ \\
    \scriptsize R: & \scriptsize $17316\,\pm\,2602$
\end{tabular} & \begin{tabular}{l@{\:}r@{}}
    \scriptsize C: & \scriptsize $115\,\pm\,26$ \\
    \scriptsize M: & \scriptsize $1096\,\pm\,37$ \\
    \scriptsize R: & \scriptsize $12300\,\pm\,379$
\end{tabular} & \begin{tabular}{l@{\:}r@{}}
    \scriptsize C: & \scriptsize $229\,\pm\,34$ \\
    \scriptsize M: & \scriptsize $796\,\pm\,121$ \\
    \scriptsize R: & \scriptsize $9235\,\pm\,1272$
\end{tabular} & \begin{tabular}{l@{\:}r@{}}
    \scriptsize C: & \scriptsize $421\,\pm\,25$ \\
    \scriptsize M: & \scriptsize $654\,\pm\,52$ \\
    \scriptsize R: & \scriptsize $8060\,\pm\,536$
\end{tabular} & \SetCell{bg=gray!20} \begin{tabular}{l@{\:}r@{}}
    \scriptsize C: & \scriptsize $460\,\pm\,29$ \\
    \scriptsize M: & \scriptsize $611\,\pm\,54$ \\
    \scriptsize\textbf{R:} & \scriptsize\bm{$7657\,\pm\,544$}
\end{tabular}\\
     & Mix. & \SetCell{bg=gray!20} \begin{tabular}{l@{\:}r@{}}
    \scriptsize C: & \scriptsize $0\,\pm\,0$ \\
    \scriptsize M: & \scriptsize $0\,\pm\,0$ \\
    \scriptsize\textbf{R:} & \scriptsize\bm{$22\,\pm\,9$}
\end{tabular} & \begin{tabular}{l@{\:}r@{}}
    \scriptsize C: & \scriptsize $0\,\pm\,0$ \\
    \scriptsize M: & \scriptsize $0\,\pm\,0$ \\
    \scriptsize R: & \scriptsize $22\,\pm\,9$
\end{tabular} & \begin{tabular}{l@{\:}r@{}}
    \scriptsize C: & \scriptsize $1\,\pm\,2$ \\
    \scriptsize M: & \scriptsize $0\,\pm\,0$ \\
    \scriptsize R: & \scriptsize $24\,\pm\,12$
\end{tabular} & \begin{tabular}{l@{\:}r@{}}
    \scriptsize C: & \scriptsize $12\,\pm\,3$ \\
    \scriptsize M: & \scriptsize $0\,\pm\,0$ \\
    \scriptsize R: & \scriptsize $46\,\pm\,9$
\end{tabular} & \begin{tabular}{l@{\:}r@{}}
    \scriptsize C: & \scriptsize $29\,\pm\,9$ \\
    \scriptsize M: & \scriptsize $0\,\pm\,0$ \\
    \scriptsize R: & \scriptsize $79\,\pm\,12$
\end{tabular} & \begin{tabular}{l@{\:}r@{}}
    \scriptsize C: & \scriptsize $55\,\pm\,4$ \\
    \scriptsize M: & \scriptsize $0\,\pm\,0$ \\
    \scriptsize R: & \scriptsize $132\,\pm\,7$
\end{tabular} & \begin{tabular}{l@{\:}r@{}}
    \scriptsize C: & \scriptsize $114\,\pm\,6$ \\
    \scriptsize M: & \scriptsize $0\,\pm\,0$ \\
    \scriptsize R: & \scriptsize $249\,\pm\,14$
\end{tabular} & \begin{tabular}{l@{\:}r@{}}
    \scriptsize C: & \scriptsize $123\,\pm\,9$ \\
    \scriptsize M: & \scriptsize $0\,\pm\,0$ \\
    \scriptsize R: & \scriptsize $268\,\pm\,14$
\end{tabular}\\
    \SetCell[r=2]{m} 50 & Unif. & \begin{tabular}{l@{\:}r@{}}
    \scriptsize C: & \scriptsize $0\,\pm\,0$ \\
    \scriptsize M: & \scriptsize $3672\,\pm\,197$ \\
    \scriptsize R: & \scriptsize $40407\,\pm\,2165$
\end{tabular} & \begin{tabular}{l@{\:}r@{}}
    \scriptsize C: & \scriptsize $3\,\pm\,5$ \\
    \scriptsize M: & \scriptsize $3604\,\pm\,287$ \\
    \scriptsize R: & \scriptsize $39669\,\pm\,3151$
\end{tabular} & \begin{tabular}{l@{\:}r@{}}
    \scriptsize C: & \scriptsize $190\,\pm\,20$ \\
    \scriptsize M: & \scriptsize $1882\,\pm\,125$ \\
    \scriptsize R: & \scriptsize $21104\,\pm\,1337$
\end{tabular} & \begin{tabular}{l@{\:}r@{}}
    \scriptsize C: & \scriptsize $422\,\pm\,13$ \\
    \scriptsize M: & \scriptsize $1338\,\pm\,52$ \\
    \scriptsize R: & \scriptsize $15576\,\pm\,552$
\end{tabular} & \begin{tabular}{l@{\:}r@{}}
    \scriptsize C: & \scriptsize $1099\,\pm\,42$ \\
    \scriptsize M: & \scriptsize $760\,\pm\,31$ \\
    \scriptsize R: & \scriptsize $10582\,\pm\,284$
\end{tabular} & \SetCell{bg=gray!20} \begin{tabular}{l@{\:}r@{}}
    \scriptsize C: & \scriptsize $2200\,\pm\,82$ \\
    \scriptsize M: & \scriptsize $412\,\pm\,38$ \\
    \scriptsize\textbf{R:} & \scriptsize\bm{$8947\,\pm\,296$}
\end{tabular} & \begin{tabular}{l@{\:}r@{}}
    \scriptsize C: & \scriptsize $4632\,\pm\,53$ \\
    \scriptsize M: & \scriptsize $39\,\pm\,7$ \\
    \scriptsize R: & \scriptsize $9715\,\pm\,71$
\end{tabular} & \begin{tabular}{l@{\:}r@{}}
    \scriptsize C: & \scriptsize $4947\,\pm\,62$ \\
    \scriptsize M: & \scriptsize $2\,\pm\,3$ \\
    \scriptsize R: & \scriptsize $9938\,\pm\,89$
\end{tabular}\\
     & Mix. & \SetCell{bg=gray!20} \begin{tabular}{l@{\:}r@{}}
    \scriptsize C: & \scriptsize $0\,\pm\,0$ \\
    \scriptsize M: & \scriptsize $0\,\pm\,0$ \\
    \scriptsize\textbf{R:} & \scriptsize\bm{$26\,\pm\,12$}
\end{tabular} & \begin{tabular}{l@{\:}r@{}}
    \scriptsize C: & \scriptsize $0\,\pm\,0$ \\
    \scriptsize M: & \scriptsize $0\,\pm\,0$ \\
    \scriptsize R: & \scriptsize $26\,\pm\,12$
\end{tabular} & \begin{tabular}{l@{\:}r@{}}
    \scriptsize C: & \scriptsize $14\,\pm\,5$ \\
    \scriptsize M: & \scriptsize $0\,\pm\,0$ \\
    \scriptsize R: & \scriptsize $54\,\pm\,2$
\end{tabular} & \begin{tabular}{l@{\:}r@{}}
    \scriptsize C: & \scriptsize $32\,\pm\,8$ \\
    \scriptsize M: & \scriptsize $0\,\pm\,0$ \\
    \scriptsize R: & \scriptsize $89\,\pm\,7$
\end{tabular} & \begin{tabular}{l@{\:}r@{}}
    \scriptsize C: & \scriptsize $82\,\pm\,18$ \\
    \scriptsize M: & \scriptsize $0\,\pm\,0$ \\
    \scriptsize R: & \scriptsize $189\,\pm\,32$
\end{tabular} & \begin{tabular}{l@{\:}r@{}}
    \scriptsize C: & \scriptsize $169\,\pm\,27$ \\
    \scriptsize M: & \scriptsize $0\,\pm\,0$ \\
    \scriptsize R: & \scriptsize $364\,\pm\,47$
\end{tabular} & \begin{tabular}{l@{\:}r@{}}
    \scriptsize C: & \scriptsize $334\,\pm\,19$ \\
    \scriptsize M: & \scriptsize $0\,\pm\,0$ \\
    \scriptsize R: & \scriptsize $694\,\pm\,42$
\end{tabular} & \begin{tabular}{l@{\:}r@{}}
    \scriptsize C: & \scriptsize $359\,\pm\,24$ \\
    \scriptsize M: & \scriptsize $0\,\pm\,0$ \\
    \scriptsize R: & \scriptsize $743\,\pm\,49$
\end{tabular}\\
    \end{tblr}
    \end{adjustbox}
\end{table}
Finally, we provide in Tables \ref{tab:nearestquery_results_table_cube} and \ref{tab:nearestquery_results_table_sphere} detailed results for the nearest-query distance modification of \vht\ introduced in Appendix \ref{sec:additional_real_world_experiments}, where the distance to the closest point in each class is used instead of the distance to convex hulls. As expected, this version performs slightly worse than \vht, but it still outperforms \extc\ in almost all settings, while being just as computationally efficient. Overall, these numerical results suggest using \vht\ when additional computation is affordable, and otherwise using nearest-query \vht. We summarize the results of all synthetic experiments in Table \ref{tab:synth_summary}.
\begin{table}[!htbp]
\centering
    \caption{Performance Metrics on $[0,1]^d$ for \vht\ with nearest-query distance ($T=5000$).}
    \label{tab:nearestquery_results_table_cube}
    \begin{adjustbox}{max width=\textwidth,center}
    \begin{tblr}{colspec = { | Q[c,m]| Q[l,m]| Q[c,m]| Q[c,m]| Q[c,m]| Q[c,m]| Q[c,m]| Q[c,m]| Q[c,m]| },row{1} = {font=\small\bfseries,m},row{2-Z} = {font=\small,m},hlines,colsep=3pt, rowsep=3pt}
    Dim & Dist. & $\tau=0.00$ & $\tau=0.20$ & $\tau=0.40$ & $\tau=0.60$ & $\tau=0.80$ & $\tau=0.95$ & $\tau=1.00$\\
    \SetCell[r=2]{m} 1 & Unif. & \begin{tabular}{l@{\:}r@{}}
    \scriptsize C: & \scriptsize $4983\,\pm\,5$ \\
    \scriptsize M: & \scriptsize $0\,\pm\,0$ \\
    \scriptsize R: & \scriptsize $9996\,\pm\,3$
\end{tabular} & \begin{tabular}{l@{\:}r@{}}
    \scriptsize C: & \scriptsize $124\,\pm\,3$ \\
    \scriptsize M: & \scriptsize $2\,\pm\,2$ \\
    \scriptsize R: & \scriptsize $302\,\pm\,19$
\end{tabular} & \SetCell{bg=gray!20} \begin{tabular}{l@{\:}r@{}}
    \scriptsize C: & \scriptsize $59\,\pm\,3$ \\
    \scriptsize M: & \scriptsize $3\,\pm\,3$ \\
    \scriptsize\textbf{R:} & \scriptsize\bm{$186\,\pm\,33$}
\end{tabular} & \begin{tabular}{l@{\:}r@{}}
    \scriptsize C: & \scriptsize $35\,\pm\,3$ \\
    \scriptsize M: & \scriptsize $12\,\pm\,3$ \\
    \scriptsize R: & \scriptsize $237\,\pm\,33$
\end{tabular} & \begin{tabular}{l@{\:}r@{}}
    \scriptsize C: & \scriptsize $21\,\pm\,4$ \\
    \scriptsize M: & \scriptsize $34\,\pm\,10$ \\
    \scriptsize R: & \scriptsize $445\,\pm\,110$
\end{tabular} & \begin{tabular}{l@{\:}r@{}}
    \scriptsize C: & \scriptsize $12\,\pm\,2$ \\
    \scriptsize M: & \scriptsize $101\,\pm\,33$ \\
    \scriptsize R: & \scriptsize $1168\,\pm\,363$
\end{tabular} & \begin{tabular}{l@{\:}r@{}}
    \scriptsize C: & \scriptsize $0\,\pm\,0$ \\
    \scriptsize M: & \scriptsize $375\,\pm\,126$ \\
    \scriptsize R: & \scriptsize $4160\,\pm\,1384$
\end{tabular}\\
     & Mix. & \begin{tabular}{l@{\:}r@{}}
    \scriptsize C: & \scriptsize $4987\,\pm\,5$ \\
    \scriptsize M: & \scriptsize $0\,\pm\,0$ \\
    \scriptsize R: & \scriptsize $9994\,\pm\,5$
\end{tabular} & \begin{tabular}{l@{\:}r@{}}
    \scriptsize C: & \scriptsize $136\,\pm\,9$ \\
    \scriptsize M: & \scriptsize $1\,\pm\,1$ \\
    \scriptsize R: & \scriptsize $307\,\pm\,21$
\end{tabular} & \SetCell{bg=gray!20} \begin{tabular}{l@{\:}r@{}}
    \scriptsize C: & \scriptsize $65\,\pm\,5$ \\
    \scriptsize M: & \scriptsize $9\,\pm\,5$ \\
    \scriptsize\textbf{R:} & \scriptsize\bm{$245\,\pm\,50$}
\end{tabular} & \begin{tabular}{l@{\:}r@{}}
    \scriptsize C: & \scriptsize $40\,\pm\,1$ \\
    \scriptsize M: & \scriptsize $21\,\pm\,7$ \\
    \scriptsize R: & \scriptsize $333\,\pm\,78$
\end{tabular} & \begin{tabular}{l@{\:}r@{}}
    \scriptsize C: & \scriptsize $27\,\pm\,1$ \\
    \scriptsize M: & \scriptsize $42\,\pm\,14$ \\
    \scriptsize R: & \scriptsize $532\,\pm\,151$
\end{tabular} & \begin{tabular}{l@{\:}r@{}}
    \scriptsize C: & \scriptsize $15\,\pm\,2$ \\
    \scriptsize M: & \scriptsize $155\,\pm\,31$ \\
    \scriptsize R: & \scriptsize $1753\,\pm\,337$
\end{tabular} & \begin{tabular}{l@{\:}r@{}}
    \scriptsize C: & \scriptsize $0\,\pm\,0$ \\
    \scriptsize M: & \scriptsize $823\,\pm\,389$ \\
    \scriptsize R: & \scriptsize $9071\,\pm\,4275$
\end{tabular}\\
    \SetCell[r=2]{m} 4 & Unif. & \begin{tabular}{l@{\:}r@{}}
    \scriptsize C: & \scriptsize $4986\,\pm\,4$ \\
    \scriptsize M: & \scriptsize $0\,\pm\,0$ \\
    \scriptsize R: & \scriptsize $10000\,\pm\,0$
\end{tabular} & \begin{tabular}{l@{\:}r@{}}
    \scriptsize C: & \scriptsize $4723\,\pm\,8$ \\
    \scriptsize M: & \scriptsize $0\,\pm\,0$ \\
    \scriptsize R: & \scriptsize $9477\,\pm\,17$
\end{tabular} & \begin{tabular}{l@{\:}r@{}}
    \scriptsize C: & \scriptsize $3530\,\pm\,16$ \\
    \scriptsize M: & \scriptsize $12\,\pm\,3$ \\
    \scriptsize R: & \scriptsize $7224\,\pm\,65$
\end{tabular} & \SetCell{bg=gray!20} \begin{tabular}{l@{\:}r@{}}
    \scriptsize C: & \scriptsize $2190\,\pm\,24$ \\
    \scriptsize M: & \scriptsize $90\,\pm\,6$ \\
    \scriptsize\textbf{R:} & \scriptsize\bm{$5396\,\pm\,95$}
\end{tabular} & \begin{tabular}{l@{\:}r@{}}
    \scriptsize C: & \scriptsize $1068\,\pm\,29$ \\
    \scriptsize M: & \scriptsize $348\,\pm\,16$ \\
    \scriptsize R: & \scriptsize $5988\,\pm\,188$
\end{tabular} & \begin{tabular}{l@{\:}r@{}}
    \scriptsize C: & \scriptsize $319\,\pm\,16$ \\
    \scriptsize M: & \scriptsize $894\,\pm\,30$ \\
    \scriptsize R: & \scriptsize $10496\,\pm\,305$
\end{tabular} & \begin{tabular}{l@{\:}r@{}}
    \scriptsize C: & \scriptsize $0\,\pm\,0$ \\
    \scriptsize M: & \scriptsize $2155\,\pm\,181$ \\
    \scriptsize R: & \scriptsize $23729\,\pm\,1988$
\end{tabular}\\
     & Mix. & \begin{tabular}{l@{\:}r@{}}
    \scriptsize C: & \scriptsize $4986\,\pm\,7$ \\
    \scriptsize M: & \scriptsize $0\,\pm\,0$ \\
    \scriptsize R: & \scriptsize $10000\,\pm\,0$
\end{tabular} & \begin{tabular}{l@{\:}r@{}}
    \scriptsize C: & \scriptsize $3450\,\pm\,45$ \\
    \scriptsize M: & \scriptsize $0\,\pm\,0$ \\
    \scriptsize R: & \scriptsize $6928\,\pm\,81$
\end{tabular} & \begin{tabular}{l@{\:}r@{}}
    \scriptsize C: & \scriptsize $1575\,\pm\,31$ \\
    \scriptsize M: & \scriptsize $2\,\pm\,1$ \\
    \scriptsize R: & \scriptsize $3196\,\pm\,49$
\end{tabular} & \begin{tabular}{l@{\:}r@{}}
    \scriptsize C: & \scriptsize $760\,\pm\,26$ \\
    \scriptsize M: & \scriptsize $13\,\pm\,4$ \\
    \scriptsize R: & \scriptsize $1690\,\pm\,79$
\end{tabular} & \SetCell{bg=gray!20} \begin{tabular}{l@{\:}r@{}}
    \scriptsize C: & \scriptsize $343\,\pm\,22$ \\
    \scriptsize M: & \scriptsize $59\,\pm\,7$ \\
    \scriptsize\textbf{R:} & \scriptsize\bm{$1366\,\pm\,99$}
\end{tabular} & \begin{tabular}{l@{\:}r@{}}
    \scriptsize C: & \scriptsize $98\,\pm\,9$ \\
    \scriptsize M: & \scriptsize $187\,\pm\,24$ \\
    \scriptsize R: & \scriptsize $2283\,\pm\,267$
\end{tabular} & \begin{tabular}{l@{\:}r@{}}
    \scriptsize C: & \scriptsize $0\,\pm\,0$ \\
    \scriptsize M: & \scriptsize $422\,\pm\,199$ \\
    \scriptsize R: & \scriptsize $4666\,\pm\,2188$
\end{tabular}\\
    \SetCell[r=2]{m} 10 & Unif. & \begin{tabular}{l@{\:}r@{}}
    \scriptsize C: & \scriptsize $4983\,\pm\,10$ \\
    \scriptsize M: & \scriptsize $0\,\pm\,0$ \\
    \scriptsize R: & \scriptsize $10000\,\pm\,0$
\end{tabular} & \begin{tabular}{l@{\:}r@{}}
    \scriptsize C: & \scriptsize $4983\,\pm\,10$ \\
    \scriptsize M: & \scriptsize $0\,\pm\,0$ \\
    \scriptsize R: & \scriptsize $10000\,\pm\,0$
\end{tabular} & \begin{tabular}{l@{\:}r@{}}
    \scriptsize C: & \scriptsize $4972\,\pm\,11$ \\
    \scriptsize M: & \scriptsize $0\,\pm\,0$ \\
    \scriptsize R: & \scriptsize $9978\,\pm\,2$
\end{tabular} & \begin{tabular}{l@{\:}r@{}}
    \scriptsize C: & \scriptsize $4704\,\pm\,24$ \\
    \scriptsize M: & \scriptsize $18\,\pm\,4$ \\
    \scriptsize R: & \scriptsize $9641\,\pm\,38$
\end{tabular} & \SetCell{bg=gray!20} \begin{tabular}{l@{\:}r@{}}
    \scriptsize C: & \scriptsize $3235\,\pm\,22$ \\
    \scriptsize M: & \scriptsize $276\,\pm\,7$ \\
    \scriptsize\textbf{R:} & \scriptsize\bm{$9544\,\pm\,91$}
\end{tabular} & \begin{tabular}{l@{\:}r@{}}
    \scriptsize C: & \scriptsize $1027\,\pm\,29$ \\
    \scriptsize M: & \scriptsize $1242\,\pm\,24$ \\
    \scriptsize R: & \scriptsize $15751\,\pm\,316$
\end{tabular} & \begin{tabular}{l@{\:}r@{}}
    \scriptsize C: & \scriptsize $0\,\pm\,0$ \\
    \scriptsize M: & \scriptsize $2731\,\pm\,236$ \\
    \scriptsize R: & \scriptsize $30075\,\pm\,2577$
\end{tabular}\\
     & Mix. & \begin{tabular}{l@{\:}r@{}}
    \scriptsize C: & \scriptsize $4988\,\pm\,3$ \\
    \scriptsize M: & \scriptsize $0\,\pm\,0$ \\
    \scriptsize R: & \scriptsize $10000\,\pm\,0$
\end{tabular} & \begin{tabular}{l@{\:}r@{}}
    \scriptsize C: & \scriptsize $4635\,\pm\,21$ \\
    \scriptsize M: & \scriptsize $0\,\pm\,0$ \\
    \scriptsize R: & \scriptsize $9294\,\pm\,45$
\end{tabular} & \begin{tabular}{l@{\:}r@{}}
    \scriptsize C: & \scriptsize $1081\,\pm\,24$ \\
    \scriptsize M: & \scriptsize $0\,\pm\,0$ \\
    \scriptsize R: & \scriptsize $2187\,\pm\,54$
\end{tabular} & \begin{tabular}{l@{\:}r@{}}
    \scriptsize C: & \scriptsize $174\,\pm\,5$ \\
    \scriptsize M: & \scriptsize $0\,\pm\,0$ \\
    \scriptsize R: & \scriptsize $372\,\pm\,13$
\end{tabular} & \begin{tabular}{l@{\:}r@{}}
    \scriptsize C: & \scriptsize $26\,\pm\,3$ \\
    \scriptsize M: & \scriptsize $0\,\pm\,0$ \\
    \scriptsize R: & \scriptsize $78\,\pm\,15$
\end{tabular} & \SetCell{bg=gray!20} \begin{tabular}{l@{\:}r@{}}
    \scriptsize C: & \scriptsize $1\,\pm\,1$ \\
    \scriptsize M: & \scriptsize $0\,\pm\,0$ \\
    \scriptsize\textbf{R:} & \scriptsize\bm{$29\,\pm\,10$}
\end{tabular} & \begin{tabular}{l@{\:}r@{}}
    \scriptsize C: & \scriptsize $0\,\pm\,0$ \\
    \scriptsize M: & \scriptsize $1\,\pm\,1$ \\
    \scriptsize R: & \scriptsize $40\,\pm\,12$
\end{tabular}\\
    \SetCell[r=2]{m} 50 & Unif. & \begin{tabular}{l@{\:}r@{}}
    \scriptsize C: & \scriptsize $4978\,\pm\,3$ \\
    \scriptsize M: & \scriptsize $0\,\pm\,0$ \\
    \scriptsize R: & \scriptsize $10000\,\pm\,0$
\end{tabular} & \begin{tabular}{l@{\:}r@{}}
    \scriptsize C: & \scriptsize $4978\,\pm\,3$ \\
    \scriptsize M: & \scriptsize $0\,\pm\,0$ \\
    \scriptsize R: & \scriptsize $10000\,\pm\,0$
\end{tabular} & \begin{tabular}{l@{\:}r@{}}
    \scriptsize C: & \scriptsize $4978\,\pm\,3$ \\
    \scriptsize M: & \scriptsize $0\,\pm\,0$ \\
    \scriptsize R: & \scriptsize $10000\,\pm\,0$
\end{tabular} & \SetCell{bg=gray!20} \begin{tabular}{l@{\:}r@{}}
    \scriptsize C: & \scriptsize $4978\,\pm\,3$ \\
    \scriptsize M: & \scriptsize $0\,\pm\,0$ \\
    \scriptsize\textbf{R:} & \scriptsize\bm{$10000\,\pm\,0$}
\end{tabular} & \begin{tabular}{l@{\:}r@{}}
    \scriptsize C: & \scriptsize $4974\,\pm\,4$ \\
    \scriptsize M: & \scriptsize $2\,\pm\,1$ \\
    \scriptsize R: & \scriptsize $10019\,\pm\,8$
\end{tabular} & \begin{tabular}{l@{\:}r@{}}
    \scriptsize C: & \scriptsize $3551\,\pm\,32$ \\
    \scriptsize M: & \scriptsize $738\,\pm\,25$ \\
    \scriptsize R: & \scriptsize $15268\,\pm\,221$
\end{tabular} & \begin{tabular}{l@{\:}r@{}}
    \scriptsize C: & \scriptsize $0\,\pm\,0$ \\
    \scriptsize M: & \scriptsize $3344\,\pm\,52$ \\
    \scriptsize R: & \scriptsize $36828\,\pm\,570$
\end{tabular}\\
     & Mix. & \begin{tabular}{l@{\:}r@{}}
    \scriptsize C: & \scriptsize $4988\,\pm\,4$ \\
    \scriptsize M: & \scriptsize $0\,\pm\,0$ \\
    \scriptsize R: & \scriptsize $10000\,\pm\,0$
\end{tabular} & \begin{tabular}{l@{\:}r@{}}
    \scriptsize C: & \scriptsize $4987\,\pm\,5$ \\
    \scriptsize M: & \scriptsize $0\,\pm\,0$ \\
    \scriptsize R: & \scriptsize $9998\,\pm\,3$
\end{tabular} & \begin{tabular}{l@{\:}r@{}}
    \scriptsize C: & \scriptsize $254\,\pm\,17$ \\
    \scriptsize M: & \scriptsize $0\,\pm\,0$ \\
    \scriptsize R: & \scriptsize $532\,\pm\,31$
\end{tabular} & \begin{tabular}{l@{\:}r@{}}
    \scriptsize C: & \scriptsize $0\,\pm\,0$ \\
    \scriptsize M: & \scriptsize $0\,\pm\,0$ \\
    \scriptsize R: & \scriptsize $24\,\pm\,9$
\end{tabular} & \begin{tabular}{l@{\:}r@{}}
    \scriptsize C: & \scriptsize $0\,\pm\,0$ \\
    \scriptsize M: & \scriptsize $0\,\pm\,0$ \\
    \scriptsize R: & \scriptsize $24\,\pm\,9$
\end{tabular} & \begin{tabular}{l@{\:}r@{}}
    \scriptsize C: & \scriptsize $0\,\pm\,0$ \\
    \scriptsize M: & \scriptsize $0\,\pm\,0$ \\
    \scriptsize R: & \scriptsize $24\,\pm\,9$
\end{tabular} & \SetCell{bg=gray!20} \begin{tabular}{l@{\:}r@{}}
    \scriptsize C: & \scriptsize $0\,\pm\,0$ \\
    \scriptsize M: & \scriptsize $0\,\pm\,0$ \\
    \scriptsize\textbf{R:} & \scriptsize\bm{$24\,\pm\,9$}
\end{tabular}\\
    \end{tblr}
    \end{adjustbox}
\end{table}

\begin{table}[!htbp]
\centering
    \caption{Performance Metrics on $\Sph^{d-1}$ for \vht\ with nearest-query distance ($T=5000$).}
    \label{tab:nearestquery_results_table_sphere}
    \begin{adjustbox}{max width=\textwidth,center}
    \begin{tblr}{colspec = { | Q[c,m]| Q[l,m]| Q[c,m]| Q[c,m]| Q[c,m]| Q[c,m]| Q[c,m]| Q[c,m]| },row{1} = {font=\small\bfseries,m},row{2-Z} = {font=\small,m},hlines,colsep=3pt, rowsep=3pt}
    Dim & Dist. & $\tau=0.00$ & $\tau=0.20$ & $\tau=0.40$ & $\tau=0.60$ & $\tau=0.80$ & $\tau=1.00$\\
    \SetCell[r=2]{m} 2 & Unif. & \begin{tabular}{l@{\:}r@{}}
    \scriptsize C: & \scriptsize $4969\,\pm\,10$ \\
    \scriptsize M: & \scriptsize $0\,\pm\,0$ \\
    \scriptsize R: & \scriptsize $9998\,\pm\,1$
\end{tabular} & \begin{tabular}{l@{\:}r@{}}
    \scriptsize C: & \scriptsize $142\,\pm\,7$ \\
    \scriptsize M: & \scriptsize $2\,\pm\,1$ \\
    \scriptsize R: & \scriptsize $371\,\pm\,15$
\end{tabular} & \SetCell{bg=gray!20} \begin{tabular}{l@{\:}r@{}}
    \scriptsize C: & \scriptsize $69\,\pm\,5$ \\
    \scriptsize M: & \scriptsize $5\,\pm\,3$ \\
    \scriptsize\textbf{R:} & \scriptsize\bm{$254\,\pm\,42$}
\end{tabular} & \begin{tabular}{l@{\:}r@{}}
    \scriptsize C: & \scriptsize $39\,\pm\,3$ \\
    \scriptsize M: & \scriptsize $18\,\pm\,6$ \\
    \scriptsize R: & \scriptsize $340\,\pm\,73$
\end{tabular} & \begin{tabular}{l@{\:}r@{}}
    \scriptsize C: & \scriptsize $24\,\pm\,2$ \\
    \scriptsize M: & \scriptsize $44\,\pm\,11$ \\
    \scriptsize R: & \scriptsize $596\,\pm\,144$
\end{tabular} & \begin{tabular}{l@{\:}r@{}}
    \scriptsize C: & \scriptsize $0\,\pm\,0$ \\
    \scriptsize M: & \scriptsize $315\,\pm\,148$ \\
    \scriptsize R: & \scriptsize $3524\,\pm\,1630$
\end{tabular}\\
     & Mix. & \begin{tabular}{l@{\:}r@{}}
    \scriptsize C: & \scriptsize $4988\,\pm\,4$ \\
    \scriptsize M: & \scriptsize $0\,\pm\,0$ \\
    \scriptsize R: & \scriptsize $9996\,\pm\,3$
\end{tabular} & \begin{tabular}{l@{\:}r@{}}
    \scriptsize C: & \scriptsize $100\,\pm\,5$ \\
    \scriptsize M: & \scriptsize $1\,\pm\,1$ \\
    \scriptsize R: & \scriptsize $231\,\pm\,17$
\end{tabular} & \SetCell{bg=gray!20} \begin{tabular}{l@{\:}r@{}}
    \scriptsize C: & \scriptsize $50\,\pm\,5$ \\
    \scriptsize M: & \scriptsize $3\,\pm\,1$ \\
    \scriptsize\textbf{R:} & \scriptsize\bm{$155\,\pm\,19$}
\end{tabular} & \begin{tabular}{l@{\:}r@{}}
    \scriptsize C: & \scriptsize $30\,\pm\,2$ \\
    \scriptsize M: & \scriptsize $9\,\pm\,3$ \\
    \scriptsize R: & \scriptsize $180\,\pm\,30$
\end{tabular} & \begin{tabular}{l@{\:}r@{}}
    \scriptsize C: & \scriptsize $18\,\pm\,3$ \\
    \scriptsize M: & \scriptsize $22\,\pm\,9$ \\
    \scriptsize R: & \scriptsize $302\,\pm\,100$
\end{tabular} & \begin{tabular}{l@{\:}r@{}}
    \scriptsize C: & \scriptsize $0\,\pm\,0$ \\
    \scriptsize M: & \scriptsize $345\,\pm\,77$ \\
    \scriptsize R: & \scriptsize $3812\,\pm\,849$
\end{tabular}\\
    \SetCell[r=2]{m} 4 & Unif. & \begin{tabular}{l@{\:}r@{}}
    \scriptsize C: & \scriptsize $4989\,\pm\,2$ \\
    \scriptsize M: & \scriptsize $0\,\pm\,0$ \\
    \scriptsize R: & \scriptsize $10000\,\pm\,0$
\end{tabular} & \begin{tabular}{l@{\:}r@{}}
    \scriptsize C: & \scriptsize $4117\,\pm\,19$ \\
    \scriptsize M: & \scriptsize $1\,\pm\,1$ \\
    \scriptsize R: & \scriptsize $8272\,\pm\,46$
\end{tabular} & \begin{tabular}{l@{\:}r@{}}
    \scriptsize C: & \scriptsize $2526\,\pm\,21$ \\
    \scriptsize M: & \scriptsize $22\,\pm\,6$ \\
    \scriptsize R: & \scriptsize $5319\,\pm\,79$
\end{tabular} & \SetCell{bg=gray!20} \begin{tabular}{l@{\:}r@{}}
    \scriptsize C: & \scriptsize $1425\,\pm\,9$ \\
    \scriptsize M: & \scriptsize $103\,\pm\,11$ \\
    \scriptsize\textbf{R:} & \scriptsize\bm{$4005\,\pm\,126$}
\end{tabular} & \begin{tabular}{l@{\:}r@{}}
    \scriptsize C: & \scriptsize $696\,\pm\,9$ \\
    \scriptsize M: & \scriptsize $307\,\pm\,9$ \\
    \scriptsize R: & \scriptsize $4793\,\pm\,92$
\end{tabular} & \begin{tabular}{l@{\:}r@{}}
    \scriptsize C: & \scriptsize $0\,\pm\,0$ \\
    \scriptsize M: & \scriptsize $1939\,\pm\,308$ \\
    \scriptsize R: & \scriptsize $21348\,\pm\,3381$
\end{tabular}\\
     & Mix. & \begin{tabular}{l@{\:}r@{}}
    \scriptsize C: & \scriptsize $4989\,\pm\,3$ \\
    \scriptsize M: & \scriptsize $0\,\pm\,0$ \\
    \scriptsize R: & \scriptsize $10000\,\pm\,0$
\end{tabular} & \begin{tabular}{l@{\:}r@{}}
    \scriptsize C: & \scriptsize $121\,\pm\,10$ \\
    \scriptsize M: & \scriptsize $0\,\pm\,0$ \\
    \scriptsize R: & \scriptsize $264\,\pm\,19$
\end{tabular} & \begin{tabular}{l@{\:}r@{}}
    \scriptsize C: & \scriptsize $20\,\pm\,3$ \\
    \scriptsize M: & \scriptsize $0\,\pm\,0$ \\
    \scriptsize R: & \scriptsize $62\,\pm\,11$
\end{tabular} & \begin{tabular}{l@{\:}r@{}}
    \scriptsize C: & \scriptsize $4\,\pm\,3$ \\
    \scriptsize M: & \scriptsize $0\,\pm\,0$ \\
    \scriptsize R: & \scriptsize $31\,\pm\,3$
\end{tabular} & \begin{tabular}{l@{\:}r@{}}
    \scriptsize C: & \scriptsize $0\,\pm\,0$ \\
    \scriptsize M: & \scriptsize $0\,\pm\,0$ \\
    \scriptsize R: & \scriptsize $22\,\pm\,6$
\end{tabular} & \SetCell{bg=gray!20} \begin{tabular}{l@{\:}r@{}}
    \scriptsize C: & \scriptsize $0\,\pm\,0$ \\
    \scriptsize M: & \scriptsize $0\,\pm\,0$ \\
    \scriptsize\textbf{R:} & \scriptsize\bm{$22\,\pm\,7$}
\end{tabular}\\
    \SetCell[r=2]{m} 10 & Unif. & \begin{tabular}{l@{\:}r@{}}
    \scriptsize C: & \scriptsize $4990\,\pm\,3$ \\
    \scriptsize M: & \scriptsize $0\,\pm\,0$ \\
    \scriptsize R: & \scriptsize $10000\,\pm\,0$
\end{tabular} & \begin{tabular}{l@{\:}r@{}}
    \scriptsize C: & \scriptsize $4990\,\pm\,3$ \\
    \scriptsize M: & \scriptsize $0\,\pm\,0$ \\
    \scriptsize R: & \scriptsize $10000\,\pm\,0$
\end{tabular} & \begin{tabular}{l@{\:}r@{}}
    \scriptsize C: & \scriptsize $4983\,\pm\,5$ \\
    \scriptsize M: & \scriptsize $1\,\pm\,0$ \\
    \scriptsize R: & \scriptsize $9993\,\pm\,7$
\end{tabular} & \SetCell{bg=gray!20} \begin{tabular}{l@{\:}r@{}}
    \scriptsize C: & \scriptsize $4766\,\pm\,7$ \\
    \scriptsize M: & \scriptsize $17\,\pm\,5$ \\
    \scriptsize\textbf{R:} & \scriptsize\bm{$9743\,\pm\,51$}
\end{tabular} & \begin{tabular}{l@{\:}r@{}}
    \scriptsize C: & \scriptsize $3412\,\pm\,25$ \\
    \scriptsize M: & \scriptsize $284\,\pm\,12$ \\
    \scriptsize R: & \scriptsize $9971\,\pm\,142$
\end{tabular} & \begin{tabular}{l@{\:}r@{}}
    \scriptsize C: & \scriptsize $0\,\pm\,0$ \\
    \scriptsize M: & \scriptsize $3039\,\pm\,213$ \\
    \scriptsize R: & \scriptsize $33446\,\pm\,2337$
\end{tabular}\\
     & Mix. & \begin{tabular}{l@{\:}r@{}}
    \scriptsize C: & \scriptsize $4989\,\pm\,4$ \\
    \scriptsize M: & \scriptsize $0\,\pm\,0$ \\
    \scriptsize R: & \scriptsize $10000\,\pm\,0$
\end{tabular} & \begin{tabular}{l@{\:}r@{}}
    \scriptsize C: & \scriptsize $1769\,\pm\,45$ \\
    \scriptsize M: & \scriptsize $0\,\pm\,0$ \\
    \scriptsize R: & \scriptsize $3560\,\pm\,86$
\end{tabular} & \begin{tabular}{l@{\:}r@{}}
    \scriptsize C: & \scriptsize $105\,\pm\,5$ \\
    \scriptsize M: & \scriptsize $0\,\pm\,0$ \\
    \scriptsize R: & \scriptsize $231\,\pm\,7$
\end{tabular} & \begin{tabular}{l@{\:}r@{}}
    \scriptsize C: & \scriptsize $3\,\pm\,2$ \\
    \scriptsize M: & \scriptsize $0\,\pm\,0$ \\
    \scriptsize R: & \scriptsize $28\,\pm\,8$
\end{tabular} & \begin{tabular}{l@{\:}r@{}}
    \scriptsize C: & \scriptsize $0\,\pm\,0$ \\
    \scriptsize M: & \scriptsize $0\,\pm\,0$ \\
    \scriptsize R: & \scriptsize $22\,\pm\,9$
\end{tabular} & \SetCell{bg=gray!20} \begin{tabular}{l@{\:}r@{}}
    \scriptsize C: & \scriptsize $0\,\pm\,0$ \\
    \scriptsize M: & \scriptsize $0\,\pm\,0$ \\
    \scriptsize\textbf{R:} & \scriptsize\bm{$22\,\pm\,9$}
\end{tabular}\\
    \SetCell[r=2]{m} 50 & Unif. & \begin{tabular}{l@{\:}r@{}}
    \scriptsize C: & \scriptsize $4990\,\pm\,2$ \\
    \scriptsize M: & \scriptsize $0\,\pm\,0$ \\
    \scriptsize R: & \scriptsize $10000\,\pm\,0$
\end{tabular} & \begin{tabular}{l@{\:}r@{}}
    \scriptsize C: & \scriptsize $4990\,\pm\,2$ \\
    \scriptsize M: & \scriptsize $0\,\pm\,0$ \\
    \scriptsize R: & \scriptsize $10000\,\pm\,0$
\end{tabular} & \begin{tabular}{l@{\:}r@{}}
    \scriptsize C: & \scriptsize $4990\,\pm\,2$ \\
    \scriptsize M: & \scriptsize $0\,\pm\,0$ \\
    \scriptsize R: & \scriptsize $10000\,\pm\,0$
\end{tabular} & \SetCell{bg=gray!20} \begin{tabular}{l@{\:}r@{}}
    \scriptsize C: & \scriptsize $4990\,\pm\,2$ \\
    \scriptsize M: & \scriptsize $0\,\pm\,0$ \\
    \scriptsize\textbf{R:} & \scriptsize\bm{$10000\,\pm\,0$}
\end{tabular} & \begin{tabular}{l@{\:}r@{}}
    \scriptsize C: & \scriptsize $4987\,\pm\,2$ \\
    \scriptsize M: & \scriptsize $1\,\pm\,1$ \\
    \scriptsize R: & \scriptsize $10005\,\pm\,10$
\end{tabular} & \begin{tabular}{l@{\:}r@{}}
    \scriptsize C: & \scriptsize $0\,\pm\,0$ \\
    \scriptsize M: & \scriptsize $3549\,\pm\,83$ \\
    \scriptsize R: & \scriptsize $39061\,\pm\,911$
\end{tabular}\\
     & Mix. & \begin{tabular}{l@{\:}r@{}}
    \scriptsize C: & \scriptsize $4987\,\pm\,6$ \\
    \scriptsize M: & \scriptsize $0\,\pm\,0$ \\
    \scriptsize R: & \scriptsize $10000\,\pm\,0$
\end{tabular} & \begin{tabular}{l@{\:}r@{}}
    \scriptsize C: & \scriptsize $4987\,\pm\,6$ \\
    \scriptsize M: & \scriptsize $0\,\pm\,0$ \\
    \scriptsize R: & \scriptsize $10000\,\pm\,0$
\end{tabular} & \begin{tabular}{l@{\:}r@{}}
    \scriptsize C: & \scriptsize $4983\,\pm\,7$ \\
    \scriptsize M: & \scriptsize $0\,\pm\,0$ \\
    \scriptsize R: & \scriptsize $9992\,\pm\,1$
\end{tabular} & \begin{tabular}{l@{\:}r@{}}
    \scriptsize C: & \scriptsize $1287\,\pm\,44$ \\
    \scriptsize M: & \scriptsize $0\,\pm\,0$ \\
    \scriptsize R: & \scriptsize $2600\,\pm\,97$
\end{tabular} & \begin{tabular}{l@{\:}r@{}}
    \scriptsize C: & \scriptsize $12\,\pm\,3$ \\
    \scriptsize M: & \scriptsize $0\,\pm\,0$ \\
    \scriptsize R: & \scriptsize $49\,\pm\,10$
\end{tabular} & \SetCell{bg=gray!20} \begin{tabular}{l@{\:}r@{}}
    \scriptsize C: & \scriptsize $0\,\pm\,0$ \\
    \scriptsize M: & \scriptsize $0\,\pm\,0$ \\
    \scriptsize\textbf{R:} & \scriptsize\bm{$26\,\pm\,12$}
\end{tabular}\\
    \end{tblr}
    \end{adjustbox}
\end{table}

\begin{table}[!htbp]
\centering
    \caption{Voronoi regret of all algorithms for each experimental setup ($T=5000$). Results for \extc, \vht\ and \texttt{AMP} are reported for the best-performing threshold for that setup.}
    \label{tab:synth_summary}
    \begin{adjustbox}{max width=\textwidth,center}
    \begin{tblr}{
      colspec = { | Q[c,m] | Q[c,m] | Q[l,m] | Q[c,m] | Q[c,m] | Q[c,m] | Q[c,m] | Q[c,m] | Q[c,m] | },
      row{1} = {font=\small\bfseries,m},
      row{2-Z} = {font=\small,m},
      hlines,
      colsep=3pt, 
      rowsep=3pt
    }
    $\mathcal{E}$ & Dim. & Dist. & \extc\  & \vhc\ & \vht\ & Nearest-query \vht & \texttt{SKM} & \texttt{AMP} \\
    \SetCell[r=8]{m} $\mathcal{I}^d$ & \SetCell[r=2]{m} 1 & Unif. & $6123\,\pm\,390$ & $142\,\pm\,11$ & \SetCell{bg=gray!20} \bm{$110\,\pm\,7$} & $186\,\pm\,33$ & $14128\,\pm\,4275$ & N/A \\
     & & Mix. & $4305\,\pm\,770$ & $130\,\pm\,21$ & \SetCell{bg=gray!20} \bm{$106\,\pm\,18$} & $245\,\pm\,50$ & $11141\,\pm\,3741$ & N/A \\
     & \SetCell[r=2]{m} 4 & Unif. & $9563\,\pm\,486$ & $2972\,\pm\,72$ & \SetCell{bg=gray!20} \bm{$1593\,\pm\,35$} & $5396\,\pm\,95$ & $30055\,\pm\,3659$ & N/A \\
     & & Mix. & $1064\,\pm\,142$ & $2337\,\pm\,19$ & \SetCell{bg=gray!20} \bm{$573\,\pm\,52$} & $1366\,\pm\,99$ & $793\,\pm\,94$ & N/A \\
     & \SetCell[r=2]{m} 10 & Unif. & $9782\,\pm\,462$ & $9489\,\pm\,35$ & \SetCell{bg=gray!20} \bm{$5233\,\pm\,94$} & $9544\,\pm\,91$ & $38900\,\pm\,4268$ & N/A \\
     & & Mix. & $33\,\pm\,7$ & $8821\,\pm\,48$ & $23\,\pm\,4$ & $29\,\pm\,10$ & \SetCell{bg=gray!20} \bm{$20\,\pm\,5$} & N/A \\
     & \SetCell[r=2]{m} 50 & Unif. & $9597\,\pm\,143$ & $10000\,\pm\,0$ & \SetCell{bg=gray!20} \bm{$9559\,\pm\,36$} & $10000\,\pm\,0$ & $42659\,\pm\,3378$ & N/A \\
     & & Mix. & $24\,\pm\,9$ & $10000\,\pm\,0$ & \SetCell{bg=gray!20} \bm{$23\,\pm\,5$} & $24\,\pm\,9$ & \SetCell{bg=gray!20} \bm{$23\,\pm\,5$} & N/A \\
    \SetCell[r=8]{m} $\Sph^{d-1}$ & \SetCell[r=2]{m} 2 & Unif. & $9347\,\pm\,472$ & $132\,\pm\,12$ & \SetCell{bg=gray!20} \bm{$125\,\pm\,6$} & $254\,\pm\,42$ & $26804\,\pm\,1923$ & $4967\,\pm\,421$ \\
     & & Mix. & $963\,\pm\,173$ & $138\,\pm\,8$ & \SetCell{bg=gray!20} \bm{$104\,\pm\,13$} & $155\,\pm\,19$ & $791\,\pm\,645$ & $10000\,\pm\,0$ \\
     & \SetCell[r=2]{m} 4 & Unif. & $6099\,\pm\,889$ & $1225\,\pm\,31$ & \SetCell{bg=gray!20} \bm{$836\,\pm\,40$} & $4005\,\pm\,126$ & $25645\,\pm\,8157$ & $2799\,\pm\,110$ \\
     & & Mix. & \SetCell{bg=gray!20} \bm{$22\,\pm\,7$} & $1083\,\pm\,56$ & \SetCell{bg=gray!20} \bm{$22\,\pm\,7$} & \SetCell{bg=gray!20} \bm{$22\,\pm\,7$} & \SetCell{bg=gray!20} \bm{$22\,\pm\,7$} & $312\,\pm\,170$ \\
     & \SetCell[r=2]{m} 10 & Unif. & $7657\,\pm\,544$ & $4878\,\pm\,150$ & \SetCell{bg=gray!20} \bm{$2550\,\pm\,123$} & $9743\,\pm\,51$ & $39956\,\pm\,3931$ & $3902\,\pm\,158$ \\
     & & Mix. & $22\,\pm\,9$ & $7490\,\pm\,74$ & \SetCell{bg=gray!20} \bm{$20\,\pm\,11$} & $22\,\pm\,9$ & \SetCell{bg=gray!20} \bm{$20\,\pm\,11$} & $50\,\pm\,21$ \\
     & \SetCell[r=2]{m} 50 & Unif. & $8947\,\pm\,296$ & $9983\,\pm\,4$ & \SetCell{bg=gray!20} \bm{$6135\,\pm\,73$} & $10000\,\pm\,0$ & $41246\,\pm\,1091$ & $6796\,\pm\,156$ \\
     & & Mix. & $26\,\pm\,12$ & $10000\,\pm\,0$ & \SetCell{bg=gray!20} \bm{$19\,\pm\,5$} & $26\,\pm\,12$ & \SetCell{bg=gray!20} \bm{$19\,\pm\,5$} & $28\,\pm\,8$ \\
    \end{tblr}
    \end{adjustbox}
\end{table}

\clearpage

%% file: b.appendix/10.Datasets.tex
\subsection{Real-World Experiments}
\subsubsection{Real-World Datasets}
\label{Datasets}

\begin{figure}[ht!]
  \centering
  \begin{subfigure}[b]{0.33\textwidth}
    \centering
    \includegraphics[width=\linewidth]{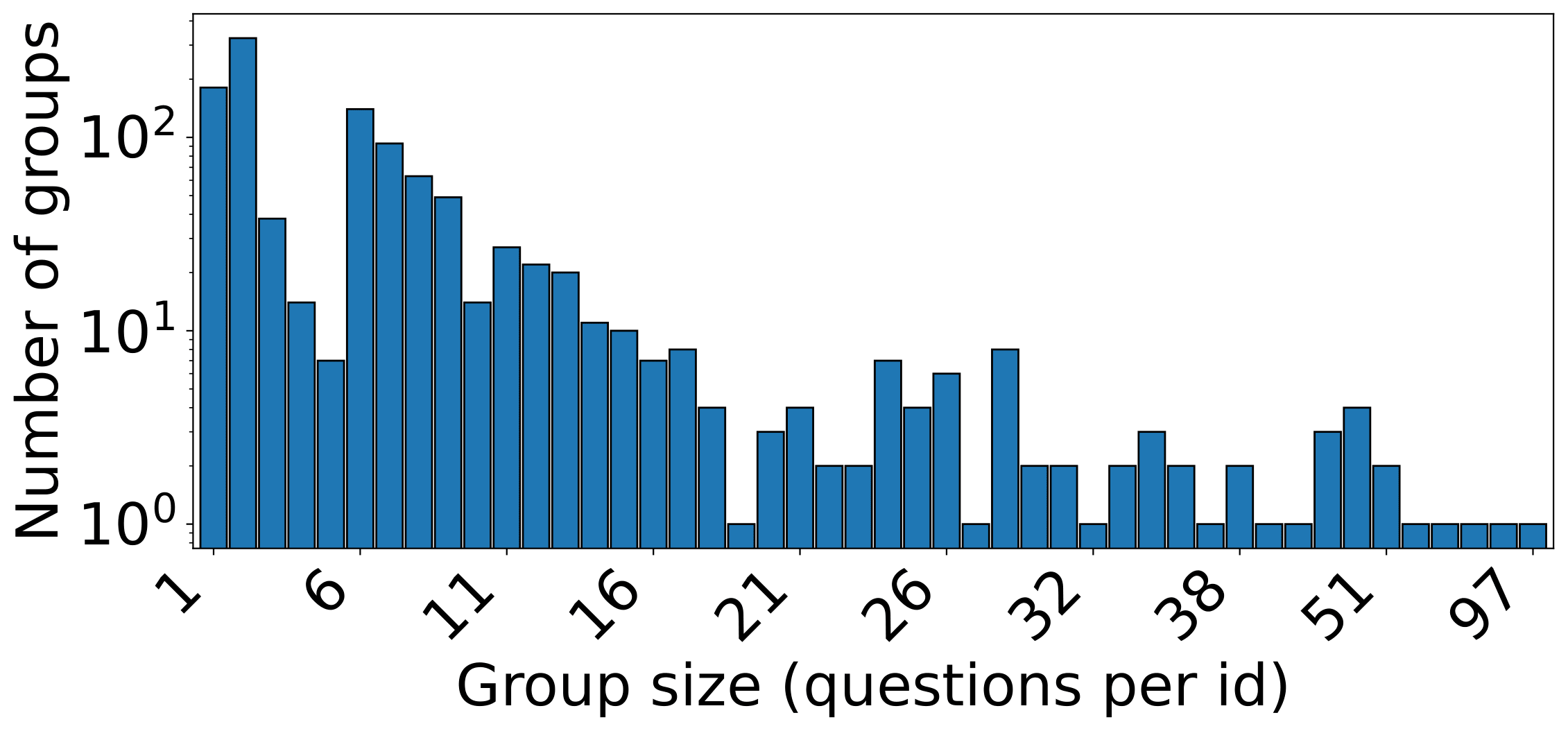}
    \caption{Quora Question Groups}
    \label{fig:img1}
  \end{subfigure}\hfill
  \begin{subfigure}[b]{0.33\textwidth}
    \centering
    \includegraphics[width=\linewidth]{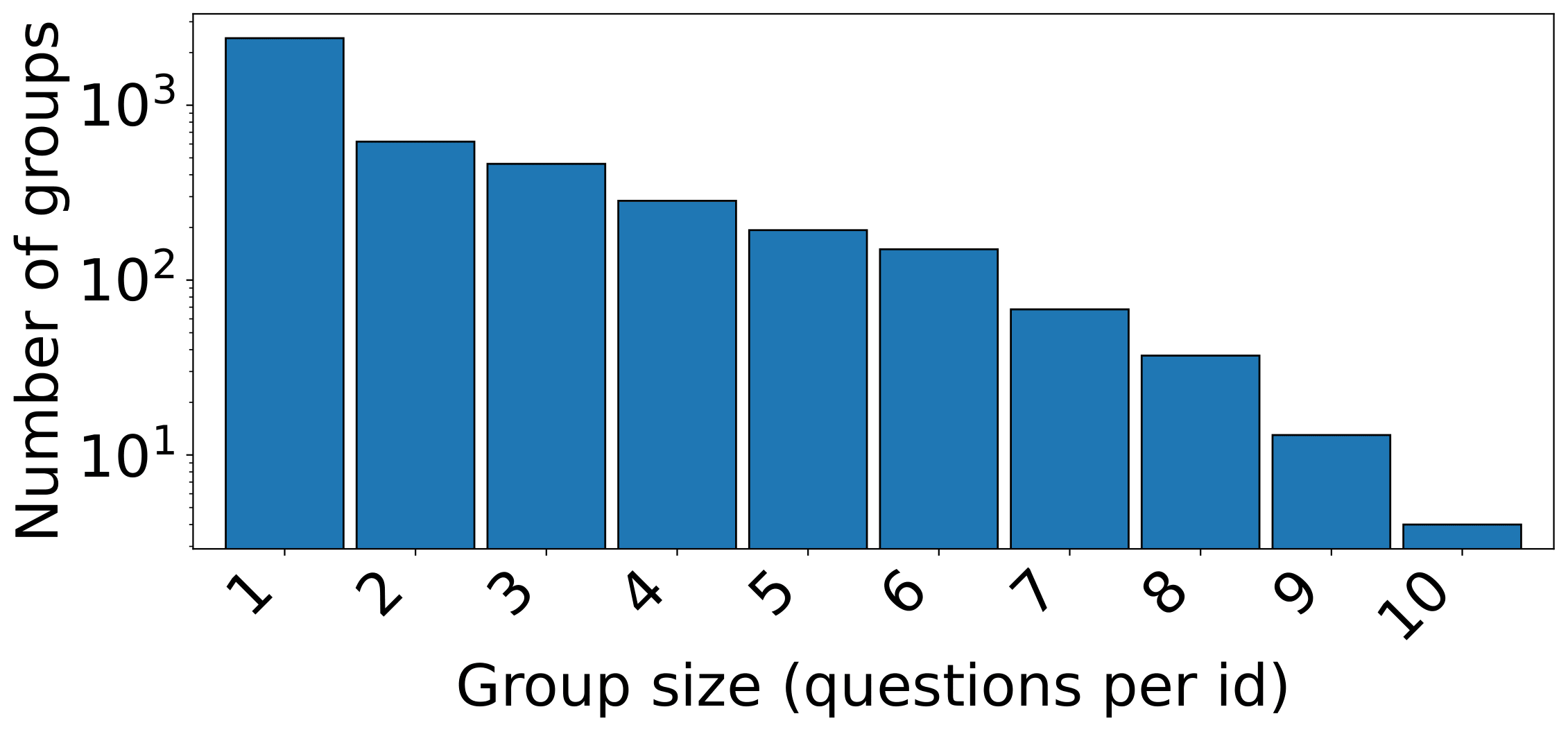}
    \caption{ComQA}
    \label{fig:img2_comqa}
  \end{subfigure}\hfill
  \begin{subfigure}[b]{0.33\textwidth}
    \centering
    \includegraphics[width=\linewidth]{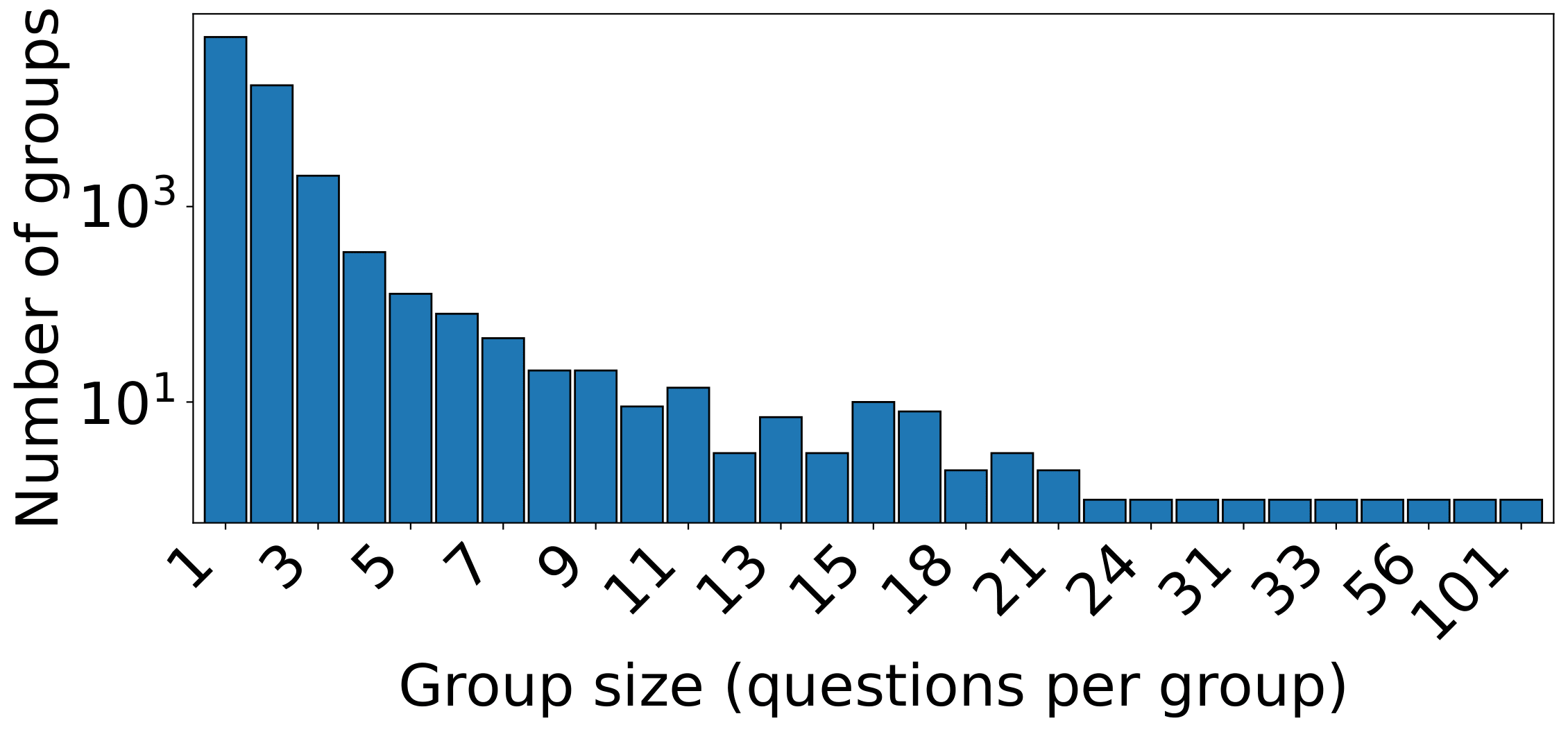}
    \caption{CQADupStack}
    \label{fig:img2_cqa}
  \end{subfigure}
  \caption{Distributions of group sizes: number of questions per group (with the same answer) within the datasets.}
  \label{fig:Distribution of our datasets}
\end{figure}

In this section, we describe the question-answer datasets used in our real-world experiments. We say that two questions belong to the same \emph{question group} if a single answer can adequately address both, i.e., if their correct answer label is identical.

\textbf{Quora Question Groups (QQG):} Quora Question Groups is derived from the Quora Question Pairs dataset~\citep{DBLP:journals/corr/WangHF17}. The original dataset comprises over 400{,}000 question pairs, each annotated with a binary label indicating whether the questions are paraphrases of each other. To transform this pairwise paraphrase data into coherent question groups, we represented the dataset as a graph where nodes correspond to questions and edges indicate paraphrase relationships.

To construct the Quora Question Groups dataset, we first extracted all connected components from the graph, treating each component as a candidate group. From these, we sampled a subset with the goal of obtaining approximately 1,200 groups, favoring components with larger group sizes. We then manually reviewed and cleaned the selected groups to ensure that all questions within a group could be addressed by the same answer. This involved discarding noisy or weakly related links and merging or splitting groups where appropriate. Throughout this process, we preserved the structural assumptions of the original dataset while focusing on high-quality, semantically coherent groupings. At the end of this process, we obtained $1,103$ distinct groups comprising a total of $7,365$ curated questions.

\textbf{ComQA\footnote{https://paperswithcode.com/dataset/comqa} \citep{abujabal2019comqa}}:
ComQA consists of 11{,}214 English questions collected from the WikiAnswers forum and grouped into 4{,}834 paraphrase clusters by crowd workers. By design, ComQA emphasizes reasoning-intensive phenomena—namely compositional queries, explicit temporal constraints, and comparative or superlative constructions—making it a rigorous benchmark for research on complex factoid question answering. 
The middle panel of Fig.  \ref{fig:Distribution of our datasets} presents the group size distribution of this dataset.

\textbf{CQADupStack \footnote{https://github.com/D1Doris/CQADupStack} \citep{hoogeven2015}}: CQADupStack is a public benchmark available to researchers working on question-answering learning tasks. It consists of complete threads from twelve StackExchange subforums (android, english, gis, mathematica, physics, programmers, statistics, superuser, tex, unix, webmasters, wordpress) annotated with community‐marked duplicate links. To evaluate our classification approach across distinct expert domains, we apply it to the Mathematica, Physics, and Statistics parts of the dataset. The right panel of Fig.  \ref{fig:Distribution of our datasets} presents the group size distribution of this dataset.
The subset of the dataset we use comprises 99,785 questions organized into 74,519 groups. CQADupStack is the most challenging dataset used in our experiments, because of its relatively small average group size.

%% file: b.appendix/11.Our_LLMs.tex
\subsubsection{Text Embedding Models}\label{OurLLMS}

In this section, we describe the text embedding models selected for our experiments, along with the way they were trained and fine-tuned. We chose three models that are widely recognized in the literature and frequently cited due to their strong performance across a range of Natural Language Processing tasks.

\textbf{E5 \citep{wang2022text}:} EmbEddings from bidirEctional Encoder rEpresentations (E5) is available in three versions. The first version is initialized with MiniLM \citep{DBLP:journals/corr/abs-2002-10957}, the second is initialized with BERT-base \citep{reimers-gurevych-2019-sentence}, and the third is initialized with BERT-large. This model follows a bi-encoder architecture, where both the query and passage encoders are initialized with BERT. The training process consists of two stages. The first stage, called ``unsupervised'', uses a large number of unlabeled data, including title and passage pairs from Wikipedia, questions and answers from Reddit, and more. The InfoNCE contrastive loss \citep{DBLP:journals/corr/abs-1807-03748} is employed to minimize the distance between related queries and passages, while maximizing the distance between unrelated queries and passages. The second stage involves training the model on high-quality human-annotated data, such as NLI \citep{gao-etal-2021-simcse}, MS-MARCO \citep{bajaj2016ms}, and Natural Questions \citep{karpukhin-etal-2020-dense}. During this stage, the model is trained with a loss that combines the KL divergence between the probability distribution of the label, as given by a cross-encoder teacher model, and the probability distribution generated by our E5 student model, along with the InfoNCE contrastive loss. This second stage further improves model performance on benchmarks such as BEIR \citep{thakur2021beir} and MTEB \citep{muennighoff2022mteb}.

\medskip
\noindent
\textbf{NOMIC \citep{nussbaum2024nomic}:} Nomic is initialized from BERT and modified to address long-context retrieval. Nomic consists of 100 million parameters and supports a sequence length of up to 2048. Nomic's pretraining is divided into three stages. The first stage focuses on Masked Language Modeling to learn longer sequence representations. The subsequent stages are supervised and unsupervised, both employing the InfoNCE contrastive loss. This model was trained on a significantly larger dataset compared to the previous versions, encompassing both supervised and unsupervised phases. Nomic uses task-specific prefixes—such as search, search query document, classification, and clustering—to distinguish between the behaviors of each task. For the purpose of our work, we used the clustering prefix.

\medskip
\noindent
\textbf{Mistral\_E5 \citep{wang-etal-2024-improving-text}:} This is the first unidirectional decoder architecture we use for our work. The model is initialized from Mistral 7B \cite{jiang2023mistral} and consists of 7 billion parameters. The model takes as input the query \(q^+\) and the task\_definition and generates the instruction template:
\[
q^+_{inst} = \text{‘Instruct: \{task\_definition\} \textbackslash n Query: \{$q^+\}$’}
\]

Where the task definition is:

``Given a web search query, retrieve relevant search queries that paraphrase the query.''

The query and document instruction templates are then passed to the model. We obtain the query and document embeddings, \(h^+_q\) and \(h^+_d\), from the last layer of the model at the [EOS] position. The model was trained on a large corpus of both original and synthetic data. The synthetic data was generated using advanced large language models such as GPT-4.

%% file: b.appendix/12.real_experiments.tex
\subsubsection{Additional Real-World 
Experiments on QQG}\label{sec:additional_real_world_experiments}
In this section, we evaluate the performance of \vht\ (Algorithm \ref{alg:vht}), \extc\ and the Sequential $k$-Means (SKM) baseline using the datasets and embedding models described in Sections \ref{Datasets} and \ref{OurLLMS} respectively. The models we utilize are typically trained using a loss function based on cosine similarity. Since they output unit-normalized embeddings, maximizing cosine similarity is equivalent to minimizing Euclidean distance. Hence, the reliance of \vht\ on Euclidean distance naturally leverages the models’ cosine-trained embeddings. As explained in Section \ref{sec:experiments}, we always initialize the algorithms by manually picking an example question from each group. For \vht, this skips the first phase where \vhc\ is applied until every group is represented. We then run our algorithms on the remaining questions in the dataset considered. 

To align with the Retrieval-Augmented Generation (RAG) literature~\citep{DBLP:conf/nips/LewisPPPKGKLYR020} and improve computational efficiency, it is sensible to consider the following modification of the \vht\ algorithm: instead of computing distances to the convex hulls $\hat{\mathcal{C}}_{i,t}$ of each set of expert queries \(\mathcal{Q}_{i,t}\), we compute the distance to the nearest query in \(\mathcal{Q}_{i,t}\). We refer to this version of \vht\ as “nearest-query \vht”. We compare the performance of \vht\ with three distances: the distance to the nearest query, the distance to Euclidean convex hulls, and the distance to spherical convex hulls. 

In Table ~\ref{regrettable}, we present the average cumulative regret for a range of thresholds across all three retriever models—E5, Nomic, and Mistral\_E5 on the QQG dataset. We observe that  changing the distance does not meaningfully affect the regret. This is expected in high dimension, as the query hulls $\hat{\mathcal{C}}_{i,_t}$ then typically occupy a negligible portion of the space, see discussion below Corollary \ref{cor:explicit_regret_bound}. We have made the same observation for the two other datasets. We also observe that the performance of each model is highly sensitive to the threshold, with optimal performance often concentrated around specific values (namely 0.85–0.90). This suggests the importance of careful threshold tuning for maximizing algorithm performance. Lastly, Mistral\_E5, owing to its recent development and high-dimensional latent space, demonstrates lower cumulative regret across all distance metrics for the QQG dataset. This effect is particularly pronounced at higher thresholds, indicating its robustness to the choice of metric and its superior capacity for generalization across different geometric interpretations of similarity.


\begin{table}[h!]
\centering
    
\caption{Average cumulative regret of \vht\  using different distance metrics after 6262 steps on QQG dataset.}
\label{regrettable}
\begin{tabular}{|c|c|c|c|c|c|c|c|c|c|c|}
\hline
\textbf{Threshold ($\tau$)} &  \textbf{0.10} & \textbf{0.20} & \textbf{0.40} & \textbf{0.60} & \textbf{0.70} & \textbf{0.80} & \textbf{0.85} & \textbf{0.90} & \textbf{0.95} & \textbf{1.00} \\
\hline
\multicolumn{11}{|c|}{\textbf{Nearest-query}} \\
\hline
E5 &    6722 &   6632   &  5750     &   3756   &  2479 &    1365   &   1132   &    \textbf{1062} &    1964    &   5822\\
\hline
Nomic   & 6790 &  6624 &  6709  &  5499  &  4098  &  2398  & 1677 &  \textbf{1239}  & 1442  & 5852  \\
\hline
Mistral\_E5   &  8192     &  6932 &  6401 &  4616 & 3157 &  1596 & 1196 &  \textbf{921} &  1269 &  3109 \\
\hline
\multicolumn{11}{|c|}{\textbf{Spherical hull}} \\
\hline
E5   &   7942   &  7582   &  6060 &  3398 &  2163 & 1422 & \textbf{ 1101} &  1192  & 1927 &  5734 \\
\hline
Nomic    & 7920 &  7936 & 7706 &  5410 & 3608 & 1996 & 1442 & \textbf{ 1140} & 1422 & 4928 \\
\hline
Mistral\_E5    &  7888 &  7735 &  6723 & 4137 & 2637 & 1414 & 1062  & \textbf{ 1002} &  1293 & 4620 \\
\hline
\multicolumn{11}{|c|}{\textbf{Euclidean hull}} \\
\hline
E5    & 7949 & 7628 & 6073 & 3424 & 2238 & 1358 & \textbf{ 1165} & 1197 &    1773 &  5016 \\
\hline
Nomic    &  7874 & 7916 & 7818 &  5652 & 3849 & 2178 & 1555 &      \textbf{ 1215} & 1445 &  6457 \\
\hline
Mistral\_E5  & 7898 & 7742 & 6749 & 4217 & 2705 & 1476 & 1074 & \textbf{ 836} & 1204 &  4048 \\
\hline

\end{tabular}

\end{table}

Next, we evaluate the Sequential $k$-Means (\texttt{SKM}) baseline described in Appendix \ref{app:synthetic_ghc_skm} and the tunable version of \extc\ described in Appendix \ref{app:synthetic_etc_ghc}, for each embedding model on the QQG dataset. The length of the exploration phase of \extc\ is determined by the threshold $$\hat{m}(t)=\frac{C\sigma^2}{\hat{\delta}^2_{\min}(t)}(d+2\log{T})$$ where we have chosen $\sigma=0.1$ to allow easier comparison with our findings on synthetic datasets (note that the value of $\sigma$ can be chosen arbitrarily as it amounts to scaling $C$ by a fixed constant). 

We report the regrets of both algorithms in Table \ref{table:baselines_regret}. We observe that \extc\ has two distinct behaviors depending on the value of $C$: if $C$ is sufficiently small, $\hat{m}(t)$ typically remains bounded by $1$ at each round $t$. As the algorithm is initialized with one example per group, this means that the expert is never called and the regret solely comes from the wrong guesses. If instead $C$ is sufficiently large,  $\hat{m}(t)$ typically remains larger than $1$ and \extc\ calls the expert up until another query for each group is observed. Given that on the QQG dataset, the number of groups $N=1103$ is non-negligible compared to size of the dataset $T=6262$, this results in \extc\ calling the expert at each round $t$ and high regret. This highlights that \extc\ is poorly suited to settings where $T$ is comparable to $N\log{N}$, as mentioned in Section \ref{sec:threshold}. The \texttt{SKM} baseline performs better than \extc\ for two out of three embedding models. Both algorithms are still significantly outperformed by properly tuned \vht\ using any of the distance metrics reported in Table \ref{regrettable}.

\begin{table}[h!]
\centering
\caption{Average cumulative regret of \extc\ and \texttt{SKM} baseline after 6262 steps on QQG dataset.}
\label{table:baselines_regret}
\begin{tabular}{|l|c|c|c|c|c|c|}
\hline
\textbf{Algorithm} & \multicolumn{5}{c|}{\textbf{\extc}} & \textbf{\texttt{SKM}} \\
\hline
\textbf{Threshold ($C$)} & \textbf{0.0} & \textbf{0.01} & \textbf{0.025} & \textbf{0.05} & \textbf{0.1} & \textbf{N/A} \\
\hline
E5 & \textbf{5533} & 5533 & 12524 & 12524 & 12524 & 3531 \\
\hline
Nomic & \textbf{5654} & 5654 & 5654 & 5654 & 12524 & 6204 \\
\hline
Mistral\_E5 & \textbf{3289} & 12524 & 12524 & 12524 & 12524 & 2442 \\
\hline
\end{tabular}
\end{table}

In Fig. \ref{fig:results_on_COMQA_and_CQADupStack}, we present the performance of our models on the ComQA and CQADupStack datasets using the nearest-query version of \vht. Nomic surprisingly outperforms Mistral\_E5 on both datasets, achieving a lower accumulated regret of 1,870 compared to 2,081 on ComQA, and 7,549 compared to 7,586 on CQADupStack. We also observe that Mistral\_E5 accumulates less regret than E5-large, which is an expected outcome given the design and capabilities of the two models.
 
Furthermore, by comparing the performance of the models between datasets, we can observe that the task becomes more challenging as the average number of questions per group decreases. Indeed, the lowest accumulated regret for QQG was 921 after handling around 6,200 questions, whereas it reached 1,870 (resp. 2,070) for ComQA (resp. CQADupStack) after handling 4,800 questions only. We also observe that the optimal threshold for QQG is 0.9, while it is 0.8 for ComQA. Among all datasets, CQADupStack is in principle the most challenging, with an average of 1.34 questions per group. This is reflected in the optimal threshold, which is significantly lower at 0.3. This suggests that the fewer the questions per group, the lower the confidence our model should have in its predictions. 

We note that even in the challenging datasets, ComQA and CQADupStack, the growth of the accumulated regret decreases over time. This indicates that the algorithm is learning effectively and consistently improving its performance over time. More specifically, in the CQADupStack dataset, the average regret per step decreases from about one-half after 3,000 steps to about one-third after 24,000 steps, indicating a clear improvement in performance over time even in the most challenging scenario. This trend is observed across all our experiments.

\begin{figure}[h!]
    \centering

    \begin{subfigure}{\textwidth}
        \centering
        \begin{subfigure}{0.49\textwidth}
            \includegraphics[width=\linewidth]{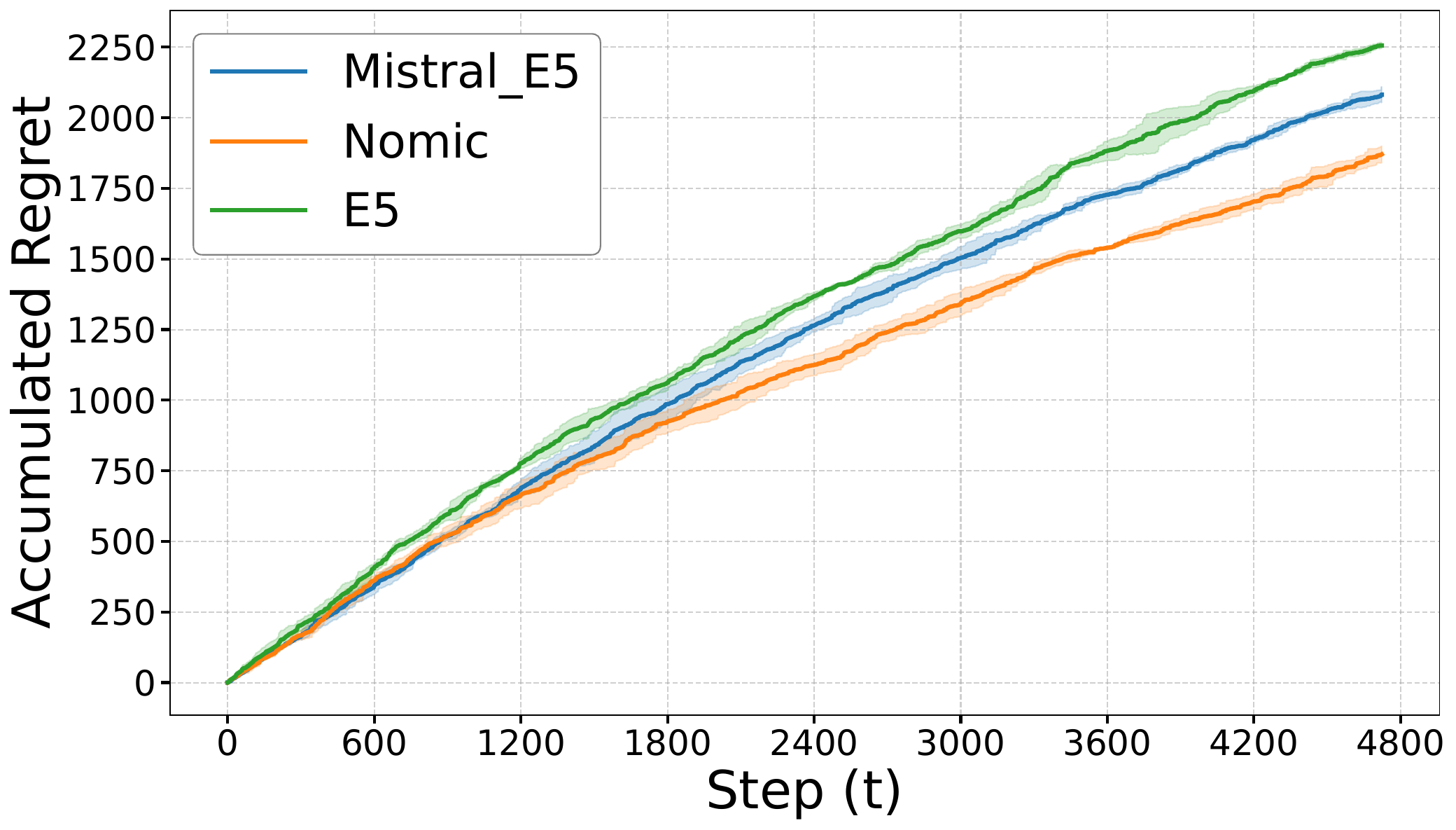}
            \caption{Regret vs. step, ComQA.}
            \label{fig:image1}
        \end{subfigure}\hfill
        \begin{subfigure}{0.49\textwidth}
            \includegraphics[width=\linewidth]{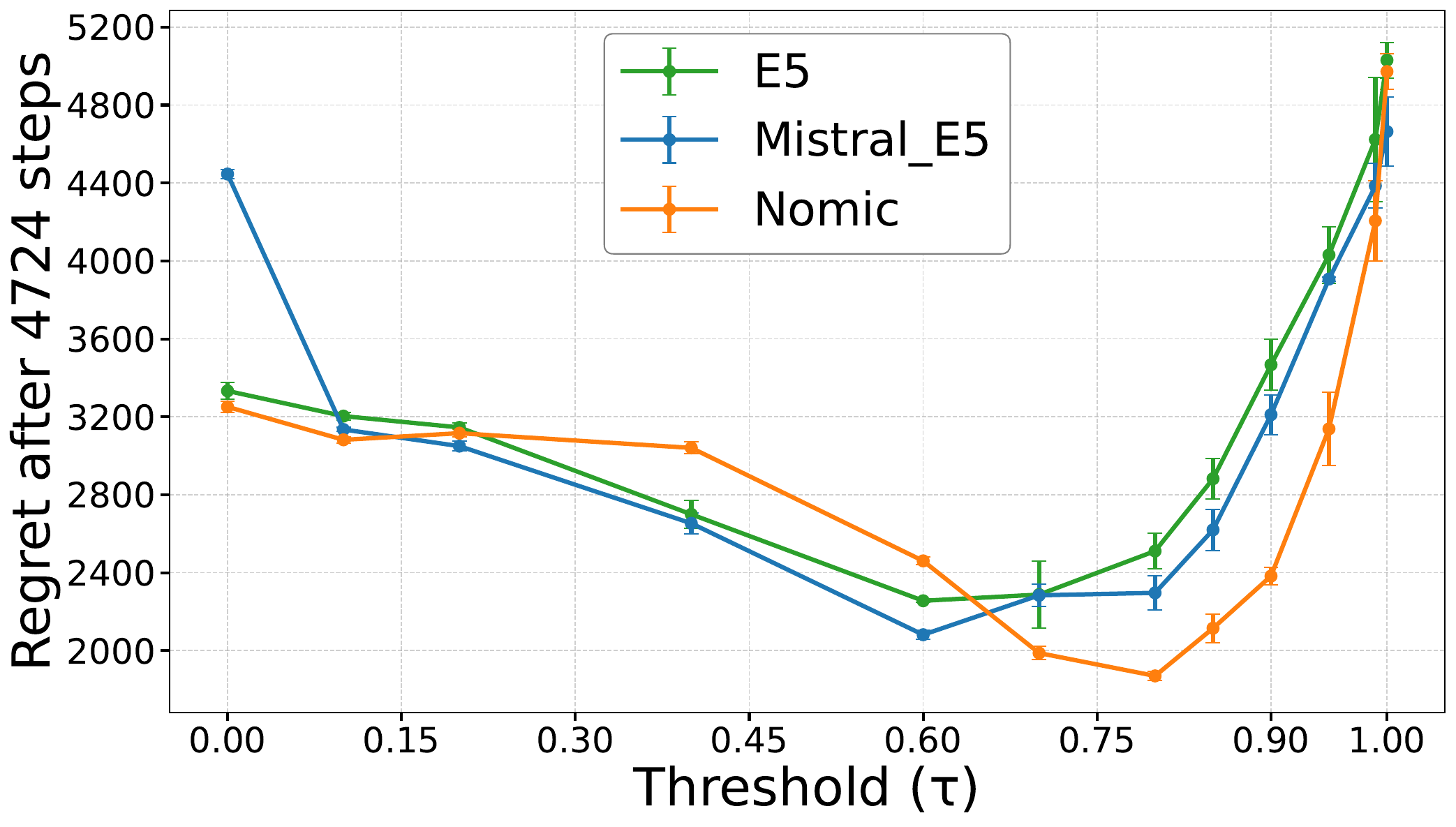}
            \caption{Regret vs. threshold ($T=4724$), ComQA.}
            \label{fig:image2}
        \end{subfigure}
    \end{subfigure}

    \vspace{0.5cm}

    \begin{subfigure}{\textwidth}
        \centering
        \begin{subfigure}{0.49\textwidth}
            \includegraphics[width=\linewidth]{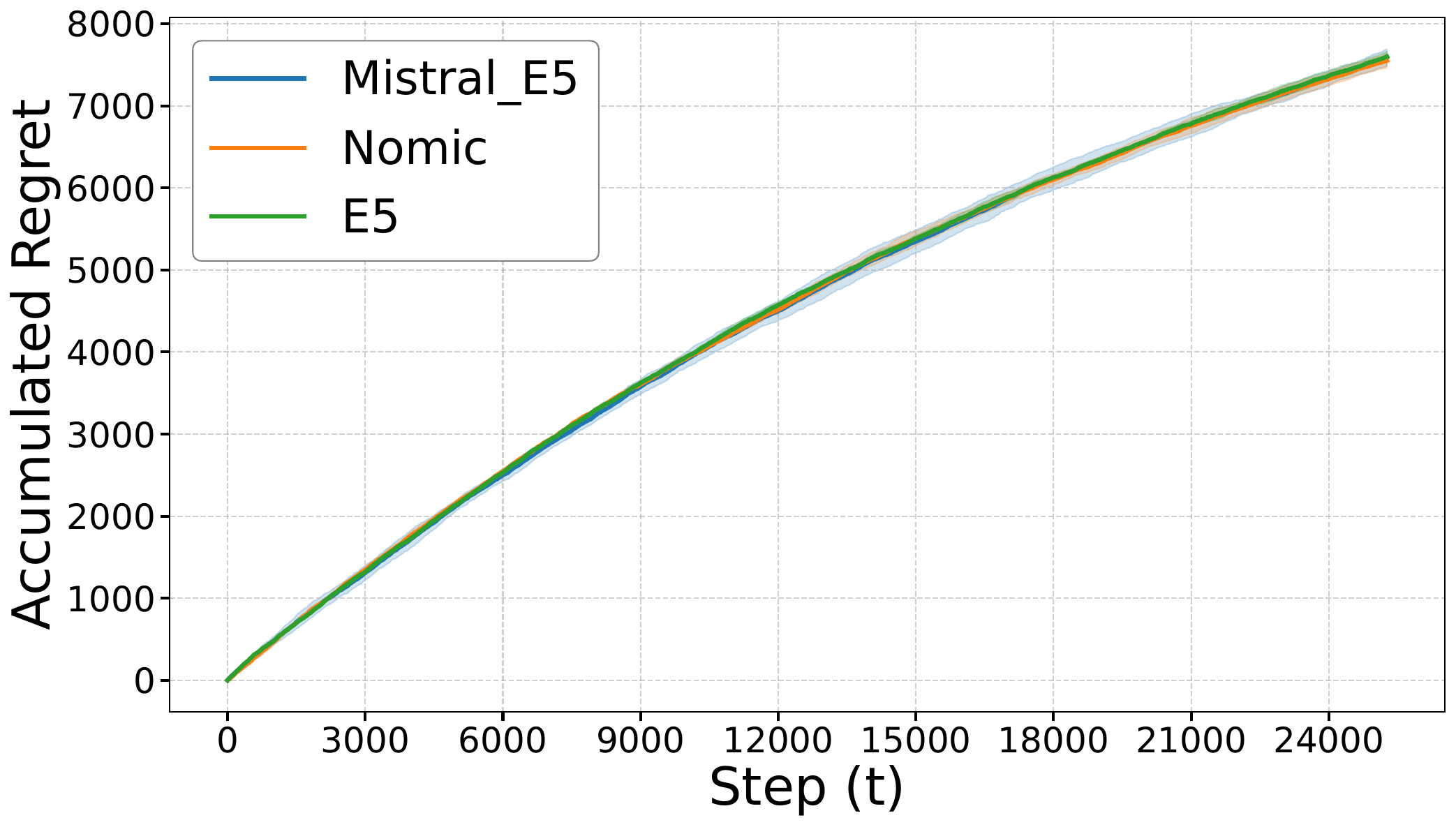}
            \caption{Regret vs. step, CQADupStack.}
            \label{fig:image3}
        \end{subfigure}\hfill
        \begin{subfigure}{0.49\textwidth}
            \includegraphics[width=\linewidth]{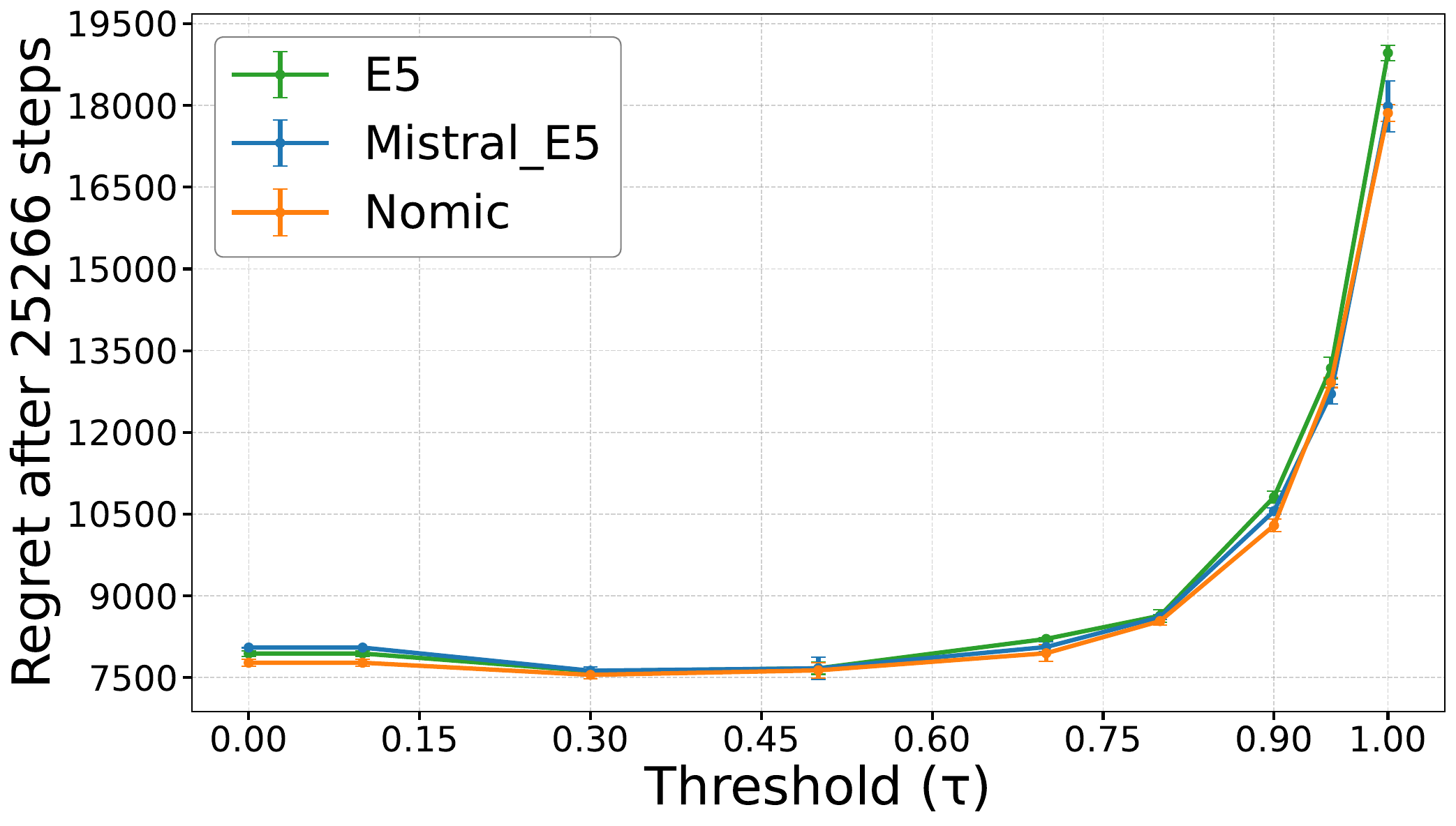}
            \caption{Regret vs. threshold ($T=25266$), CQADupStack.}
            \label{fig:image4}
        \end{subfigure}
    \end{subfigure}

    \caption{Comparison of \vht\ with different text embedding models on ComQA and CQADupStack datasets.}
    \label{fig:results_on_COMQA_and_CQADupStack}
\end{figure}

\subsubsection{Comparison of \vht\ with \extc\ and baselines}\label{sec:comparison_of_different_algorithm}

In this section, we extend the comparative analysis of Section \ref{sec:experiments} and compare the nearest-query version of \vht\ (see Section \ref{sec:additional_real_world_experiments}) with \extc\ for $C=0$ (see Section \ref{sec:etc_algo}), \texttt{AMP} and \texttt{SKM} \footnote{See Appendix \ref{app:synthetic_ghc_skm} for a description of these baselines.} on the QQG and ComQA datasets, using all three embedding models. As in Section \ref{sec:experiments}, we perform three independent runs, and only report the average cumulative regret of \vht\ and \texttt{AMP} for the best performing threshold among specific ranges, which are:
\begin{enumerate}[label=(\roman*)]
\item $S=\{0,0.10,0.20,0.40,0.60,0.70,0.80,0.85,0.90,0.95,0.99,1\}$ for \vht\
\item $S \cup \{0.92,0.94,0.96,0.98\}$ for \texttt{AMP} (because \texttt{AMP} exhibited better performance with very high thresholds)
\end{enumerate}
The results, which are reported in Fig. \ref{fig:results_on_QQG_and_comqa_all_algorithms}, provide additional insight into the relative difficulty of the QQG and ComQA datasets across different algorithms and embedding models. Overall, ComQA appears to be a more challenging dataset, as evidenced by consistently higher regret for all algorithms and all models when compared to QQG. This trend holds across both Mistral and Nomic models, indicating that the increased difficulty is dataset-intrinsic rather than model-specific. When comparing embedding models, Mistral\_E5 achieves the best overall performance on QQG, yielding lower accumulated regret across all algorithms. In contrast, Nomic performs best on the ComQA dataset, suggesting stronger robustness to the increased complexity of ComQA-style queries. In terms of algorithmic performance, the \texttt{AMP} baseline exhibits sublinear regret on QQG, while showing approximately linear regret  on ComQA, highlighting its difficulty in adapting to the harder setting. This indicates that \texttt{AMP} struggles to effectively leverage feedback under increased ambiguity and noise. In contrast, our proposed algorithm \vht\ consistently achieves the lowest regret across all datasets and embedding models. Moreover, the performance gap between \vht\ and competing baselines increases over time, particularly on ComQA, demonstrating superior robustness in challenging environments.
\begin{figure}[h!]
    \centering

    \begin{subfigure}{\textwidth}
        \centering
        \begin{subfigure}{0.49\textwidth}
            \includegraphics[width=\linewidth]{figures/All_algorithms/Mistral_with_legend.pdf}
            \caption{Regret vs. step, Mistral\_E5, QQG.}
            \label{fig:image1}
        \end{subfigure}\hfill
        \begin{subfigure}{0.49\textwidth}
            \includegraphics[width=\linewidth]{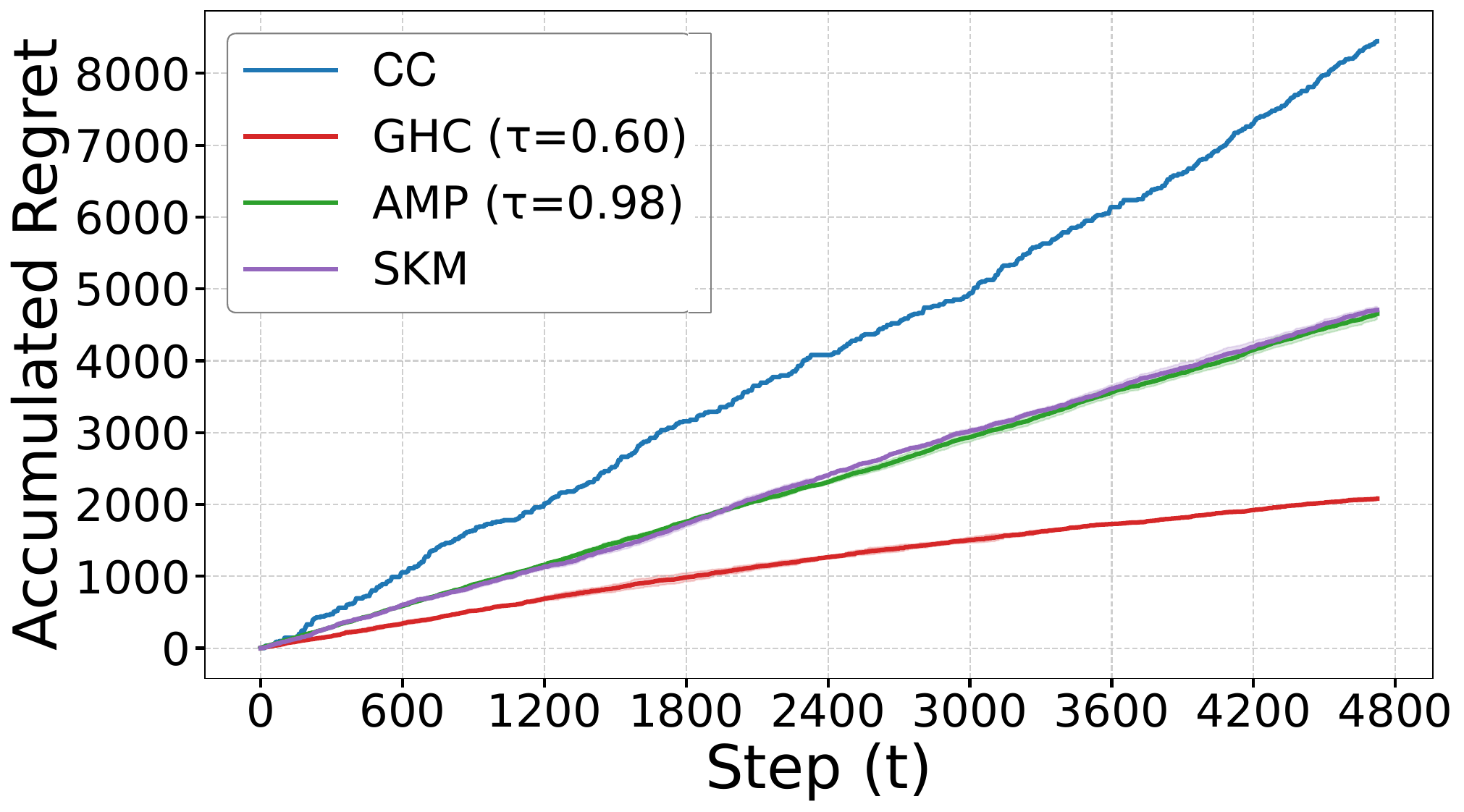}
            \caption{Regret vs. step, Mistral\_E5, ComQA.}
            \label{fig:image2}
        \end{subfigure}
    \end{subfigure}

    \vspace{0.5cm}

    \begin{subfigure}{\textwidth}
        \centering
        \begin{subfigure}{0.49\textwidth}
            \includegraphics[width=\linewidth]{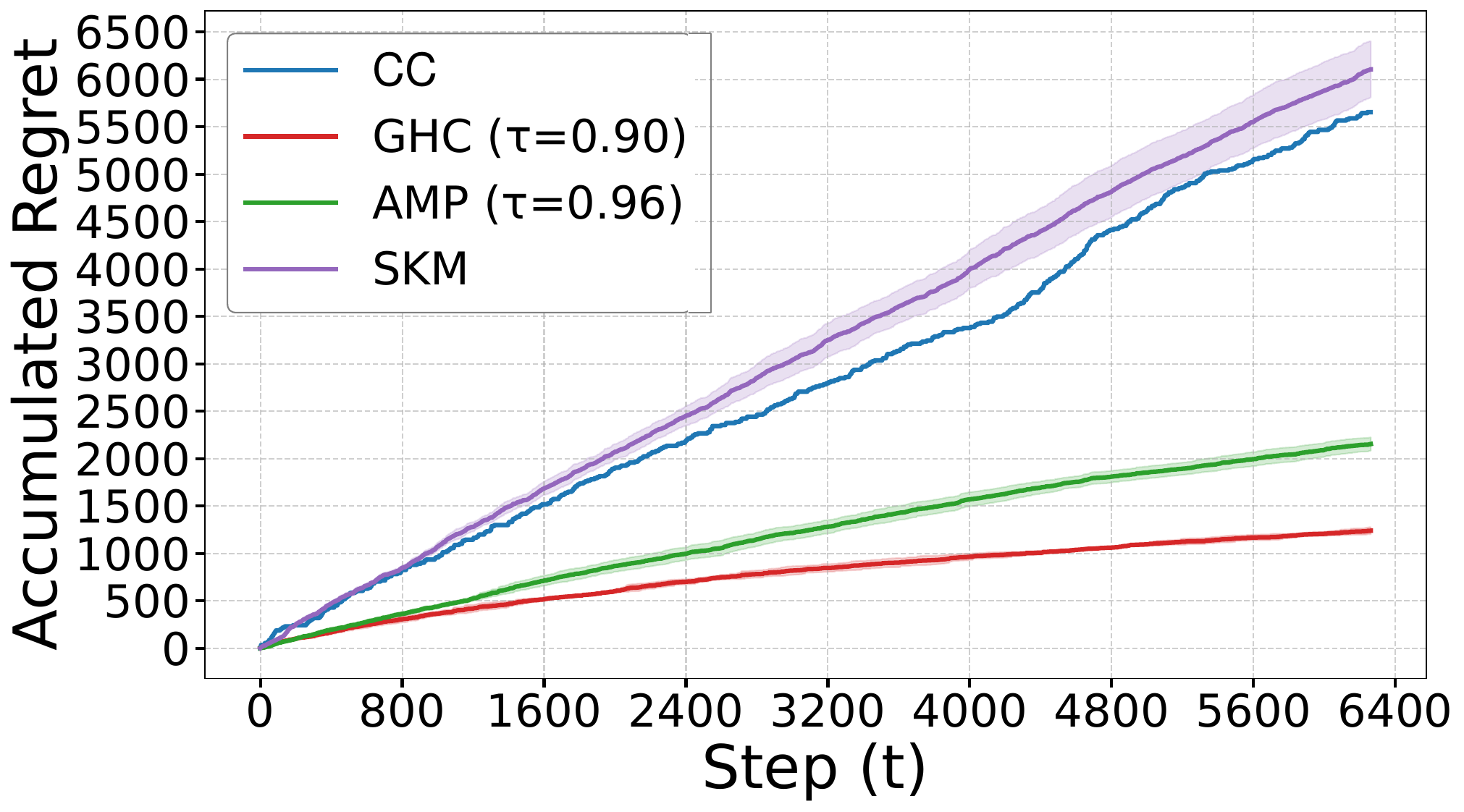}
            \caption{Regret vs. step, Nomic, QQG.}
            \label{fig:image3}
        \end{subfigure}\hfill
        \begin{subfigure}{0.49\textwidth}
            \includegraphics[width=\linewidth]{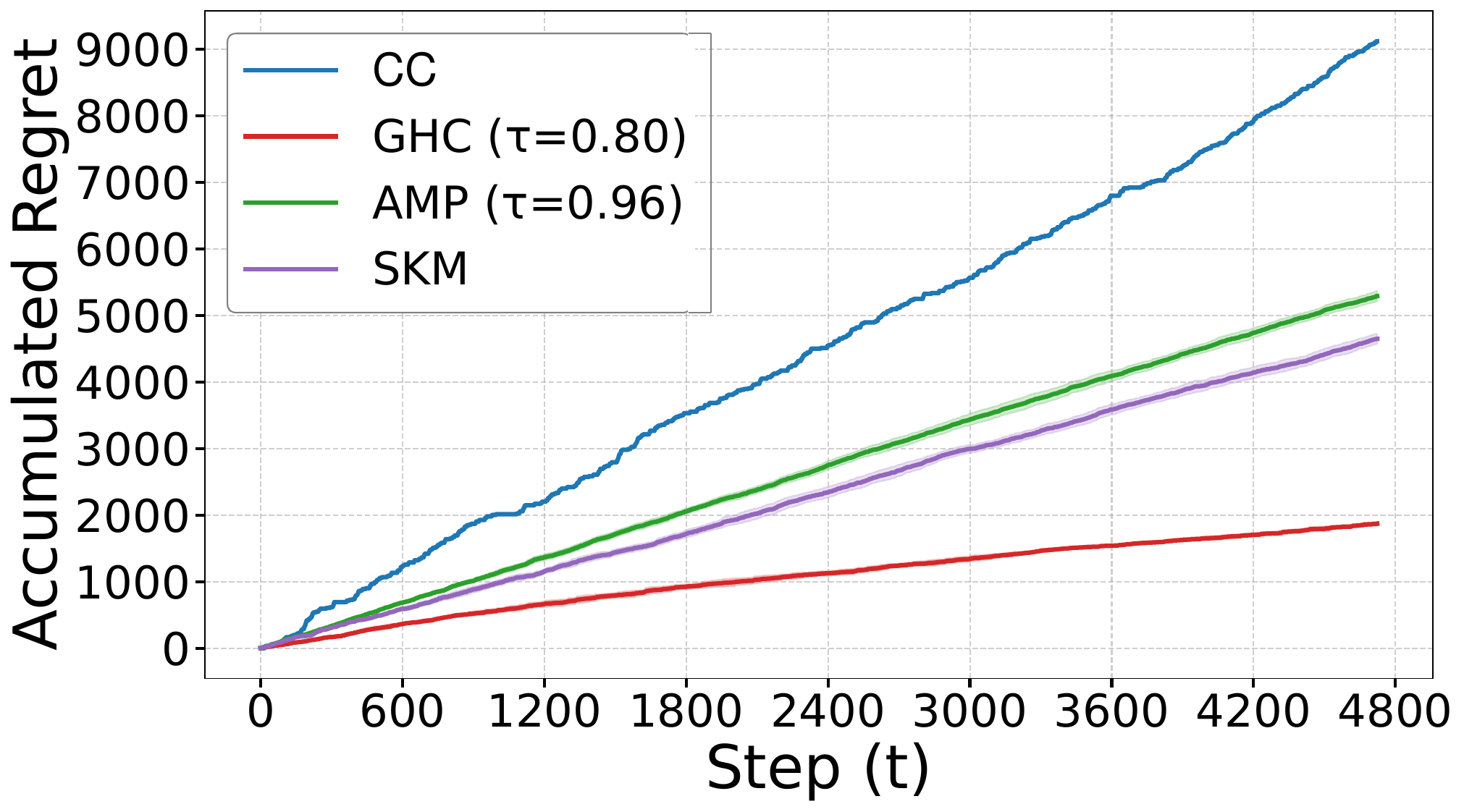}
            \caption{Regret vs. step, Nomic, ComQA.}
            \label{fig:image4}
        \end{subfigure}
    \end{subfigure}

    \vspace{0.5cm}

    \begin{subfigure}{\textwidth}
        \centering
        \begin{subfigure}{0.49\textwidth}
            \includegraphics[width=\linewidth]{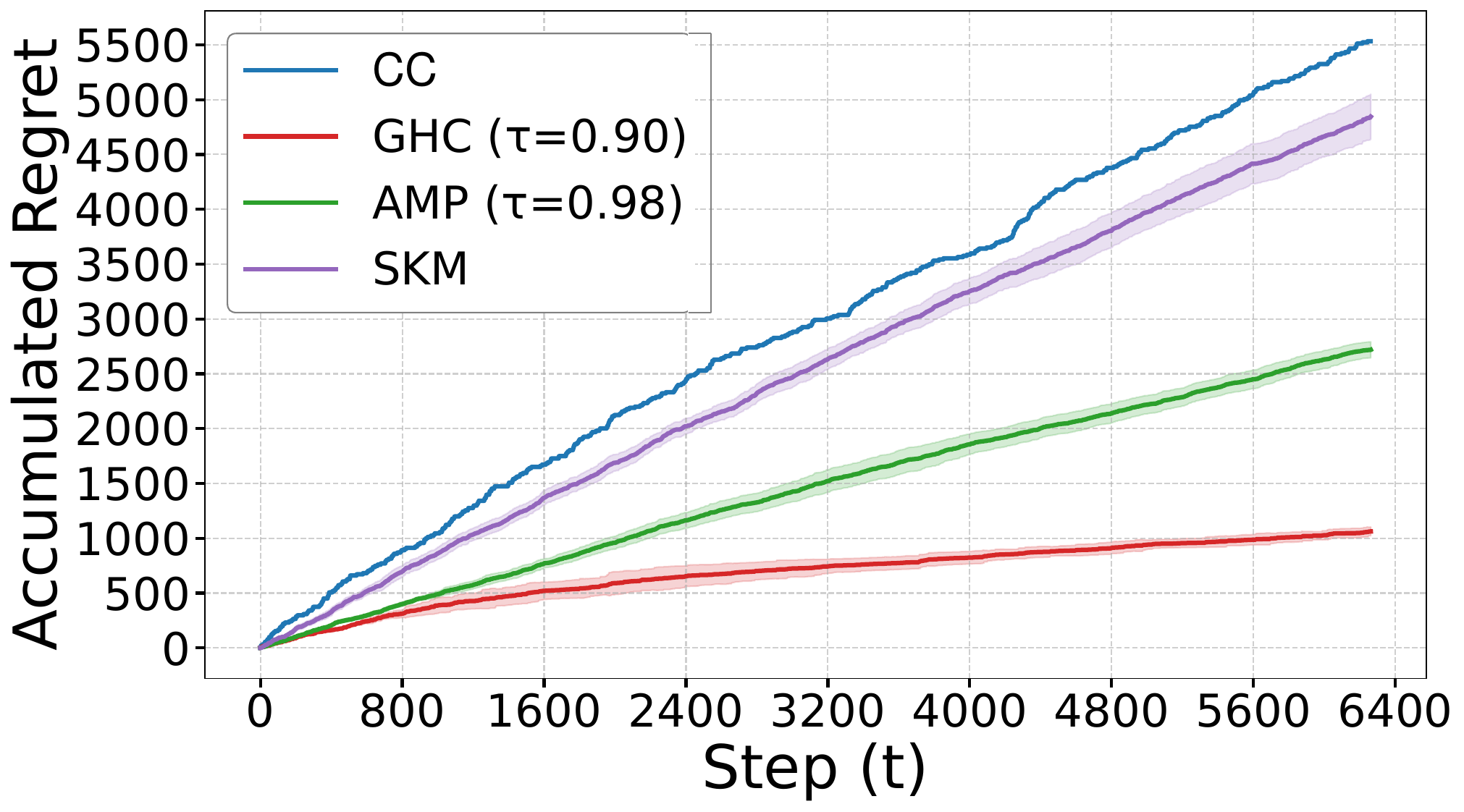}
            \caption{Regret vs. step, E5, QQG.}
            \label{fig:image3}
        \end{subfigure}\hfill
        \begin{subfigure}{0.49\textwidth}
            \includegraphics[width=\linewidth]{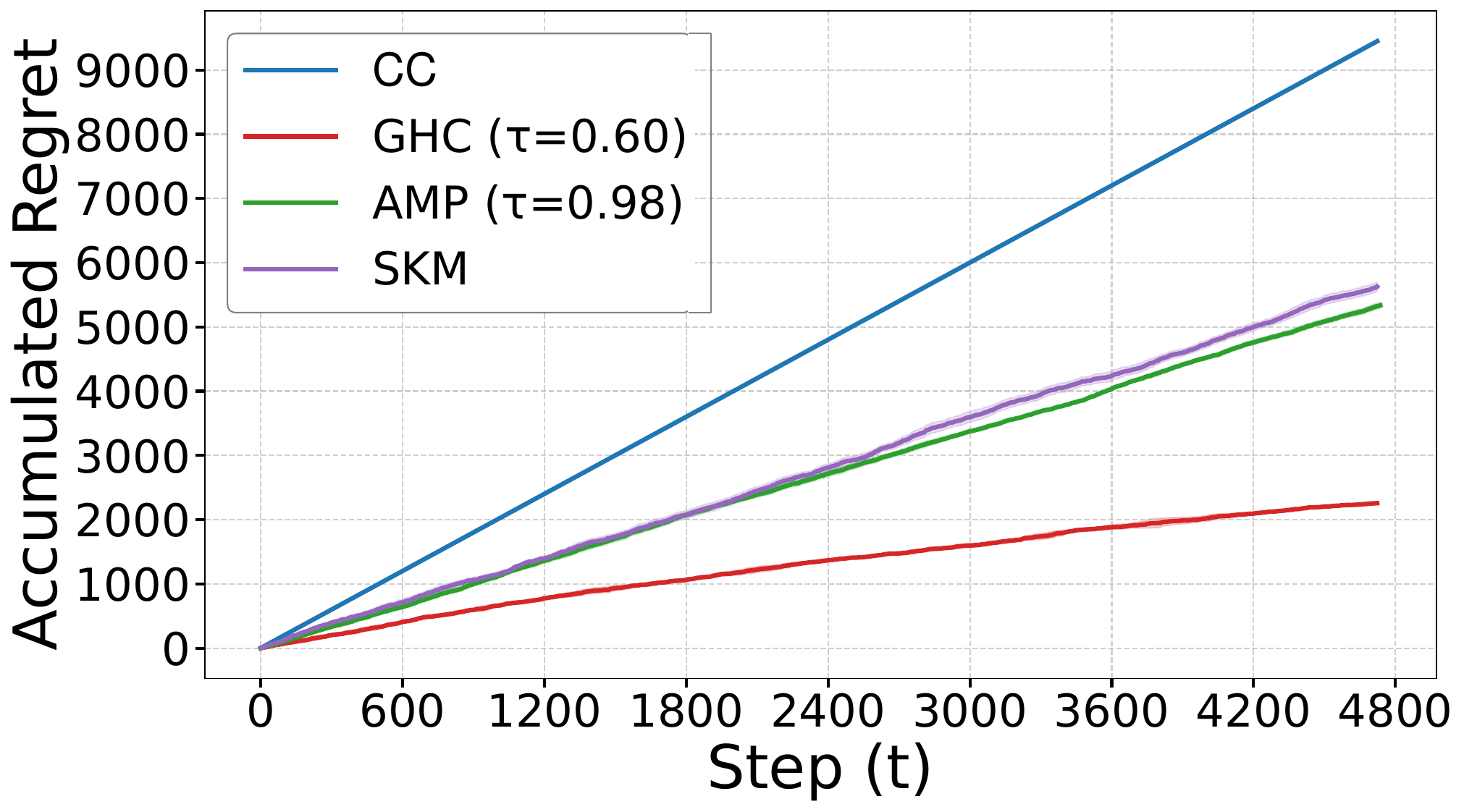}
            \caption{Regret vs. step, E5, ComQA.}
            \label{fig:image4}
        \end{subfigure}
    \end{subfigure}
    \caption{Comparison of \vht\ with baselines using different text embedding models on QQG and ComQA datasets.}
    \label{fig:results_on_QQG_and_comqa_all_algorithms}
\end{figure}
\subsection{Compute Resources for Experiments}\label{sec:compute_resources}
Throughout our experiments, the only component that required substantial computational resources was the encoding of our datasets using Mistral\_E5. As Mistral\_E5 is a text embedding model comprising 7 billion parameters, it necessitates the use of a high-memory GPU for inference, such as the NVIDIA H100 with 80 GB of memory. The encoding of QQG and ComQA takes approximately 20 minutes while the encoding of CQADupStack takes approximately 4 hours. Every other component of the experiments can be executed on a single workstation, provided it is capable of running continuously for up to two days to accommodate the more computationally intensive tasks.